\documentclass[11pt]{article}
\usepackage[margin=1in]{geometry}
\usepackage{amsthm, amsopn, bm, xcolor, amsmath}
\usepackage{multirow}
\usepackage{booktabs}
\usepackage{pifont}
\usepackage{hyperref}

\usepackage{amsmath,amsfonts,bm}

\def\eqref#1{equation~\ref{#1}}

\def\1{\bm{1}}

\def\vzero{{\bm{0}}}

\def\vmu{{\bm{\mu}}}
\def\vtheta{{\bm{\theta}}}
\def\va{{\bm{a}}}
\def\vb{{\bm{b}}}

\def\ve{{\bm{e}}}
\def\vf{{\bm{f}}}
\def\vg{{\bm{g}}}

\def\vp{{\bm{p}}}

\def\vu{{\bm{u}}}
\def\vv{{\bm{v}}}
\def\vw{{\bm{w}}}
\def\vx{{\bm{x}}}
\def\vy{{\bm{y}}}
\def\vz{{\bm{z}}}

\DeclareMathAlphabet{\mathsfit}{\encodingdefault}{\sfdefault}{m}{sl}
\SetMathAlphabet{\mathsfit}{bold}{\encodingdefault}{\sfdefault}{bx}{n}

\def\gL{{\mathcal{L}}}

\def\gR{{\mathcal{R}}}

\def\sP{{\mathbb{P}}}

\newcommand{\E}{\mathbb{E}}

\newcommand{\R}{\mathbb{R}}

\DeclareMathOperator*{\argmin}{arg\,min}

\newcommand{\vbeta}{\boldsymbol{\beta}}

\newcommand{\bF}{\mathbf{F}}
\newcommand{\bO}{\mathbf{O}}
\newcommand{\bQ}{\mathbf{Q}}

\newcommand{\bX}{\mathbf{X}}
\newcommand{\bW}{\mathbf{W}}
\newcommand{\bV}{\mathbf{V}}
\newcommand{\bD}{\mathbf{D}}

\newcommand{\bM}{\mathbf{M}}
\newcommand{\tr}{\operatorname{tr}}
\newcommand{\bR}{\mathbf{R}}

\newcommand{\bS}{\mathbf{S}}
\newcommand{\bU}{\mathbf{U}}
\newcommand{\bK}{\mathbf{K}}

\newcommand{\bT}{\mathbf{T}}
\newcommand{\bG}{\mathbf{G}}

\newcommand{\vnu}{\boldsymbol \nu}
\newcommand{\bSigma}{\boldsymbol \Sigma}
\newcommand{\vepsilon}{\boldsymbol{\epsilon}}

\newcommand{\bDelta}{\mathbf{\Delta}}

\newcommand{\bZ}{\mathbf{Z}}

\newcommand{\bI}{\mathbf{I}}
\newcommand{\normal}{{\sf N}}
\newcommand{\ep}{\varepsilon}

\newcommand\scalemath[2]{\scalebox{#1}{\mbox{\ensuremath{\displaystyle #2}}}}
\newcommand*\colourcheck[1]{%
  \expandafter\newcommand\csname #1check\endcsname{\textcolor{#1}{\ding{52}}}%
}

\colourcheck{green}
\newcommand\independent{\protect\mathpalette{\protect\independenT}{\perp}}
\def\independenT#1#2{\mathrel{\rlap{$#1#2$}\mkern2mu{#1#2}}}

\usepackage{microtype}
\usepackage{graphicx}
\usepackage{subfigure}
\usepackage{booktabs} 

\usepackage{amssymb}
\usepackage{mathtools}
\usepackage[utf8]{inputenc}

\newtheorem{theorem}{Theorem}[section]

\newtheorem{lemma}[theorem]{Lemma}
\newtheorem{corollary}[theorem]{Corollary}

\newtheorem{condition}[theorem]{Condition}

\renewcommand{\epsilon}{\varepsilon}
\usepackage{cleveref}
\newcommand{\op}{\textnormal{op}}

\title{A Theory of Non-Linear Feature Learning
with One Gradient Step
in Two-Layer Neural Networks} 
\author{
Behrad Moniri\footnote{Equal Contribution.} \footnote{Department of Electrical and Systems Engineering, University of Pennsylvania.}
\quad
Donghwan Lee\footnotemark[1]\,\,\footnote{Graduate Group in Applied Mathematics and Computational Science, University of Pennsylvania.} 
\quad 
Hamed Hassani\footnotemark[2]
\quad 
Edgar Dobriban\footnote{Department of Statistics and Data Science, University of Pennsylvania.\newline\hspace*{10pt}
\texttt{\{bemoniri, hassani\}@seas.upenn.edu, dh7401@sas.upenn.edu, \;\;
dobriban@wharton.upenn.edu}.}
}
\date{}

\def\hmath$#1${\texorpdfstring{{\rmfamily\textit{#1}}}{#1}}

\usepackage{setspace}

\begin{document}
\maketitle

\begin{abstract}
Feature learning is thought to be one of the fundamental reasons for the success of deep neural networks. 
It is rigorously known that in two-layer fully-connected neural networks  under certain conditions, one step of gradient descent on the first layer can lead to feature learning; characterized by the appearance of a separated rank-one component---spike---in the spectrum of the feature matrix.
However, with a constant gradient descent step size, this spike only carries information from the linear component of the target function and therefore learning non-linear components is impossible.
We show that with a learning rate that grows with the sample size,
such training in fact introduces 
multiple rank-one components, 
each corresponding to a specific polynomial feature.
We further prove that the limiting large-dimensional and large sample training and test errors of the updated neural networks are fully characterized by these spikes. 
By precisely analyzing the improvement in the training and test errors, we demonstrate that these non-linear features can enhance learning.

\end{abstract}

\section{Introduction}
Learning non-linear features---or representations---from data is thought to be one of the fundamental reasons for the success of deep neural networks  (see e.g., \cite{bengio2013representation,donahue2016adversarial,yang2020feature_learn, shi2022theoretical,radhakrishnan2022feature}, etc.). 
This has been observed in a wide range of domains, including computer vision and natural language processing.
At the same time, the current theoretical understanding of feature learning is incomplete. 
In particular, among many theoretical approaches to study neural networks, much work has focused on 
two-layer fully-connected neural networks with a randomly generated, untrained first layer weights
and a trained second layer---or \textit{random features models} \cite{RahimiRecht}. 
Despite their simplicity, 
random features models can capture various empirical properties of
deep neural networks, and have been used to study generalization,
overparametrization and ``double descent",
adversarial robustness, transfer learning, estimation of out-of-distribution performance, and uncertainty quantification (see e.g., \cite{mei2022generalization,hassani2022curse,tripuraneni2021covariate,disagreement,bombari2023stability,clarte2023double, lin2021causes, adlam2019random}, etc.).

Nevertheless, feature learning is absent in random features models, because the first layer weights are assumed to be randomly generated, and then fixed. 
Although these models can represent non-linear functions of the data,
in the commonly studied setting where the sample size, dimension, and hidden layer size are proportional, 
under certain reasonable conditions they can only learn the \textit{linear} component of the true model---or, 
teacher function---and other components of the teacher function effectively behave as Gaussian noise.
Thus, in this setting, learning in a random features model is equivalent to learning in a \textit{noisy linear model} with Gaussian features and Gaussian noise. This property is known as the \textit{Gaussian equivalence property} (see e.g., \cite{goldt2022gaussian,adlam2019random,adlam2020neural,mei2022generalization,montanari2022universality,hu2022universality}).
While other models such as the neural tangent kernel \cite{jacot2018neural,du2018gradient} can be more expressive, they also lack feature learning.
    \begin{figure*}
        \centering
        \includegraphics{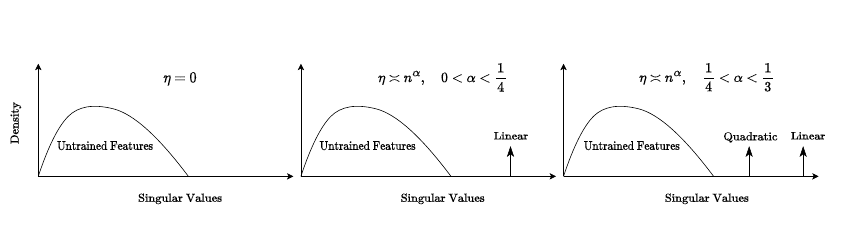}
        \caption{Spectrum of the updated feature matrix for different regimes of the gradient step size $\eta$. Spikes corresponding to monomial features are added to the spectrum of the initial matrix. The number of spikes depends on the range of $\alpha$. See Theorems \ref{thm:spectrum_of_feature_matrix} and \ref{thm:subspace} for details.}
        \label{fig:three-plot-cartoon}
    \end{figure*}
    
To bridge the gap between random features models and feature learning, several recent approaches have shown provable feature learning for neural networks under certain conditions; see Section \ref{relw} for details.
In particular, the recent pioneering work of \cite{ba2022high} analyzed two-layer neural networks, trained with one gradient step on the first layer. 
They showed that when the step size is small, after one gradient step, the resulting two-layer neural network can learn linear features.
However, it still behaves as a noisy linear model and does not capture non-linear components of a teacher function. Moreover, they showed that for a sufficiently large step size, 
under certain conditions,
the one-step updated random features model can outperform linear and kernel predictors.
However, the effects of a large gradient step size on the features is unknown.
Moreover, what happens in the intermediate step size regime also remains unexplored. 
In this paper, we focus on the following key questions in this area:

\begin{quote}
    \textit{What nonlinear features are learned by a two-layer neural network after one gradient update? 
    How are these features reflected in the singular values and vectors of the feature matrix, and how does this depend on the scaling of the step size? 
    What is the improvement in the training and test errors due to 
    the nonlinear features learned?
    }
\end{quote}

\paragraph{Main Contributions.}
Toward answering the above questions, we make the following contributions:
\begin{itemize}
    \item We study feature learning in two-layer neural networks.
    Specifically, we follow the training procedure introduced in \cite{damian2022neural,ba2022high} where one step of gradient descent with step size $\eta$ is applied to the first layer weights, and the second layer weights are found by solving ridge regression on the updated features.
    We consider a step size $\eta \asymp n^\alpha, \alpha \in (0, \frac{1}{2})$ that grows with the sample size $n$ and examine how the learned features change with $\alpha$ (Section \ref{sec:problemsetting}).

    \item In Section \ref{sec:analfeature}, we present a spectral analysis of the updated feature matrix.
    We first show that the spectrum of the feature matrix undergoes phase transitions depending on the range of $\alpha$.
    In particular, we find that if $\alpha \in (\frac{\ell - 1}{2\ell}, \frac{\ell}{2\ell + 2})$ for some $\ell \in \{1,2,\ldots\}$, then $\ell$ separated singular values---\emph{spikes}---will be added to the spectrum of the initial feature matrix (Theorem \ref{thm:spectrum_of_feature_matrix}).
    Figure \ref{fig:three-plot-cartoon} illustrates this finding. 

    \item Building on perturbation theory for singular vectors, we argue that the left singular vectors (principal components) associated with the $\ell$ spikes are asymptotically aligned with polynomial features of different degrees (Theorem \ref{thm:subspace}).
    In other words, the updated feature matrix will contain information about the degree-$\ell$ polynomial component of the target function. 

    \item The Gaussian equivalence property \cite{hu2022universality, mei2022generalization}, an essential tool to analyze random features models, fails after a gradient update with a large step size $\eta$.
    To overcome this difficulty, we establish equivalence theorems (Theorem \ref{thm:training_equivalence} and \ref{thm:test_equivalence}) stating that the trained features $\bF$ can be replaced by sum of the untrained features $\bF_0$ and $\ell$ spikes without changing the training and test errors. Then, by applying the Gaussian equivalence to the untrained component $\bF_0$, we provide a precise characterization of the training and test errors in the high-dimensional proportional regime (Theorem \ref{thm:general_ell_simplified} and \ref{thm:test_risk}).

    \item  From the derived results, we show that in the simple case where $\ell = 1$, the neural network does not learn non-linear functions. However, in the  $\ell = 2$ regime, the neural network in fact learns  quadratic components of the target function.
\end{itemize}
\subsection{Related Works}
\label{relw}

\paragraph{Theory of shallow neural networks.}
Random features models \cite{RahimiRecht} have been used to study various aspects of deep learning, such as generalization
\cite{mei2022generalization, adlam2019random, lin2021causes, mel2021anisotropic}, adversarial robustness \cite{hassani2022curse,bombari23robustness}, transfer learning  \cite{tripuraneni2021covariate}, out-of-distribution performance estimation \cite{disagreement}, uncertainty quantification \cite{clarte2023double}, stability, and privacy \cite{bombari2023stability}.  This line of work builds upon nonlinear random matrix theory (see e.g., \cite{pennington2017nonlinear,louart2018random, fan2020spectra, benigni2021eigenvalue}, etc.) studying the spectrum of the feature matrix of two-layer neural networks at initialization.  

Two-layer neural networks have been studied extensively in the mean-field regime (see e.g., \cite{chizat2018global, mei2018mean,pmlr-v99-mei19a,sirignano2020mean,rotskoff2022trainability}, etc.), and the neural tangent kernel (NTK) regime (see e.g., \cite{jacot2018neural,lee2019wide,huang2020dynamics}, etc.).  However, these results often require the neural net to have an extremely large width. In particular, in the NTK regime, this large width will result in features not evolving over the course of training and the model behaves similar to classic kernel methods.
\cite{ghorbani2021linearized} show that for NTKs and other kernel methods, with a sample size linear in size of the input, non-linear functions cannot be learned. See also \cite{misiakiewicz2022spectrum,xiao2022precise, lu2022equivalence}. Perturbative corrections to the large-width regime to capture feature learning have also been studied in the literature (see e.g., \cite{yaida2020non,hanin2019finite,seroussi2023separation,naveh2021self}, etc.). See Section \ref{sec:additional_related} for more discussion on related work in deep learning theory.

\paragraph{Feature learning.} The problem of feature learning has been gaining a lot of attention recently. 
\cite{damian2022neural} study the problem of learning polynomials with only a few relevant directions and show a sample complexity improvement over kernel methods. \cite{collins2023provable} extend these results and analyze multi-task feature learning in two-layer ReLU networks.
\cite{wang2022spectral} empirically show that if learning rate is sufficiently large, an outlier in the spectrum of the weight and feature matrix emerges with the corresponding singular vector aligned to the structure of the training data. 
\cite{nichani2023provable,wang2023learning} provide theoretical evidence that three-layer neural networks have provably richer feature learning capabilities than their two-layer counterparts.

Recently, \cite{ba2022high} show that in two-layer neural networks, when the dimension, sample size and hidden layer size are proportional, one gradient step with a constant step size on the first layer weights can lead to feature learning. 
However, non-linear components of a single-index target function are still not learned.  They further show that with a sufficiently large step size, when the teacher function has a non-zero first Hermite coefficient,  and under certain conditions, the updated neural networks can outperform linear and kernel methods. 
However, the precise effects of large gradient step sizes on learning nonlinear features, and their precise effects on the loss remain unexplored. \cite{dandi2023learning} show that with a sample size proportional to the input dimension $d$, it is only possible to learn a single direction of multi-index teacher function using gradient updates on the first layer of a two-layer neural network. They also show that for single index models with information exponent (the index of the first non-zero Hermite coefficient) $\kappa$, there are hard directions whose learning requires a sample size of order $\Theta(d^\kappa)$. See also \cite{arous2021online}.

\paragraph{High-dimensional asymptotics.}
We use tools developed in work on high-dimensional asymptotics, which dates back at least to the 1960s
\cite{raudys1967determining,deev1970representation,raudys1972amount}. Recently, these tools have been used in a wide range of areas such as wireless communications  (e.g., \cite{tulino2004random,couillet2011random},  etc.),
high-dimensional statistics (e.g., \cite{raudys2004results,serdobolskii2007multiparametric,paul2014random,yao2015large,dobriban2018high}, etc.), and machine learning  (e.g., \cite{gyorgyi1990statistical,opper1995statistical,opper1996statistical,couillet2022random,engel2001statistical},  etc.).
In particular, the spectrum of 
 so-called information plus noise random matrices that arise in Gaussian equivalence results
has been studied in
\cite{dozier2007empirical,peche}
and its spikes in
\cite{capitaine2014exact}.

\section{Preliminaries}

{\bf Notation}.
We let $\mathbb{N} = \{1,2,\ldots\}$ be the set of positive integers.
For a positive integer $d\ge 1$, we denote $[d] = \{1,\ldots,d\}$.
We use $O(\cdot)$ and $o(\cdot)$ for the standard big-O and little-o notation.
For a matrix $\mathbf{A}$ and a non-negative integer $k$, $\mathbf{A}^{\circ k}=\mathbf{A}\circ\mathbf{A}\circ\ldots\circ\mathbf{A}$ is the matrix of the $k$-th powers of the elements of $\mathbf{A}$.
For positive sequences $(A_n)_{n\ge1}, (B_n)_{n\ge1}$,
we write $A_n = \Theta(B_n)$
or $A_n \asymp B_n$
or $A_n \equiv B_n$
if 
there is $C,C'>0$ such that
$CB_n \ge A_n \ge C'B_n$ for all $n$.
We use $O_\sP(\cdot)$, $o_\sP(\cdot)$,  and $\Theta_\sP(\cdot)$
for the same notions holding in probability.
The symbol $\to_P$ denotes convergence in probability.

\subsection{Problem Setting}\label{sec:problemsetting}
In this paper, we study a supervised learning problem with training data $(\vx_i, y_i) \in \R^d\times \R,$ for $i \in [2n]$, where $d$ is the feature dimension and $n\ge 2$ is the sample size. We assume that the data is generated according to
\begin{align}\label{eqn:datagen}
\vx_i \stackrel{\text{i.i.d.}}{\sim} \normal(0, \bI_d), \text{ and } y_i = f_\star(\vx_i) + \epsilon_i,
\end{align}
in which 
$f_\star$ is the ground truth or \emph{teacher function}, and
$ \ep_i \stackrel{\text{i.i.d.}}{\sim} \normal(0, \sigma_\ep^2)$ is additive noise. 

We fit a model to the data in order to predict outcomes for unlabeled examples at test time; using a two-layer neural network.  We let the width of the internal layer be $N \in \mathbb{N}$. For 
a weight matrix $\bW_{\rm NN} \in \R^{N \times d}$, an activation function $\sigma: \R \to \R$ applied element-wise, and the weights $\va_{\rm NN} \in \R^{N}$ of a linear layer, we define the two-layer neural network  as
$f_{\bW_{\rm NN},\va_{\rm NN}}(\vx) = \va_{\rm NN}^\top \sigma\left(\bW_{\rm NN} \vx\right).$

Following \cite{damian2022neural,ba2022high}, for the convenience of the theoretical analysis, 
we split the training data into two parts: $\bX = [\vx_1, \dots, \vx_n]^\top \in \mathbb{R}^{n \times d}, \vy = (y_1, \dots, y_n)^\top \in \R^n$ and $\tilde \bX = [\vx_{n + 1}, \dots, \vx_{2n}]^\top \in \mathbb{R}^{n \times d}, \tilde \vy = (y_{n + 1}, \dots, y_{2n})^\top \in \R^n$.

We train the two layer neural network as follows.
First, we initialize $\va_{\rm NN}$ with $N^{-1/2}\,\va$ where $\va = (a_1, \dots, a_N)^\top$ in which $a_i \stackrel{\text{i.i.d.}}{\sim} \normal\left(0, 1/{N}\right)$, and initialize $\bW_{\rm NN}$ with $\bW_0 = \begin{bmatrix} \vw_{0, 1}, \dots, \vw_{0, N} \end{bmatrix}^\top \in \mathbb{R}^{N \times d},$ $\vw_{0, i} \stackrel{\text{i.i.d.}}{\sim} \operatorname{Unif}(\mathbb{S}^{d - 1})$
where $\mathbb{S}^{d - 1}$ is the unit sphere in $\R^d$ and $\operatorname{Unif}(\mathbb{S}^{d - 1})$ is the uniform measure over it.
Although we choose this initialization for a simpler analysis,  many arguments can be shown to hold if we switch from the uniform distribution over the sphere to a Gaussian; for example, see Section \ref{pflemma:l1_limits}.
Fixing $\va_{\rm NN}$ at initialization, we perform \emph{one step of gradient descent} on $\bW_{\rm NN}$ with respect to the squared loss computed on $(\bX, \vy)$.
Recalling that $\circ$ denotes element-wise multiplication,
the (rescaled) negative gradient can be written as
\begin{align*}
    \bG&:=-\sqrt{N}\frac{\partial}{\partial \bW_{\rm NN}} \left[ \frac{1}{2n} \left\Vert \vy - \frac{1}{\sqrt{N}}\sigma(\bX \bW_{\rm NN}^\top) \va \right\Vert_2^2 \right]_{\bW_{\rm NN} = \bW_0}\\[0.3cm] &= \frac{1}{n}  \left[ \left(\va \vy^\top - \frac{1}{\sqrt{N}}\va \va^\top \sigma(\bW_0 \bX^\top)\right) \circ \sigma'(\bW_0 \bX^\top) \right] \bX,
\end{align*}
and the one-step update is $\bW  = \begin{bmatrix} \vw_{1}, \dots, \vw_{N} \end{bmatrix}^\top =  \bW_0 + \eta\, \bG$ for a \emph{learning rate} or \emph{step size} $\eta$. 

After the update on $\bW_{\rm NN}$, we perform ridge regression on $\va_{\rm NN}$ using $(\tilde \bX, \tilde \vy)$.
Let $\bF = \sigma(\tilde \bX \bW^\top) \in \R^{n \times N}$ be the feature matrix 
after the one-step update.
For a regularization parameter $\lambda > 0$, we set
\begin{align}\label{eqn:ridgea}
    \hat\va = \hat\va(\bF) = \argmin_{_{\rm NN} \in \R^N} \frac{1}{n} \left\Vert \tilde\vy - \bF \va_{\rm NN} \right\Vert_2^2 + \lambda \Vert \va_{\rm NN} \Vert_2^2 = \left(\bF^\top \bF + \lambda n \bI_N\right)^{-1}\bF^\top \tilde \vy.
\end{align}
Then, 
for a test datapoint with features $\vx$,
we predict the outcome $\hat y = f_{\bW, \hat \va}(\vx) = \hat{\va}^\top \sigma\left(\bW \vx\right)$.

\subsection{Conditions}
Our theoretical analysis applies under the following conditions:
\begin{condition}[Asymptotic setting]
\label{cond:limit}
    We assume that the sample size $n$, dimension $d$, and width of hidden layer $N$ all tend to infinity with 
    \begin{align*}
        d / n \to \phi >0, \quad \text{ and } \quad d / N \to \psi>0.
    \end{align*}    
\end{condition}
We further consider the following model for the teacher function, leading to a single-index model.
\begin{condition}\label{cond:te}
    We let $f_\star: \R^d \to \R$ be $f_\star(\vx) = \sigma_\star(\vx^\top\vbeta_\star)$ for all $\vx$, 
    where $\vbeta_\star \in \R^d$ is an unknown parameter with $\vbeta_\star \sim \normal(0, \frac{1}{d} \bI_d)$ and $\sigma_\star:\R\to \R$ is a 
    $\Theta(1)$-Lipschitz
    \emph{teacher activation} function.
    \end{condition}
This condition is in line with prior work (see e.g., \cite{ba2022high,hu2022universality,goldt2022gaussian}, etc.). Recently, \cite{dandi2023learning} showed that under Condition~\ref{cond:limit}, two-layer neural networks trained with one gradient descent step can only learn a single-index approximation of a multi-index model. This shows that when studying  a single-step update, \ref{cond:te} is not restrictive.

We let $H_k$, $k\ge 1$ be the (probabilist's) Hermite polynomials on $\R$ defined by 
\begin{align*}
H_k(x) = (-1)^k \exp(x^2/2) \frac{d^k}{dx^k} \exp(-x^2/2),    
\end{align*}
for any $x\in \R$.
These polynomials form an orthogonal basis in 
the Hilbert space $L^2$ of measurable functions $f:\R\to\R$ such that
$\int f^2(x) \exp(-x^2/2 )dx<\infty$
with inner product $\langle f,g\rangle = \int f(x) g(x) \exp(-x^2/2)dx$.
The first few Hermite polynomials are $H_0(x) = 1, H_1(x) = x,$ and $H_2(x) = x^2 - 1$.

\begin{condition}\label{cond:he}
    The activation function $\sigma: \R \to \R$ has the following Hermite expansion in $L^2$:
    \begin{align*}
        \sigma(z) = \sum_{k = 1}^\infty c_k H_k(z), \quad c_k = \frac{1}{k!}\E_{Z \sim \normal(0, 1)}[\sigma(Z) H_k(Z)].
    \end{align*}
    The coefficients satisfy $c_1 \neq 0$ and $c_k^2 k! \leq C k^{-\frac{3}{2} - \omega}$ for some $C, \omega > 0$ and for all $k \ge 1$.
    Moreover, the first three derivatives of $\sigma$ exist almost surely, and are bounded.
\end{condition}

Note that in this paper, unlike \cite{hu2022universality}, we do not require the activation function to be odd. The reason is that here, unlike \cite{hu2022universality}, we do not analyze the problem for a general loss function and use a proof technique specialized for squared loss. We remark that the above condition requires
$c_0 = 0$, i.e., that $\E\sigma(Z) = 0$ for $Z \sim \normal(0, 1)$. 
This condition is in line with prior work in the area (e.g., \cite{adlam2020neural,ba2022high}, etc.), and could be removed at the expense of more complicated formulas and theoretical analysis.
The smoothness assumption on $\sigma$
is also in line with prior work in the area (see e.g., \cite{hu2022universality,ba2022high}, etc.). 
Note that the above condition is satisfied by many popular activation functions (after shifting) such as the ReLU $\sigma(x) = \max\{x, 0\} - \frac{1}{\sqrt{2\pi}}$, hyperbolic tangent $\sigma(x) = \frac{e^x - e^{-x}}{e^x + e^{-x}}$, and sigmoid $\sigma(x) = \frac{1}{1 + e^{-x}} - \frac{1}{2}$, for all $x$.

We also make similar assumptions on the teacher activation:
\begin{condition}\label{cond:tehe}
    The teacher activation $\sigma_\star: \R \to \R$ has the following Hermite expansion in $L^2$:
    \begin{align*}
        \sigma_\star(z) = \sum_{k = 1}^\infty c_{\star,k} H_k(z), \; c_{\star,k} = \frac{1}{k!}\E_{Z}[\sigma_\star(Z) H_k(Z)],
    \end{align*}
    with $Z \sim \normal(0, 1)$.
    Also, we define $c_\star = (\sum_{k = 1}^\infty k! c_{\star,k}^2)^\frac{1}{2}$.
\end{condition}

\section{Analysis of the Feature Matrix}\label{sec:analfeature}
The first step in analyzing the spectrum of the feature matrix $\bF$ is to study the matrix rescaled negative gradient $\bG$. It is shown in \cite[Proposition 2]{ba2022high} that $\bG = c_1 \va\vbeta^\top + \boldsymbol{\Delta}$, where $\|c_1 \va\vbeta^\top\|_{\rm op} = \Theta_{\sP}(1)$ and $\|\boldsymbol{\Delta}\|_{\rm op} = \tilde{O}_{\sP}(1/\sqrt{N})$, where the Hermite coefficient $c_1$ of the activation $\sigma$
is defined in Condition \ref{cond:he}, and $\vbeta = \frac{1}{n} \bX^\top \vy \in \R^d$; i.e., the matrix $\bG$ can be approximated (in operator norm) by the rank-one matrix $c_1 \va\vbeta^\top$ with high probability. Moreover, under Conditions \ref{cond:limit}-\ref{cond:tehe}, \cite{ba2022high} show that 
$\vbeta$ can be understood as a noisy estimate of $\vbeta_\star$, namely 
    \begin{align}
    \label{prop:alignment}
        \frac{\vert \vbeta_\star^\top \vbeta \vert}{\Vert \vbeta_\star \Vert_2 \Vert \vbeta \Vert_2} \to_P \frac{\vert c_{\star, 1}\vert}{\sqrt{c_{\star, 1}^2 + \phi (c_\star^2 + \sigma_\ep^2)}}.
    \end{align}
See also Lemma \ref{lemma:turn_beta_to_beta_star}.
In particular, if the sample size 
used for the gradient update is very large; i.e., $\phi \to 0$, $\vbeta$ will converge to being completely aligned to $\vbeta_\star$. 

\begin{figure*}
    \centering
    \includegraphics[width=1\textwidth]{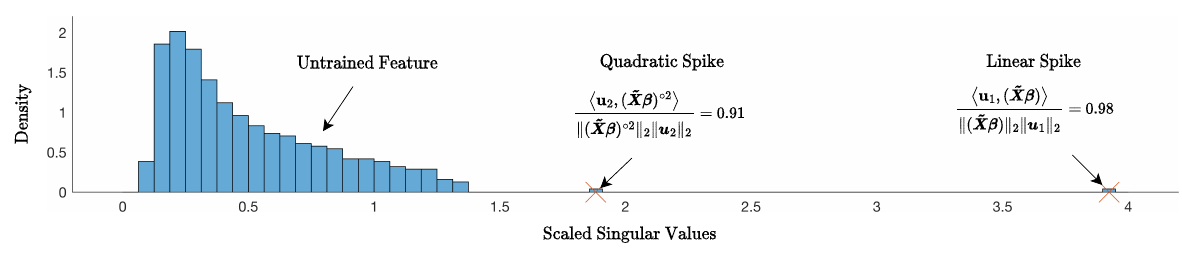}
    \caption{Histogram of the scaled singular values (divided by $\sqrt{n}$) of the feature matrix $\bF = \sigma(\tilde\bX\bW^\top)$ after the update with step size $\eta = n^{0.29}$ $(\ell = 2)$. In this regime, two isolated spikes appear in the spectrum as stated in Theorem~\ref{thm:spectrum_of_feature_matrix}. The top two left singular vectors $\vu_1$ and $\vu_2$ are aligned with $\tilde\bX\vbeta$ and $(\tilde\bX\vbeta)^{\circ 2}$, respectively. See Section \ref{sec:numerical} for the simulation details.}
    \label{fig:detailed-spikes}
\end{figure*}

Building on this result, we show the spectrum of the feature matrix $\bF  = \sigma(\tilde\bX(\bW_0+\eta\mathbf{G})^\top)$ will consist of a bulk of singular values that stick close together---given by the spectrum of the initial feature matrix $\bF_0 = \sigma(\tilde \bX \bW_0^\top)$---and $\ell$ separated spikes\footnote{Using terminology from random matrix theory \cite{bai2010spectral,yao2015large}.}, 
where $\ell$ is an integer that depends on the step size used in the gradient update.
Specifically, when the step size is $\eta \asymp n^\alpha$ with $\frac{\ell - 1}{2\ell} < \alpha < \frac{\ell}{2\ell + 2}$ for some $\ell \in \mathbb{N}$, the feature matrix $\bF$ can be approximated in operator norm by the untrained features $\bF_0 = \sigma(\tilde \bX \bW_0^\top)$ plus  $\ell$ rank-one terms, where the left singular vectors of the rank-one terms are aligned with the non-linear features $\tilde \bX\mapsto (\tilde \bX \vbeta)^{\circ k}$, for $k \in [\ell]$. See Figure \ref{fig:detailed-spikes}.

\begin{theorem}[Spectrum of feature matrix]\label{thm:spectrum_of_feature_matrix}
   Let $\eta \asymp n^\alpha$ with $\frac{\ell - 1}{2\ell} < \alpha < \frac{\ell}{2\ell + 2}$ for some $\ell \in \mathbb{N}$.
    If Conditions \ref{cond:limit}-\ref{cond:tehe} hold, then
    for $c_k$ from Condition \ref{cond:he}
    and $\bF_0 = \sigma(\tilde \bX \bW_0^\top)$,
    \begin{align}        \label{eqn:Fell}
        &\bF = \bF_\ell+\bDelta,\quad\,\textnormal{with }\quad
    \bF_\ell := \bF_0 + \sum_{k = 1}^\ell c_1^k c_k \eta^k (\tilde \bX \vbeta)^{\circ k} (\va^{\circ k})^\top,
    \end{align}
    where $\Vert \bDelta \Vert_\textnormal{op}= o(\sqrt{n})$ with probability $1 - o(1)$.
\end{theorem}
To understand $(\tilde \bX \vbeta)^{\circ k} (\va^{\circ k})^\top$, notice that for a datapoint with features $\tilde \vx_i$, the activation of each neuron is proportional to the polynomial feature
$(\tilde \vx_i^\top \vbeta)^k$, with coefficients given by $\va^{\circ k}$ for the neurons.
The spectrum of the initial feature matrix $\bF_0$ is fully characterized in {\cite{pennington2017nonlinear, benigni2021eigenvalue, benigni2022largest, louart2018random, fan2020spectra}}, and its operator norm is known to be $\Theta_\sP(\sqrt{n})$. Moreover, it follows from the proof that the operator norm of each of the terms $c_1^k c_k  \eta^k (\tilde \bX \vbeta)^{\circ k} (\va^{\circ k})^\top$, $k\in [\ell]$ is with high probability of order larger than  $\sqrt{n}$.
Thus, \Cref{thm:spectrum_of_feature_matrix} identifies the spikes in the spectrum of the feature matrix.

\paragraph{Proof Idea.}
We approximate the feature matrix $\bF  = \sigma(\tilde\bX\bW_0^\top+\tilde\bX\mathbf{G}^\top)$ by a polynomial using its Hermite expansion. Next, recalling that $\mathbf{G} \approx c_1 \vbeta \va^\top$, we set
$(\tilde \bX \bG^\top)^{\circ k}$ to $c_1^k (\tilde \bX \vbeta)^{\circ k} (\va^{\circ k})^\top$.
We show that the spike terms with $k \geq \ell + 1$ are negligible since we can show that their norm is $O_\sP(n^{k\alpha + \frac{1}{2} - \frac{k - 1}{2}}) = o_\sP(\sqrt{n})$.

The special case where $\alpha = 0$ is discussed in \cite[Section 3]{ba2022high},
which focuses on the spectrum of the updated weight matrix $\bW = \bW_0 + \eta \bG$. However, here we study the updated feature matrix $\bF = \sigma (\tilde \bX (\bW_0 + \eta \bG)^\top)$ because that is more directly related to the learning problem---as we will discuss in the consequences for the training and test errors below.

In the following theorem, we argue that the subspace spanned by the non-linear features $\{\sigma(\tilde \bX \vw_i)\}_{i \in [N]}$ can be approximated by the subspace spanned by the monomials $\{(\tilde \bX \vbeta)^{\circ k}\}_{k \in [\ell]}$.
For 
two $\ell$-dimensional subspaces $\mathcal{U}_1, \mathcal{U}_2 \subseteq \R^n$,
with orthonormal bases
$\mathbf{U}_1, \mathbf{U}_2 \in \R^{n \times \ell}$,
recall the principal angle distance  between $\mathcal{U}_1, \mathcal{U}_2$ defined by
$d(\mathcal{U}_1, \mathcal{U}_2) = \min_{\mathbf{Q}} \Vert \mathbf{U}_1 -  \mathbf{U}_2 \mathbf{Q} \Vert_\textnormal{op}$,
where the minimum is over $\ell \times \ell$ orthogonal matrices \cite{stewart1990matrix}. 
This definition is invariant to the choice of $\mathbf{U}_1, \mathbf{U}_2$. 

\begin{theorem}\label{thm:subspace}
    Let $\mathcal{F}_\ell$ be the $\ell$-dimensional subspace of $\R^n$ spanned by top-$\ell$ left singular vectors (principal components) of $\bF$.
    Under the conditions of Theorem \ref{thm:spectrum_of_feature_matrix}, we have
    \begin{align*}
        d(\mathcal{F}_\ell, \operatorname{span}\{(\tilde \bX \vbeta)^{\circ k}\}_{k \in [\ell]}) \to_P 0.
    \end{align*}
\end{theorem}
This result shows that after one step of gradient descent with step size $\eta \asymp n^\alpha$ with $\frac{\ell - 1}{2\ell} < \alpha < \frac{\ell}{2\ell + 2}$, the subspace of the top-$\ell$ left singular vectors carries information from the polynomials $\{(\tilde \bX \vbeta)^{\circ k}\}_{k \in [\ell]}$. Also, recall that by \eqref{prop:alignment}, the vector $\vbeta$ is aligned with $\vbeta_\star$. Hence, it is shown that $\mathcal{F}_\ell$ carries information from the first $\ell$ polynomial components of the teacher function.

\paragraph{Proof Idea.}
We use Wedin's theorem
\cite{wedin1972perturbation}
to characterize the distance between the left singular vector space of $\sum_{k = 1}^\ell c_1^k c_k \eta^k (\tilde \bX \vbeta)^{\circ k} (a^{\circ k})^\top$ and that of $\bF$.
Here, we consider the matrix $\bF_0 + \bDelta$ as the perturbation term.

\section{Learning Higher-Degree Polynomials}\label{sec:learninghigher}

In the previous section, we studied the feature matrix $\bF$ and showed that when $\eta \asymp n^\alpha$ with $\frac{\ell - 1}{2\ell} < \alpha < \frac{\ell}{2\ell + 2}$, it can be approximated by $\bF_0 = \sigma(\tilde \bX \bW_0^\top)$ plus $\ell$ rank-one or spike terms. 
We also saw that the left singular vectors of the spike terms are aligned with the non-linear functions $\tilde \bX \mapsto (\tilde \bX \vbeta)^{\circ k}$. 
Intuitively, this result suggests that after the gradient update, the trained weights are becoming aligned with the teacher model and we should expect the ridge regression estimator on the learned features to achieve better performance. 
In particular, when $\alpha>0$, we expect the ridge regression estimator to capture the non-linear part of the teacher function. 
This is impossible for $\eta = O(1)$ \cite{ba2022high} or $\eta = 0$ \cite{hu2022universality,mei2022generalization}. 

In this section, we aim to make this intuition rigorous and show that the spikes in the feature matrix lead to 
a decrease in the error achieved by the estimator. 
Moreover, for large enough step sizes, 
the model can learn non-linear components of the teacher function.
For this, we first need to prove
\textit{equivalence theorems} 
showing that 
instead of the true feature matrix $\bF$, 
the approximations from Theorem \ref{thm:spectrum_of_feature_matrix} can be used to compute error terms (i.e., the effect of $\bDelta$ on the error is negligible).

\subsection{Equivalence Theorems}
\label{eqthm}
{The Gaussian equivalence property (\cite{goldt2022gaussian,hu2022universality,montanari2022universality}, etc.) implies that the training and test errors of a random features model are asymptotically the same  as that of a noisy linear model. In other words, the limiting behavior of these quantities is unchanged if 
we replace the untrained feature matrix $\bF_0 = \sigma(\tilde \bX \bW_0^\top)$ with
$\bF_0 = c_1 \tilde\bX \bW_0^\top + c_{>1} \bZ$,
where $\bZ \in \R^{n \times N}$ is an independent random matrix with i.i.d. $\normal(0, 1)$ entries. 

This property has been used extensively in work on random features models  (see e.g., \cite{adlam2020understanding, adlam2020neural,tripuraneni2021covariate,mel2021anisotropic}, etc.) as it provides a powerful tool to analyze non-linear random matrices.
However, the Gaussian equivalence property fails when the weight matrix $\bW$ is updated with a large gradient descent step \cite{ba2022high}, posing a significant challenge to the analysis.

In this section, we first prove that we can replace the trained features $\bF$ with their approximation $\bF_\ell$ from Theorem~\ref{thm:spectrum_of_feature_matrix} in terms of $\bF_0$ and spikes,
without changing the limiting training and test errors.
Then, in the next sections we will see that the training and test errors can be derived by applying the Gaussian equivalence property 
to the untrained features $\bF_0$ only.

Given a regularization parameter $\lambda >0$, 
recalling the ridge estimator $\hat{\va}(\bF)$ from \eqref{eqn:ridgea}, we define the training loss
\begin{align*}
    \gL_\textnormal{tr}(\bF) = \frac{1}{n} \Vert \tilde \vy - \bF\hat{\va}(\bF)\Vert_2^2 + \lambda \Vert \hat{\va}(\bF)\Vert_2^2.
\end{align*}
In the next theorem, we show that when $\eta \asymp n^\alpha$ with $\frac{\ell - 1}{2\ell} < \alpha < \frac{\ell}{2\ell + 2}$, the training loss $\gL_\textnormal{tr}(\bF)$ can be approximated with negligible error by $ \gL_\textnormal{tr}(\bF_\ell)$. 
In other words, the approximation of the feature matrix from Theorem \ref{thm:spectrum_of_feature_matrix} can be used to derive the asymptotics of the training loss.

\begin{theorem}[Training loss equivalence]
\label{thm:training_equivalence}
    Let $\eta \asymp n^\alpha$ with $\frac{\ell - 1}{2\ell} < \alpha < \frac{\ell}{2\ell + 2}$ for some $\ell \in \mathbb{N}$ and recall  $\bF_\ell$ from \eqref{eqn:Fell}. 
    If Conditions \ref{cond:limit}-\ref{cond:tehe} hold, then
    for any fixed $\lambda >0$, with probability $1 - o(1)$ we have    
    $$\gL_\textnormal{tr}(\bF) - \gL_\textnormal{tr}(\bF_\ell) = o(1).$$   
\end{theorem}
Similar equivalence results can also be proved for the test error, i.e., the average test loss. 
For any $\va \in \R^N$, we define the test error of $\va$ as
$\gL_\textnormal{te}(\va) = \E_{\vf, y}(y - \vf^\top \va)^2$,
in which the expectation is taken over
$(\vx,y)$ where
$\vf = \sigma(\bW\vx)$ with $\vx \sim \normal(0, \bI_d)$ and $y = f_\star(\vx) + \epsilon$ with $\epsilon \sim \normal(0, \sigma_\epsilon^2)$. The next theorem shows that one can also use the  approximation of the feature matrix from Theorem \ref{thm:spectrum_of_feature_matrix} to derive the asymptotics of the test error.

\begin{theorem}[Test error equivalence]
\label{thm:test_equivalence}

       Let $\eta \asymp n^\alpha$ with $\frac{\ell - 1}{2\ell} < \alpha < \frac{\ell}{2\ell + 2}$ for some $\ell \in \mathbb{N}$, and $\bF_\ell$ be defined as in \eqref{eqn:Fell}.
       If Conditions \ref{cond:limit}-\ref{cond:tehe} hold, then
       for any $\lambda > 0$, 
       if $\gL_\textnormal{te}(\hat\va(\bF)) \to_P \gL_\bF$ and
       $\gL_\textnormal{te}(\hat\va(\bF_\ell)) \to_P \gL_{\bF_\ell}$, we have $\gL_{\bF} = \gL_{\bF_\ell}$.
\end{theorem}

\paragraph{Proof Idea.} 
To prove Theorem~\ref{thm:test_equivalence}, we show first show that the norm of the trained second layer weight $\hat\va$, is $O_\sP(1)$. Then, we use Theorem~\ref{thm:spectrum_of_feature_matrix} to conclude the proof. To prove Theorem~\ref{thm:test_equivalence}, we will use a \emph{free-energy trick} \cite{abbasi2019universality,hu2022universality,hassani2022curse}. We first extend Theorem \ref{thm:training_equivalence} and show that for any $\lambda, \zeta>0$, the minima over $\va$ of
\begin{align*}
\gR_{\zeta}(\va, \bar\bF) = \frac{1}{n}\Vert \tilde\vy - \bar\bF \va\Vert_2^2 + \lambda \Vert \va\Vert_2^2  + \zeta \gL_{\rm te}(\va),
\end{align*}
for $\bar\bF = \bF$ and $\bar\bF = \bF_\ell$ are close. Then, we use this to argue that the limiting test errors are also close.

With Theorem \ref{thm:training_equivalence} and \ref{thm:test_equivalence} in hand, for $\eta \asymp n^\alpha$, we can  use the approximation $\bF_\ell$---with the appropriate $\ell$---of the feature matrix $\bF$ to analyze the training and  test error.

\subsection{Analysis of Training and Test Errors}
\label{section:analysis_of_train_loss}
In this section, we quantify the discrepancy between the training loss of the ridge estimator trained on the new---learned---feature matrix $\bF$ and the same ridge estimator trained on the untrained feature matrix $\bF_0$. We will do this for the step size $\eta \asymp n^\alpha$ with $\frac{\ell - 1}{2\ell} < \alpha < \frac{\ell}{2\ell + 2}$ for  various $\ell \in \mathbb{N}$.

Our results depend on the limits of 
traces of the matrices
$(\bF_0\bF_0^\top + \lambda n \bI_n)^{-1}$
and
$\tilde \bX^\top (\bF_0\bF_0^\top + \lambda n \bI_n)^{-1} \tilde \bX$.
These limits have been determined in \cite{adlam2019random,adlam2020neural},
see also \cite{pennington2017nonlinear,peche},
and 
depend on the values $m_1, m_2 > 0$, which are the unique solutions of the following system of coupled equations, for $\lambda > 0$:
    \begin{align}\label{fpe}
        &\phi\left(m_1-m_2\right)\left(c_{>1}^2 m_1+c_1^2 m_2\right)+ \Psi(m_1, m_2)=0,\nonumber \\
        &\frac{\phi}{\psi}\left(c_1^2 m_1 m_2+\phi\left(m_2-m_1\right)\right)+ \Psi(m_1, m_2)=0,
    \end{align}
    where $\Psi(m_1, m_2) = c_1^2 m_1 m_2\left(\lambda \psi m_1/\phi-1\right)$ and $c_{>1} = (\sum_{k = 2}^\infty k! c_k^2)^{1/2}$.
Here, $m_1$ is the limiting Stieltjes transform of the matrix $\frac{1}{n} \mathbf{F}_0 \mathbf{F}_0^\top$ and $m_2$ is an auxiliary transform. 
For instance, we leverage that under Condition~\ref{cond:limit}, we have
\begin{align*}
    &\lim_{d, n, N \to \infty} \tr(\tilde \bX^\top (\bF_0\bF_0^\top + \lambda n \bI_n)^{-1} \tilde \bX)/d =  \psi m_2/\phi > 0,\\
    &\lim_{d, n, N \to \infty} \tr((\bF_0\bF_0^\top + \lambda n \bI_n)^{-1}) = \psi m_1/\phi > 0.
\end{align*}
See Lemma \ref{lemma:l1_limits} and its proof for more details. 
For instance, as argued in \cite{pennington2017nonlinear,adlam2019random}, these can be reduced to a quartic equation for $m_1$ and are convenient to solve numerically.
However, the existence of these limits does not imply our results; on the contrary, the proofs of our results require extensive additional calculations and several novel ideas.

\begin{theorem}
\label{thm:general_ell_simplified}
    If Conditions \ref{cond:limit}-\ref{cond:tehe} are satisfied, 
    and we have $c_1,\cdots,c_\ell \neq 0$,
    as well as
    $\eta \asymp n^\alpha$ with $\frac{\ell - 1}{2\ell} < \alpha < \frac{\ell}{2\ell + 2}$,
     then
     for the learned feature map $\bF$ and the untrained feature map $\bF_0$, we have $$\gL_{\rm tr}(\bF_0) - 
         \gL_{\rm tr}(\bF) \to_P \Delta_\ell\ge 0,$$ where $\Delta_\ell$ can be found in Section \ref{sec:general_ell}. 
\end{theorem}
The  expression for $\Delta_\ell$ is complex and given in Section \ref{sec:general_ell} due to space limitations. 
For a better understanding of Theorem \ref{thm:general_ell_simplified}, we consider two specific cases,
$\ell =1$ and $\ell = 2$.
\begin{corollary} 
\label{corollay}
Under the assumptions of Theorem \ref{thm:general_ell_simplified}, for $\ell = 1$, we have $\gL_{\rm tr}(\bF_0) -
         \gL_{\rm tr}(\bF)  \to_P \Delta_1$ with
\begin{align}
    \Delta_1 := \frac{\psi\lambda c_{\star,1}^4 m_2}{\phi[c_{\star,1}^2 + \phi(c_{\star}^2 + \sigma_\ep^2)]} \ge 0.
\end{align}
For $\ell = 2$, we have $\gL_{\rm tr}(\bF_0) - \gL_{\rm tr}(\bF)  \to_P \Delta_2 $ with
\begin{align}
    \Delta_2  := \Delta_1 
         + \frac{4\psi\lambda c_{\star,1}^4c_{\star,2}^2m_1}
    {3\phi[\phi(c_\star^2 + \sigma_\ep^2) + c_{\star,1}^2]^2} \ge 0.
\end{align}
\end{corollary}

The above result shows that after one gradient step with sufficiently large step size, the model 
can fit nonlinear components of the teacher function. 
This is impossible with a small step size. 
For example, when $\ell = 1$, the improvement in the loss is 
\emph{increasing in the strength of the linear component $c_{\star, 1}$}, keeping the signal strength $c_{\star}$ fixed.
This is not the case for the 
strength of the non-linear component $c_{\star, >1}^2 = c_{\star}^2-c_{\star, 1}^2$. 

When we further increase the step size to the $\ell = 2$ regime, the loss of the trained model will drop by an additional positive value,
depending on the strength $c_{\star, 2}$ of the quadratic signal, 
which shows that the quadratic component of the target function is  being fit. 
Also, note that if $c_{\star, 1} = 0$; i.e., if the \textit{information exponent} (the index of the first non-zero Hermite coefficient) of $\sigma_\star$ is greater than one, 
the gradient step does not change the limiting loss. 
In this case, according to \eqref{prop:alignment}, the alignment between the learned direction $\vbeta$ and the true direction $\vbeta_\star$ will converge to zero. It is known that learning single-index functions with information exponent greater than one requires a sample size 
of order larger than $d$
\cite{dandi2023learning}.

The limiting value of the test error can be analyzed similarly. 
\begin{theorem}
    \label{thm:test_risk}
     Let Conditions \ref{cond:limit}-\ref{cond:tehe}
     and the assumptions of Theorem \ref{thm:test_equivalence} hold.
     If $c_1 \neq 0$, then for $\ell = 1$, we have $\gL_{\rm te}(\hat\va(\bF_0)) -
         \gL_{\rm te}(\hat\va{(\bF)})  \to_P \Lambda_1$ with
    \begin{align}
        \Lambda_1 := \frac{c_{\star,1}^4 \Gamma_1}{[c_{\star,1}^2 + \phi(c_{\star}^2 + \sigma_\ep^2)]} \ge 0, 
    \end{align}
    where $\Gamma_1$ does not depend on the target function.

    If further $c_2\neq 0$, then for $\ell = 2$, we have $\gL_{\rm te}(\hat\va(\bF_0)) -
         \gL_{\rm te}(\hat\va{(\bF)})  \to_P \Lambda_2$ with
    \begin{align}
        \Lambda_2 := \Lambda_1 + \frac{c_{\star,1}^4 c_{\star,2}^2\Gamma_2}{[c_{\star,1}^2 + \phi(c_{\star}^2 + \sigma_\ep^2)]^2} \ge 0,
    \end{align}
    where $\Gamma_2$ does not depend on the target function. The complete expressions for $\Gamma_1$ and 
     $\Gamma_2$ can be found in \eqref{eq:final_l1_test} and \eqref{eq:final_l2_test}, respectively.
\end{theorem}
Similar to the training loss result, when $\ell = 1$, the improvement in the test error is increasing in $c_{\star, 1}$; keeping the signal strength $c_\star$ fixed. Moreover, the improvement in the test error for $\ell = 2$ depends on the strength $c_{\star, 2}$ of the quadratic signal, showing that the nonlinear component is being learned.

\paragraph{Proof Idea.} Using Theorem \ref{thm:training_equivalence} and \ref{thm:test_equivalence},  in the expression for the training and test error, we replace the trained features matrix $\bF$ with its approximation $\bF_\ell$ from Theorem~\ref{thm:spectrum_of_feature_matrix}. Then, by applying the Woodbury formula, we express the training and test errors in terms of $\bar \bR_0 = (\bF_0 \bF_0^\top + \lambda n \bI_n)^{-1}$ and the non-linear spikes from Theorem~\ref{thm:spectrum_of_feature_matrix}. Using the Gaussian equivalence property (Appendix \ref{sec:gec}) for the untrained features $\bF_0$, we show that the  interaction between the first $\ell$ Hermite components of $\tilde\vy$ and the spike terms will result in non-vanishing terms corresponding to learning different components of the target function. Finally, we compute the limiting value of these terms in terms of $m_1, m_2$ and their derivatives using tools from random matrix theory.

\subsection{Staircase Property}
Recently, \cite{abbe2021staircase,abbe2022merged} show that when learning Boolean functions, under certain conditions on the teacher function, a two-layer trained by SGD in
the mean-field regime will learn the target function incrementally; i.e., Fourier coefficients of higher order are sequentially learned over time. \cite{berthier2023learning} study the problem of learning a single-index function using a wide two-layer neural network trained using gradient flow, and show that in a specific training setting where the stepsizes for the first layer are much smaller than those of the second layer, the decrease rate of training error is non-monotone; there are long plateaus where there is barely any progress, and there are intervals of rapid decrease.

In the one-step updated two layer neural network, we  observe a similar phenomenon. Theorems~\ref{thm:test_risk} and Theorem~\ref{thm:general_ell} show that given $\ell \in \mathbb{N}$,  the errors of the trained model is 
asymptotically constant for all $\eta = c n^\alpha$ with
$\frac{\ell - 1}{2\ell} < \alpha < \frac{\ell}{2\ell + 2}$ and $c \in \R$. There are sharp jumps at the edges between regimes of $\alpha$, 
whose size is precisely characterized above. This shows that a non-monotone rate of decrease in training and test error can also be seen after one step of gradient descent, as a function of step size. For an illustration of this phenomenon, see Figure~\ref{fig:three_figures} \textbf{(Right)}.

\section{Numerical Simulations}
\label{sec:numerical}

To support and illustrate our theoretical results, we present some numerical simulations.
We use the shifted ReLU activation $\sigma(x) = \max(x, 0) - 1 / \sqrt{2\pi}$, $n = 1000$, $N = 500$, $d = 300$, and the regularization parameter $\lambda = 0.01$.

\begin{figure*}
    \centering
    \includegraphics[width = \textwidth]{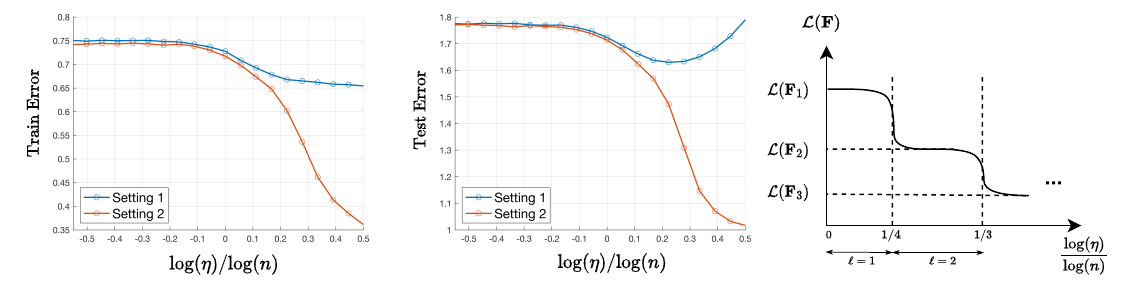}
    \caption{\textbf{(Left, Middle) }Training and test errors after one gradient as functions of $\log (\eta) / \log (n)$. \textbf{(Right)} A toy plot illustrating the theoretical training/test error curve as a function of $\log(\eta) / \log(n)$.}
    \label{fig:three_figures}
\end{figure*}
\paragraph{Singular Value Spectrum of $\bF$.}
We let the 
 the teacher function be $f_\star(\vx) = H_1(\vbeta_\star^\top \vx) + H_2(\vbeta_\star^\top \vx),$
set the noise variance
$\sigma_\ep^2 = 0.5$,
and the step size to $\eta = n^{0.29}$, so $\ell = 2$.
We plot the histogram of singular values of the updated feature matrix $\bF$.
In Figure \ref{fig:detailed-spikes}, we see two spikes corresponding to $\tilde \bX \vbeta, (\tilde \bX \vbeta)^{\circ 2}$ as suggested by Theorem \ref{thm:spectrum_of_feature_matrix} and \ref{thm:subspace}.
Since $f_\star$ has a linear component $H_1$ and a quadratic component $H_2$, these spikes will lead to feature learning.

\vspace{-0.2cm}
\paragraph{Quadratic Feature Learning.}
To support the findings of Corollary \ref{corollay} and Theorem \ref{thm:test_risk} for $\ell = 2$, we consider the following two settings:

\begin{align*}
    &\textbf{Setting 1}: y = H_1(\vbeta_\star^\top \vx) + \varepsilon, \quad \varepsilon \sim \normal(0, 1),\\ &\textbf{Setting 2}: y = H_1(\vbeta_\star^\top \vx) + \frac{1}{\sqrt{2}} H_2(\vbeta_\star^\top \vx).
\end{align*}
\vspace{-0.7cm}

Note that $c_{\star, 1}$ and $c_{\star} + \sigma_\ep^2$ are same in these two settings.
This ensures that the improvement due to learning the linear component is the same.
We plot the training and test errors of the two-layer neural networks trained with the procedure described in Section~\ref{sec:problemsetting} as functions of $\log(\eta)/\log(n)$.
In Figure \ref{fig:three_figures} \textbf{(Left)}, we see that the errors decrease in the range $\log (\eta) / \log (n) \in (0, \frac{1}{4})$ as the model learns the linear component $H_1(\vbeta_\star^\top \vx)$.
In the range $\log (\eta) / \log (n) \in (\frac{1}{4}, \frac{1}{3})$, the model starts to learn the quadratic feature.
However since the quadratic feature is not present in Setting 1, the errors under the two settings diverge. Although the proofs reveal that the convergence rates of the training/test errors after one step can be slow,
these results are consistent with Corollary \ref{corollay} and Theorem \ref{thm:test_risk}.

\section{Conclusion}

We have studied feature learning in two-layer neural networks under one-step gradient descent with step size $\eta \asymp n^\alpha, \alpha \in (0, \frac{1}{2})$. 
We showed that if $\alpha \in ((\ell - 1)/(2\ell), \ell/(2\ell + 2))$, 
the gradient update will add $\ell$ separated singular values to the initial feature matrix spectrum corresponding to different nonlinear features. We then proved equivalence theorems and used them to derive a precise characterization of the training and test errors in the high-dimensional proportional limit. Using this, we showed that in certain examples, non-linear components of the teacher function are learned.

\vspace{-0.3cm}
\paragraph{Future Work.}
In this paper, we only study the problem when $\eta \asymp n^\alpha$ with $\alpha \in ((\ell - 1)/(2\ell), \ell/(2\ell + 2))$. The boundary case where $\eta \asymp n^{(\ell - 1)/(2\ell)}$ is an interesting problem and is left as future work. 

Also, following prior work in the area (see e.g., \cite{damian2022neural, ba2022high, dandi2023learning, nichani2023provable,wang2023learning}, etc.), we use sample splitting in our two-step training procedure. Although this setting is natural for the analysis of pretrained models, it does not cover the case where feature learning and ridge regression use the same
data. We leave this setting as a future direction. 

In this paper, we focused on the case where $\alpha < 1/2$. The behaviour of the feature matrix can be significantly different when $\alpha = 1/2$. When $\alpha < 1/2$, as proved in this paper, the spectrum of the feature matrix will consist of a finite number of spikes added to the spectrum of the untrained feature matrix. However, when $\alpha = 1/2$, the behaviour can deviate from the spectrum of the untrained model in other ways. Note that according to Theorem~\ref{thm:spectrum_of_feature_matrix}, the number of spikes in the spectrum of $\mathbf{F}$ will increase as we increase $\alpha$ from $0$ to $1/2$ and will diverge as we approach $1/2$. The limiting empirical singular-value distribution of the feature matrix and the training and test errors of the network  when $\alpha = 1/2$ is an open problem and we leave it as future work.

\section*{Acknowledgements}
The work of Behrad Moniri is supported by The Institute for Learning-enabled Optimization at Scale (TILOS), under award number NSF-CCF-2112665.  Donghwan Lee and Edgar Dobriban were supported in part by ARO W911NF-20-1-0080, ARO W911NF-23-1-0296, NSF 2031895, NSF DMS 2046874, ONR N00014-21-1-2843, and the Sloan Foundation.  Hamed Hassani, Donghwan Lee and Behrad Moniri are also supported by EnCORE: Institute for Emerging CORE Methods in Data Science, under NSF award number 2217058. The authors would like to thank Doan Dai Nguyen and Simone Bombari for bringing to our attention a small mistake in the proof of Theorem~\ref{thm:spectrum_of_feature_matrix} in the previous version of the paper which has now been fixed.

{\small
\bibliography{full_references}
\bibliographystyle{alpha}
}

\appendix
\onecolumn

\section{Additional Related Work}
\label{sec:additional_related}
\cite{goel2019learning} provide a polynomial time algorithm that learns neural networks with two non-linear layers. 
Our setting is different because we do not apply a non-linear activation after the second layer.
\cite{chen2022hardness} show that learning two-hidden-layer neural networks from noise-free Gaussian data requires superpolynomially many statistical queries. \cite{zhenmei2022theoretical} show that neural networks trained by gradient
descent can succeed on problems where the labels are determined by a set of class-relevant patterns and if these patterns are removed, no polynomial algorithm in the Statistical Query model can learn even weakly.

\section{Additional Notation and Terminology}

In the appendix, we use the following additional notations.
We let 
$\mathbb{N}_0 = \{0,1,2,\ldots\}$ be the set of non-negative integers.
For a set $X$ and $x_1,x_2\in X$,
$\delta_{x_1, x_2}$ is the Kronecker delta, which equals unity if $x_1=x_2$, and is zero otherwise.
We use $\tilde O(\cdot)$ for the standard big-O notation up to logarithmic factors in $n$.
For a positive integer $k$, $k!!$
is the product of all the positive integers up to $n$ with the same parity as $n$.
For two random quantities 
$X,Y$,
$X \independent Y$
denotes that $X$ is independent of $Y$.
By orderwise analysis, we mean bounding a term by the triangle inequality
and the inequality $\|Ab\|_2 \le \|A\|_{\mathrm{op}}\|b\|_2$ for a conformable matrix-vector pair $A,b$,
to reduce it to operator norms of matrices and Euclidean norms of vectors, and then use simple bounds for those quantities. 
Constants such as $C,c'$, etc., can change from line to line unless specified otherwise.
For two random quantities $A,B$, $A=_dB$ denotes that $A$ and $B$ have the same distribution.
Limits of random variables are understood in probability.
For two matrices $\mathbf{A}, \mathbf{B}$ with equal shape, we write $\mathbf{A} \circ \mathbf{B}$ to denote their entry-wise (Hadamard) product.

We denote
$\bX \vbeta = \vtheta$,
$\tilde \bX \vbeta = \tilde\vtheta$,
$\bX \vbeta_\star = \vtheta_\star$,
and $\tilde \bX \vbeta_\star = \tilde\vtheta_\star$.
 We also define 
 $\bar\bR_0 = (\bF_0 \bF_0^\top + \lambda n \bI_n)^{-1}$
 and
 $\bR_0 = (\bF_0^\top \bF_0 + \lambda n \bI_N)^{-1}$.

\section{Basic Lemmas}
\begin{lemma}[Orthogonality of Hermite polynomials]\label{lem:twohermite}
    Let $(Z_1, Z_2)$ be jointly Gaussian with $\E[Z_1] = \E[Z_2] =0$, $\E[Z_1^2] = \E[Z_2^2] = 1$, and $\E[Z_1 Z_2] = \rho$.
    Then for any $k_1, k_2 \in \mathbb{N}_0$,
    \begin{align*}
        \E [H_{k_1}(Z_1) H_{k_2}(Z_2)] = k_1! \rho^{k_1} \delta_{k_1, k_2}.
    \end{align*}
    In particular, if for some positive integer $d$,
    $\mathbf{Z} \sim \normal(0, \bI_d)$, and if $\va,\vb \in \mathbb{S}^{d - 1}$, then 
    \begin{align*}
        \E [H_{k_1}(\va^\top \mathbf{Z}) H_{k_2}(\vb^\top \mathbf{Z})] = k_1! (\va^\top \vb)^{k_1} \delta_{k_1, k_2}.
    \end{align*}
\end{lemma}
\begin{proof}
    See \cite[Chapter 11.2]{o2014analysis}.
\end{proof}

\begin{lemma}[Taylor expansion of Hermite polynomials]\label{lem:hermitesum}
    For any $k \in \mathbb{N}_0$ and $x, y \in \mathbb{R}$,
    \begin{align*}
        H_k(x  + y) = \sum_{j = 0}^k \binom{k}{j} x^j H_{k -j}(y).
    \end{align*}
\end{lemma}
\begin{proof}
    Note that $\frac{d}{dx} H_k(x) = k H_{k - 1}(x)$ \cite[Equation 22.8.8]{abramowitz1968handbook} and thus $\frac{d^j}{dx^j} H_k(x) = \frac{k!}{(k - j)!} H_{k - j}(x)$.
    By Taylor expanding $H_k(x + y)$ at $y$, we find
    \begin{align*}
        H_k(x + y) = \sum_{j = 0}^k \frac{x^j}{j!} \frac{d^j}{dy^j} H_k(y) = \sum_{j = 0}^k \binom{k}{j} x^j H_{k - j}(y).
    \end{align*}
\end{proof}

The following Lemma, proved in Section \ref{pflem:Ms}, 
provides several bounds used in the proofs.
\begin{lemma}\label{lem:Ms}
    Under  Conditions \ref{cond:limit}-\ref{cond:tehe},
    there exists $C > 0$ such that the following holds with probability $1 - o(1)$.
    \begin{enumerate}
        \item[(a)] $M_{\va} := \max_{1 \leq i \leq N} |a_i| \leq Cn^{-\frac{1}{2}} \log^\frac{1}{2} n$,

        \item[(b)] $M_{\vbeta} := \max_{1 \leq i \leq n} |\langle \tilde \vx_i, \vbeta \rangle| \leq C \log^\frac{1}{2} n$,
        
        \item[(c)] $M_{\bW_0} := \sup_{k \geq 1} \Vert (\bW_0 \bW_0^\top)^{\circ k} \Vert_\textnormal{op} \leq C$,

        \item[(d)] $\Vert \tilde \bX \Vert_\textnormal{op} \leq C \sqrt{n} $.
    \end{enumerate}
\end{lemma}

\section{Proof of \eqref{prop:alignment}}
    By Lemma \ref{lemma:turn_beta_to_beta_star} with $\mathbf{v} = \vbeta_\star$ and $\mathbf{D} = \bI_d$, we have
    \begin{align*}
        &\vbeta_\star^\top \vbeta \to_P c_{\star, 1} \Vert \vbeta_{\star} \Vert_2^2 = c_{\star, 1},\quad
        \Vert \vbeta \Vert_2^2 = \vbeta^\top \vbeta \to_P \phi (c_\star^2 + \sigma_\ep^2) + c_{\star, 1}^2 \vbeta_\star^\top \vbeta_\star = c_{\star, 1}^2+ \phi(c_\star^2 + \sigma_\ep^2).    
    \end{align*}
    By the continuous mapping theorem, we conclude
    \begin{align*}
        \frac{|\vbeta_\star^\top \vbeta|}{\Vert \vbeta_\star \Vert_2 \Vert \vbeta \Vert_2} \to_P \frac{|c_{\star, 1}|}{\sqrt{c_{\star, 1}^2+ \phi(c_\star^2 + \sigma_\ep^2)}}.
    \end{align*}

\section{Proof of Theorem \ref{thm:spectrum_of_feature_matrix}}
    
    Recall that $\mathbf{G} = c_1  \va \vbeta^\top + \mathbf{\Delta}$ where $\|\mathbf{\Delta}\|_{\rm op} = \tilde{O}_{\sP}(1/\sqrt{N})$. The feature matrix is given by
    \begin{align*}
        \bF = \sigma(\tilde \bX \bW^\top) + \sigma(\tilde \bX (\bW_0 + \eta\mathbf{\Delta})^\top + c_1 \eta \tilde\bX\vbeta \va^\top).
    \end{align*}
    Note that we have  $\|\bW_0\|_{\rm op} = \Theta_{\sP}(1)$. As long as $\eta = o(\sqrt{N})$, we have $\|\eta\mathbf{\Delta}\|_{\rm op} = {o}_{\sP}(1)$, $\|\bW_0 + \eta\boldsymbol{\Delta}\|_{\rm op} = \|\bW_0\|_{\rm op} + o_\sP(1)$ and $\|((\bW_0+\boldsymbol{\Delta})(\bW_0+\boldsymbol{\Delta}))^{\circ j \top}\|_{\rm op} = \|(\bW_0\bW_0^\top)^{\circ j}\|_{\rm op} + o_{\sP}(1)$. With these, it can easily be verified that the effect of $\mathbf{\Delta}$ in the following arguments is negligible (see \cite[Lemma 12]{dandi2023learning} for more details).
    
    We consider any fixed $\bW_0$ such that the event 
    $\Omega = \{\sup_{k \geq 1} \Vert (\bW_0 \bW_0^\top)^{\circ k} \Vert_\textnormal{op} \leq C\}$
    from Lemma \ref{lem:Ms} (c) holds.
    By Lemma \ref{lem:twohermite}, each row of $H_j(\tilde \bX \bW_0^\top)$ has second moment matrix
    \begin{align*}
        \E_{\vx \sim \normal(0, \bI_d)}[H_j(\bW_0 \vx) H_j(\bW_0 \vx)^\top] = j! (\bW_0 \bW_0^\top)^{\circ j},
    \end{align*}
    whose operator norm is $O(j!)$ on $\Omega$.
     Thus by \cite[Theorem 5.48]{vershynin2010introduction} and Markov's inequality, for any $j \in [L]$, 
     for $t\ge(Cnj!)^{1/2}$,
     and with
     $M = \E \max_{i=1}^n \|H_j(\bW_0 \tilde\vx_i)\|^2$,
     $\delta = C \sqrt{M \log \min(n,N)}$,
      \begin{align*}
      P(\Vert H_j(\tilde \bX \bW_0^\top) \Vert_\textnormal{op}\ge t) 
      &\le
        P\left(\|H_j(\tilde \bX \bW_0^\top)  H_j(\tilde \bX \bW_0^\top) ^\top/n-j! (\bW_0 \bW_0^\top)^{\circ j}\|_\textnormal{op} \ge 
        t^2/n-Cj! \right)\\[0.5cm]
        & \le
        \frac{\mathbb{E}\|H_j(\tilde \bX \bW_0^\top)  H_j(\tilde \bX \bW_0^\top) ^\top/n-j! (\bW_0 \bW_0^\top)^{\circ j}\|_\textnormal{op}}{t^2/n-Cj!}\\
        &\le
        \frac{\delta\max((Cj!)^{1/2},\delta)}{t^2/n-Cj!}.
    \end{align*}
    Next, we observe that 
    since $H_j$ is a $j$-th degree polynomial and the normal absolute moments increase with $j$,
    $M = \E \max_{i=1}^n \|H_j(\bW_0 \tilde\vx_i)\|^2 \le 
    C_j \E \max_{i=1}^n \|(\bW_0 \tilde\vx_i)^{\circ j}\|^2$.
    Now, note that for any vectors $\vx_1, \vx_2$, we have $\Vert \vx_1 \circ \vx_2\Vert^2 \leq \Vert\vx_1\Vert^2 \Vert\vx_2\Vert^2$ by simply expanding the norms. Thus, on the event $\Omega$, one can verify that 
    for all $\vx$,
    $\|(\bW_0 \vx)^{\circ j}\|^2 \le C'_j \|\vx\|^{2j} $ for some $C'_j>0$.
   Also, we have that $A_i  = \|\tilde\vx_i\|^{2j}/N$ for $i\in [n]$ are sub-Weibull random variables with tail parameter $1/(2j)$, (see e.g., \cite{vladimirova2020sub,zhang2020concentration}).
    Thus, by the maximal inequality for sub-Weibull random variables \cite[Proposition A.6 and Remark A.1]{kuchibhotla2022moving}, it follows that 
    for all $j\ge 1$,
    there is $C_j>0$ such that
    $\E  \max_{i=1}^n A_i \le C_j(\log n)^{2j}$.
Hence, $M \le C_j'' N (\log n)^{2j}$.

    Thus, choosing $t = C'\sqrt{n j!} (\log n)^j$ for sufficiently large $C' $ leads to 
    \begin{align}\label{eqn:Hjbound}
        \Vert H_j(\tilde \bX \bW_0^\top) \Vert_\textnormal{op} = O\left( \sqrt{n j!}\, (\log n)^j \right)
    \end{align}
    with probability $1 - o(1)$.

Define,
for all $z\in \R$,
$ \sigma_L(z) = \sum_{k = 0}^L c_k H_k(z)$, where $L = \max\bigl\{ \ell, \frac{\log n}{4(\ell + 1) \log\log n} \bigr\}$.
Each row of $(\sigma - \sigma_L)(\tilde \bX \bW_0^\top)$ has second moment matrix
    \begin{align*}
        \E_{\vx \sim \normal(0, \bI_d)}[(\sigma - \sigma_L)(\bW_0 \vx)(\sigma - \sigma_L)(\bW_0 \vx)^\top] = \sum_{k = L + 1}^\infty k! c_k^2 (\bW_0 \bW_0^\top)^{\circ k},
    \end{align*}
    whose operator norm is $O(L^{-\frac{1}{2} - \omega})$ by Lemma \ref{lem:Ms} (c) and Condition \ref{cond:he}.
    Therefore, 
    \begin{align}\label{eqn:truncatebound-1}
        \Vert (\sigma - \sigma_L)(\tilde \bX \bW_0^\top) \Vert_\textnormal{op} = O(\sqrt{n \log n} L^{-\frac{1}{2} - \omega}) = o(\sqrt{n})
    \end{align}
    with probability $1 - o(1)$.
    Since $\eta = o(\sqrt{n})$, the rows of have $\bW$ norm of $O_\sP(1)$.
    Thus, we can repeat the same argument to show that with probability $1 - o(1)$, we have
    \begin{align}\label{eqn:truncatebound-2}
        \Vert (\sigma - \sigma_L)(\tilde \bX \bW^\top) \Vert_\textnormal{op} = O(\sqrt{n \log n} L^{-\frac{1}{2} - \omega}) = o(\sqrt{n}).
    \end{align}
    
    Let $\bF^{(L)} := \sigma_L(\tilde \bX \bW^\top)$ and $\bF^{(L)}_0 := \sigma_L(\tilde \bX \bW^\top_0)$. We can write
    \begin{align*}
        \bF^{(L)} = \bF^{(L)}_0 + \sum_{k = 1}^L c_k (H_k(\tilde\bX \bW^\top) - H_k(\tilde\bX \bW_0^\top)).
    \end{align*}
    By Lemma \ref{lem:hermitesum}, using  $\bW = \bW_0 + \eta\, \bG$ so that
    $\tilde \bX \bW^\top = \tilde \bX \bW_0^\top +\eta\tilde \bX \bG^\top$,
    and using that $H_0(z)=1$ for all $z\in\R$,
    \begin{align*}
        H_k(\tilde \bX \bW^\top) &- H_k(\tilde \bX \bW_0^\top) = \eta^k (\tilde \bX \bG^\top)^{\circ k} + \sum_{j = 1}^{k - 1} \binom{k}{j} \eta^{j}  H_{k - j}(\tilde \bX \bW_0^\top) \circ (\tilde \bX \bG^\top)^{\circ j}.
    \end{align*}    
    Therefore, setting $\mathbf{G} = c_1 \va\vbeta^\top$, we have
    \begin{align*}
        \bF^{(L)} &= \bF^{(L)}_0 + \sum_{k = 1}^\ell c_1^k c_k \eta^k (\tilde \bX \vbeta)^{\circ k} (\va^{\circ k})^\top + \underbrace{\sum_{k = \ell + 1}^L  c_1^k c_k \eta^k (\tilde \bX \vbeta)^{\circ k} (\va^{\circ k})^\top}_{\bDelta_1}\\
        &+ \underbrace{\sum_{k = 1}^L \sum_{j = 1}^{k - 1} c_1^j c_k \binom{k}{j} \eta^{j}  H_{k - j}(\tilde \bX \bW_0^\top) \circ \bigl[(\tilde \bX \vbeta)^{\circ j} (\va^{\circ j})^\top\bigr]}_{\bDelta_2}.
    \end{align*}
    We will show that each of $\Vert \bDelta_1 \Vert_\textnormal{op}, \Vert \bDelta_2\Vert$ is $o(\sqrt{n})$ with probability $1-o(1)$. By Lemma \ref{lem:Ms} (a) and (b), we have
    using that $\alpha < \frac{\ell}{2\ell + 2}$,
    \begin{align*}
        \Vert \bDelta_1 \Vert_\textnormal{op} &\leq \sum_{k = \ell + 1}^L  c_1^k c_k \eta^k \Vert (\tilde\bX \vbeta)^{\circ k} \Vert_2 \Vert \va^{\circ k} \Vert_2 
        \le \sum_{k = \ell + 1}^L c_1^k c_k \eta^k n M_{\va}^k M_{\vbeta}^k 
        = \tilde O(n (\eta / \sqrt{n})^{\ell + 1}) = o(\sqrt{n})
    \end{align*}
    with probability $1 - o(1)$.
    Also, since
    \begin{align*}
        \Vert H_{k - j} (\tilde \bX \bW_0^\top) \circ (\tilde \bX \vbeta)^{\circ j} (\va^{\circ j})^\top \Vert_\textnormal{op} &\leq (M_{\va} M_{\vbeta})^j \sqrt{(k - j)!} n^\frac{1}{2} \log^{k - j} n\\[0.2cm]
        &\leq C^j \sqrt{(k - j)!} n^{-\frac{j}{2} + \frac{1}{2}} \log^k n,
    \end{align*}
    we also have
    \begin{align*}
        \Vert \bDelta_2 \Vert_\textnormal{op} &\leq \sum_{k = 1}^L \sum_{j = 1}^{k - 1} c_k C^j \binom{k}{j} \sqrt{(k - j)!} \eta^{j} n^{-\frac{j}{2} + \frac{1}{2}} \log^k n= \tilde O(\eta) = o(\sqrt{n}).
    \end{align*}
This proves that with probability $1 - o(1)$, we have $\bF^{(L)} = \bF_0^{(L)} + \sum_{k = 1}^\ell c_1^k c_k \eta^k (\tilde \bX \vbeta)^{\circ k} (\va^{\circ k})^\top +\bDelta$, with $\Vert\bDelta\Vert_{\rm op} = o(1)$. This, alongside \eqref{eqn:truncatebound-1} and \eqref{eqn:truncatebound-2}, concludes the proof.

\section{Proof of Theorem \ref{thm:subspace}}
By Theorem \ref{thm:spectrum_of_feature_matrix}, 
letting $\mathbf{E} = \bF_0 + \bDelta$, 
we have $\Vert \mathbf{E} \Vert_\textnormal{op} = O_\sP(\sqrt{n})$.
Note that $\sum_{k = 1}^\ell c_1^k c_k \eta^k (\tilde \bX \vbeta)^{\circ k}(a^{\circ k})^\top$ has rank $\ell$ almost surely and its left singular vector space is $\operatorname{span}\{(\tilde \bX \vbeta)^{\circ k}\}_{k \in [\ell]}$.
Also, the subspace spanned by the top-$\ell$ left singular vectors of $\bF$ is $\mathcal{F}_\ell$.
By Wedin's theorem \cite{wedin1972perturbation}, \cite[Theorem 2.9]{chen2021spectral},
and as $\alpha>\frac{\ell - 1}{2\ell}$,
we have
\begin{align*}
    d(\mathcal{F}_\ell, \operatorname{span}\{(\tilde \bX \vbeta)^{\circ k}\}_{k \in [\ell]}) = O_\sP \left( \frac{\Vert \mathbf{E} \Vert_\textnormal{op}}{\eta^\ell n^{\frac{1}{2} - \frac{\ell - 1}{2}}-\Vert \mathbf{E} \Vert_\textnormal{op}} \right) 
    = O_\sP(n^{\frac{\ell - 1}{2} - \alpha \ell}) = o_\sP(1),
\end{align*}
which concludes the proof.
\section{Proof of Theorem \ref{thm:training_equivalence}}
By the definition of $\hat\va(\bF)$, we have
    \begin{align*}
        \max \left\{ \frac{1}{n} \Vert \tilde\vy - \bF \hat{\va}(\bF) \Vert_2^2, \lambda \Vert \hat{\va}(\bF) \Vert_2^2\right\} &\leq \frac{1}{n} \Vert \tilde\vy - \bF \hat{\va}(\bF) \Vert_2^2 + \lambda \Vert \hat{\va}(\bF) \Vert_2^2\\
        &\leq \frac{1}{n} \Vert \tilde\vy - \bF \cdot \vzero \Vert_2^2 + \lambda \Vert \vzero \Vert_2^2 = \frac{1}{n} \Vert \tilde\vy \Vert_2^2 = O_\sP(1).
    \end{align*}
    Thus,
    \begin{align}\label{eqn:Fbound}
        \Vert \hat{\va}(\bF) \Vert_2 = O_\sP(1), \quad \Vert \tilde\vy - \bF \hat{\va}(\bF) \Vert_2 = O_\sP(\sqrt{n}).
    \end{align}
    A similar argument gives
    \begin{align}\label{eqn:tildeFbound}
        \Vert \hat{\va}(\bF_\ell) \Vert_2 = O_\sP(1), \quad \Vert \tilde\vy - \bF_\ell \hat{\va}(\bF_\ell) \Vert_2 = O_\sP(\sqrt{n}).
    \end{align}
    Also, by the triangle inequality, and using \eqref{eqn:tildeFbound} and Theorem \ref{thm:spectrum_of_feature_matrix}, which states $\Vert \bF_\ell - \bF \Vert_\textnormal{op} = o_\sP(\sqrt{n})$, we have
    \begin{align}\label{eqn:crossbound-1}
        \Vert \tilde\vy - \bF \hat{\va}(\bF_\ell) \Vert_2 &\leq \Vert \tilde\vy - \bF_\ell \hat{\va}(\bF_\ell) \Vert_2 + \Vert (\bF_\ell - \bF) \hat{\va}(\bF_\ell) \Vert_2 \nonumber\\
        &\leq \Vert \tilde\vy - \bF_\ell \hat{\va}(\bF_\ell) \Vert_2 + \Vert \bF_\ell - \bF \Vert_\textnormal{op} \Vert \hat{\va}(\bF_\ell) \Vert_2 = O_\sP(\sqrt{n}).
    \end{align}

    Similarly, we can prove that
    \begin{align}\label{eqn:crossbound-2}
        \Vert \tilde\vy - \bF_\ell \hat{\va}(\bF) \Vert_2 = O_\sP(\sqrt{n}).
    \end{align}

    For $\va = \hat{\va}(\bF)$ or $\va = \hat{\va}(\bF_\ell)$,
    \begin{align*}
        \frac{1}{n} \left( \Vert \tilde\vy - \bF \va \Vert_2^2 - \Vert \tilde\vy - \bF_\ell \va \Vert_2^2\right) &= \frac{1}{n} \left\langle (\bF_\ell - \bF) \va,  \tilde\vy - \bF \va + \tilde\vy - \bF_\ell \va \right\rangle\\
        &\leq \frac{1}{n} \Vert \bF_\ell - \bF \Vert_\textnormal{op} \Vert \va \Vert_2 \left(\Vert \tilde\vy - \bF \va \Vert_2 + \Vert \tilde\vy - \bF_\ell \va \Vert_2\right) = o_\sP(1)
    \end{align*}
    by \eqref{eqn:Fbound}, \eqref{eqn:tildeFbound}, \eqref{eqn:crossbound-1}, \eqref{eqn:crossbound-2}, and Theorem \ref{thm:spectrum_of_feature_matrix}. Therefore,
    using the definition of $\hat{\va}(\bF_\ell)$,
    \begin{align*}
        \frac{1}{n}\Vert \tilde\vy - \bF_\ell \hat{\va}(\bF_\ell) \Vert_2^2 + \lambda \Vert \hat{\va}(\bF_\ell) \Vert_2^2  &\leq \frac{1}{n}\Vert \tilde\vy - \bF_\ell \hat{\va}(\bF) \Vert^2 + \lambda \Vert \hat{\va}(\bF) \Vert_2^2\\
        &= \frac{1}{n} \Vert \tilde\vy - \bF \hat{\va}(\bF) \Vert_2^2 + \lambda \Vert \hat{\va}(\bF) \Vert_2^2 + o_\sP(1)
    \end{align*}
    and using the definition of $\hat{\va}(\bF)$,
    \begin{align*}
        \frac{1}{n}\Vert \tilde\vy -  \bF \hat{\va}( \bF) \Vert_2^2 + \lambda \Vert \hat{\va}(\bF) \Vert_2^2 &\leq \frac{1}{n}\Vert \tilde\vy - \bF \hat{\va}(\bF_\ell) \Vert^2 + \lambda \Vert \hat{\va}(\bF_\ell) \Vert_2^2 \\
        &= \frac{1}{n} \Vert \tilde\vy - \bF_\ell \hat{\va}(\bF_\ell) \Vert_2^2 + \lambda \Vert \hat{\va}(\bF_\ell) \Vert_2^2 + o_\sP(1).
    \end{align*}
    These together prove the theorem.

\section{Proof of Theorem \ref{thm:test_equivalence}}

First, we will prove a general lemma regarding the equivalence of an augmented training loss. We will later use this result to prove the equivalence of the test error.
\begin{lemma} 
\label{lem:general_test_risk}
Let $\eta \asymp n^\alpha$ with $\frac{\ell - 1}{2\ell} < \alpha < \frac{\ell}{2\ell + 2}$ for some $\ell \in \mathbb{N}$ and $\bF_\ell$ be defined as in \eqref{eqn:Fell}.
For the test error $\gL_{\rm te}$ from Section \ref{eqthm},
define
\begin{align*}
    \gR_{\zeta}(\va, \bF) = \frac{1}{n}\Vert \tilde\vy - \bF \va\Vert_2^2 + \lambda \Vert \va\Vert_2^2  + \zeta \gL_{\rm te}(\va).
\end{align*}
Then, for any $\lambda > 0$, $\zeta > 0$, we have
    \begin{align}
        \left| \min_{\va} \gR_\zeta(\va, \bF_\ell) - \min_{\va} \gR_\zeta(\va, \bF) \right| = o(1),
    \end{align}    
with probability $ 1 - o(1)$.
\end{lemma}

\begin{proof}[Proof of Lemma \ref{lem:general_test_risk}]
Letting $\hat\va_\zeta (\bF) = \argmin_{\va} \gR_\zeta(\va, \bF)$, we can write
\begin{align*}
    \max \left\{ \frac{1}{n} \Vert \tilde\vy - \bF \hat\va_\zeta(\bF) \Vert_2^2, \lambda \Vert \hat\va_\zeta(\bF) \Vert_2^2, \zeta \gL_{\rm te}(\hat\va_\zeta(\bF))\right\}
    &\leq \frac{1}{n} \Vert \tilde\vy - \bF \hat\va_\zeta(\bF) \Vert_2^2 + \lambda \Vert \hat\va_\zeta(\bF) \Vert_2^2 + \zeta \gL_{\rm te}(\hat\va_\zeta(\bF))\\
    &\leq \frac{1}{n} \Vert \tilde\vy - \bF \cdot \vzero \Vert_2^2 + \lambda \Vert \vzero \Vert_2^2 +\zeta \gL_{\rm te}(\vzero) = O_\sP(1).
\end{align*}
Thus,
    \begin{align}\label{eqn:Fbound_extended}
        \gL_{\rm te}(\hat\va_\zeta(\bF)) = O_\sP(1), \quad \Vert \hat\va_\zeta(\bF)\Vert_2 = O_\sP(1), \quad \Vert \tilde\vy - \bF \hat\va_\zeta(\bF) \Vert_2 = O_\sP(\sqrt{n}).
    \end{align}
    A similar argument gives
    \begin{align}\label{eqn:tildeFbound_extended}
        \gL_{\rm te}(\hat\va_\zeta(\bF_\ell)) = O_\sP(1), \quad \Vert \hat\va_\zeta(\bF_\ell) \Vert_2 = O_\sP(1), \quad \Vert \tilde\vy - \bF_\ell \hat\va_\zeta(\bF_\ell) \Vert_2 = O_\sP(\sqrt{n}).
    \end{align}
    Also by the triangle inequality,
    \eqref{eqn:tildeFbound_extended} and Theorem \ref{thm:spectrum_of_feature_matrix}, which states $\Vert \bF_\ell - \bF \Vert_\textnormal{op} = o_\sP(\sqrt{n})$,
    \begin{align}\label{eqn:crossbound-1_extended}
        \Vert \tilde\vy - \bF \hat{\va}_\zeta(\bF_\ell) \Vert_2 &\leq \Vert \tilde\vy - \bF_\ell \hat{\va}_\zeta(\bF_\ell) \Vert_2 + \Vert (\bF_\ell - \bF) \hat{\va}_\zeta(\bF_\ell) \Vert_2 \nonumber\\
        &\leq \Vert \tilde\vy - \bF_\ell \hat{\va}_\zeta(\bF_\ell) \Vert_2 + \Vert \bF_\ell - \bF \Vert_\textnormal{op} \Vert \hat{\va}_\zeta(\bF_\ell) \Vert_2 = O_\sP(\sqrt{n}).
    \end{align}
Similarly, we can show that
    \begin{align}\label{eqn:crossbound-2_extended}
        \Vert \tilde\vy - \bF_\ell \hat{\va}_\zeta(\bF) \Vert_2 = O_\sP(\sqrt{n}).
    \end{align}

    For $\va = \hat{\va}_\zeta(\bF)$ or $\va = \hat{\va}_\zeta(\bF_\ell)$,
    \begin{align*}
        \frac{1}{n} \left( \Vert \tilde\vy - \bF \va \Vert_2^2 - \Vert \tilde\vy - \bF_\ell \va \Vert_2^2\right) &= \frac{1}{n} \left\langle (\bF_\ell - \bF) \va,  \tilde\vy - \bF \va + \tilde\vy - \bF_\ell \va \right\rangle\\
        &\leq \frac{1}{n} \Vert \bF_\ell - \bF \Vert_\textnormal{op} \Vert \va \Vert_2 \left(\Vert \tilde\vy - \bF \va \Vert_2 + \Vert \tilde\vy - \bF_\ell \va \Vert_2\right) = o_\sP(1)
    \end{align*}
    by \eqref{eqn:Fbound_extended}, \eqref{eqn:tildeFbound_extended}, \eqref{eqn:crossbound-1_extended}, \eqref{eqn:crossbound-2_extended}, and Theorem \ref{thm:spectrum_of_feature_matrix}. Therefore,
     using the definition of $\hat{\va}_\zeta(\bF_\ell)$,
    \begin{align*}
        \frac{1}{n}\Vert \tilde\vy - \bF_\ell \hat{\va}_\zeta(\bF_\ell) \Vert_2^2 &+ \lambda \Vert \hat{\va}_\zeta(\bF_\ell) \Vert_2^2 + \zeta \gL_{\rm te}(\hat{\va}_\zeta(\bF_\ell))\\
        &\leq \frac{1}{n}\Vert \tilde\vy - \bF_\ell \hat{\va}_\zeta(\bF) \Vert^2 + \lambda \Vert \hat{\va}_\zeta(\bF) \Vert_2^2+ \zeta \gL_{\rm te}(\hat{\va}_\zeta(\bF))\\
        &= \frac{1}{n} \Vert \tilde\vy - \bF \hat{\va}_\zeta(\bF) \Vert_2^2 + \lambda \Vert \hat{\va}_\zeta(\bF) \Vert_2^2+ \zeta \gL_{\rm te}(\hat{\va}_\zeta(\bF)) + o_\sP(1),
    \end{align*}
    and  using the definition of $\hat{\va}_\zeta(\bF)$,
    \begin{align*}
        \frac{1}{n}\Vert \tilde\vy -  \bF \hat{\va}_\zeta( \bF) \Vert_2^2 &+ \lambda \Vert \hat{\va}_\zeta(\bF) \Vert_2^2+ \zeta \gL_{\rm te}(\hat{\va}_\zeta(\bF))\\
        &\leq \frac{1}{n}\Vert \tilde\vy - \bF \hat{\va}_\zeta(\bF_\ell) \Vert^2 + \lambda \Vert \hat{\va}_\zeta(\bF_\ell) \Vert_2^2+ \zeta \gL_{\rm te}(\hat{\va}_\zeta(\bF_\ell)) \\
        &= \frac{1}{n} \Vert \tilde\vy - \bF_\ell \hat{\va}_\zeta(\bF_\ell) \Vert_2^2 + \lambda \Vert \hat{\va}_\zeta(\bF_\ell) \Vert_2^2 + \zeta \gL_{\rm te}(\hat{\va}_\zeta(\bF_\ell))+ o_\sP(1).
    \end{align*}
    Putting these together, we have
    \begin{align}
        |\min_{\va}\gR_\zeta(\va, \bF_\ell) -\min_{\va} \gR_\zeta(\va, \bF)| = o_\sP(1),
    \end{align}
    which concludes the proof.
    \end{proof}

Now, we use this lemma to prove the equivalence of the test error.
\begin{proof}[Proof of Theorem \ref{thm:test_equivalence}]
We will argue by contradiction. Assume that $\gL_{\bF} \neq \gL_{\bF_\ell}$ and let $\gL = \frac{1}{2}(\gL_{\bF} + \gL_{\bF_\ell})$. Now, consider the following two  optimization problems:
\begin{align*}
    &\gL_1 = \min_{\gL_{\rm te}(\va)\leq \gL\;}\frac{1}{n}\Vert \tilde\vy - \bF \va\Vert_2^2 + \lambda \Vert \va\Vert_2^2, \qquad
    \gL_2= \min_{\gL_{\rm te}(\va)\leq \gL}\;\frac{1}{n}\Vert \tilde\vy - \bF_\ell \va\Vert_2^2 + \lambda \Vert \va\Vert_2^2. 
\end{align*}

 Without loss of generality, assume that $\gL_\bF < \gL_{\bF_\ell}$. The solution of the first optimization problem will still converge to $\gL_{\rm tr}(\bF)$ because $\gL_{\bF} <\gL$. However, the solution of the second optimization problem will converge to a value greater than $\gL_{\rm tr}(\bF_\ell)$, because $\gL_{\bF_\ell} >\gL$ and the objective is $\lambda$-strongly convex. Note that by Theorem \ref{thm:training_equivalence}, we asymptotically have $\gL_{\rm tr}(\bF_\ell) = \gL_{\rm tr}(\bF) $. Thus $\gL_1$ and $\gL_2$  converge to  different quantities as $n\to\infty$. However, using the minimax theorem and since the objectives are $\lambda$-strongly convex, we can write
\begin{align*}
    &\gL_1 = \max_{\zeta>0} -\zeta \gL + \min_{\va}\left[\frac{1}{n}\Vert \tilde\vy - \bF \va\Vert_2^2 + \lambda \Vert \va\Vert_2^2 +  \zeta \gL_{\rm te}(\va)\right],\\
    &\gL_2= \max_{\zeta>0} -\zeta \gL + \min_{\va}\left[\frac{1}{n}\Vert \tilde\vy - \bF_\ell \va\Vert_2^2 + \lambda \Vert \va\Vert_2^2 +  \zeta \gL_{\rm te}(\va)\right].
\end{align*}
According to Lemma \ref{lem:general_test_risk}, the two minima above
converge to the same value for any fixed $\zeta$. 
Note that, as functions of $\zeta$,  both maxima are concave as they are minima of linear functions of $\zeta$. 
Hence, by using the concave version of \cite[Lemma 1]{abbasi2019universality}, we have that $\gL_1$ and $\gL_2$ converge to the same value, which is a contradiction.
\end{proof}

\section{Gaussian Equivalence Property}\label{sec:gec}
Gaussian equivalence results for non-linear random matrices
were introduced in \cite{el2010spectrum, cheng2013spectrum, fan2019spectral}.
They have been repeatedly used in recent studies of random feature models \cite{mei2022generalization, montanari2019generalization, adlam2020neural, adlam2020understanding, tripuraneni2021covariate, goldt2022gaussian, mel2021anisotropic, d2021interplay, loureiro2021learning, disagreement, hassani2022curse, hu2022universality, montanari2022universality}.
Also, there has been progress on proving the Gaussian equivalence property for a multi-layer network with only the final layer trained \cite{BoschPH23, cui23b}. 

In more distantly related  work in random matrix theory literature, 
the phenomenon that eigenvalue statistics in the bulk spectrum of a random matrix do not depend on the specific law of the matrix entries is referred to as ``bulk universality'' \cite{wigner1955characteristic, gaudin1961loi, mehta2004random, dyson1962brownian, erdHos2010bulk, erdHos2012bulk, el2010spectrum, tao2011random}.

\cite{erdos2019matrix} shows that local spectral laws of correlated random Hermitian matrices can be fully determined by their first and second moments, through the matrix Dyson equation.
Also, \cite{banna2015limiting, banna2020clt} show that spectral distributions of correlated symmetric random matrices can be characterized by Gaussian matrices with matching correlation structures.

In our case, we apply the Gaussian equivalence property to the following quantities for $p, q \in \mathbb{N}_0$ and $\vbeta_1, \vbeta_2 \in \{\vbeta, \vbeta_\star\}$: $H_p(\tilde \bX \vbeta_1)^\top \bar \bR_0 H_q(\tilde \bX \vbeta_2)$.

\subsection{Proof Sketch of Gaussian Equivalence Property}
In this section, we describe 
the proof idea of the Gaussian equivalence property.
We use the Lindeberg exchange method \cite{lindeberg1922neue} in which we replace each column $\vg_i = \sigma(\tilde \bX \vw_{0, i})$ of $\bF_0$ with its Gaussian equivalent $\tilde \vg_i = c_1 \tilde \bX \vw_{0, i} + c_{>1} \vz_i$, $\vz_i \stackrel{i.i.d.}{\sim} \normal(0, \bI_n)$.
Hereafter, we condition on all random variables except $\bW_0$.
Then, $H_p(\tilde \bX \vbeta_1)$ and $H_q(\tilde \bX \vbeta_2)$ become deterministic vectors with $O(1)$ entries.
We write $\vv = H_p(\tilde \bX \vbeta_1)$, $\vw = H_q(\tilde \bX \vbeta_2)$, 
and for all $i\in[N]$,
$\bM_i = \sum_{j = 1}^{i - 1} \tilde \vg_j \ve_j^\top + \sum_{j = i + 1}^N \vg_j \ve_j^\top$.

Let $\tilde \bF_0 = [\tilde \vg_1 \cdots \tilde \vg_N] \in \R^{n \times N}$ be the Gaussian equivalent of $\bF_0$ and let $\tilde \bR_0 = (\tilde \bF_0 \tilde \bF_0^\top + \lambda n \bI_n)^{-1}$.
By the triangle inequality,
\begin{align}\label{eqn:lindeberg}
    |\vv^\top \bar \E \bR_0 \vw \nonumber&- \vv^\top 
 \E \tilde \bR_0 \vw|\\ 
 &\leq \sum_{i = 1}^N |\vv^\top \E (\bM_i \bM_i^\top + \vg_i \vg_i^\top + \lambda n \bI_n)^{-1} \vw - \vv^\top \E(\bM_i \bM_i^\top + \tilde \vg_i \tilde \vg_i^\top + \lambda n \bI_n)^{-1} \vw|.
\end{align}
Defining $\bS_i = (\bM_i \bM_i^\top + \lambda n \bI_n)^{-1}$, we have
by the Sherman-Morrison formula that
\begin{align*}
    (\bM_i \bM_i^\top + \vg_i \vg_i^\top + \lambda n \bI_n)^{-1} &= \bS_i - \frac{\bS_i \vg_i \vg_i^\top \bS_i}{1 + \vg_i^\top \bS_i \vg_i}, \quad\textnormal{and}  \\[0.2cm]
    (\bM_i \bM_i^\top + \tilde\vg_i \tilde\vg_i^\top + \lambda n \bI_n)^{-1} &= \bS_i - \frac{\bS_i \tilde\vg_i \tilde\vg_i^\top \bS_i}{1 + \tilde\vg_i^\top \bS_i \tilde\vg_i}.
\end{align*}
Thus,
\begin{align*}
    \vv^\top \E(\bM_i \bM_i^\top + \vg_i \vg_i^\top + \lambda n \bI_n)^{-1} \vw &- \vv^\top \E(\bM_i \bM_i^\top + \tilde \vg_i \tilde \vg_i^\top + \lambda n \bI_n)^{-1} \vw\\ &= \E\frac{\vv^\top \bS_i \tilde \vg_i \tilde \vg_i^\top \bS_i \vw}{1 + \tilde \vg_i^\top \bS_i \tilde \vg_i} - \E\frac{\vv^\top \bS_i \vg_i \vg_i^\top \bS_i \vw}{1 + \vg_i^\top \bS_i \vg_i}.
\end{align*}
Let $\bSigma_\vg = \E \vg_i \vg_i^\top$ and $\bSigma_{\tilde \vg} = \E \tilde \vg_i \tilde \vg_i^\top$. 
By the Hanson-Wright concentration inequality, $\vg_i^\top \bS_i \vg_i = \tr(\bS_i \bSigma_\vg) + O_\sP(1 / \sqrt{n})$ and $\tilde \vg_i^\top \bS_i \tilde \vg_i = \tr(\bS_i \bSigma_{\tilde \vg}) + O_\sP(1 / \sqrt{n})$.
Hence,
\begin{align*}
    \E\frac{\vv^\top \bS_i  \vg_i \vg_i^\top \bS_i \vw}{1 + \vg_i^\top \bS_i \vg_i} &= \frac{\tr(\bS_i \vw\vv^\top \bS_i \bSigma_\vg)}{1 + \tr(\bS_i \bSigma_\vg)} + O_\sP( n^{-\frac{3}{2}}),\\
    \E\frac{\vv^\top \bS_i \tilde \vg_i \tilde \vg_i^\top \bS_i \vw}{1 + \tilde\vg_i^\top \bS_i \tilde\vg_i} &= \frac{\tr(\bS_i \vw\vv^\top \bS_i \bSigma_{\tilde \vg})}{1 + \tr(\bS_i \bSigma_{\tilde \vg})} + O_\sP( n^{-\frac{3}{2}}).
\end{align*}
Plugging into \eqref{eqn:lindeberg},
\begin{align*}
    &|\vv^\top \bar \E \bR_0 \vw - \vv^\top 
 \E \tilde \bR_0 \vw| \leq \sum_{i = 1}^N \left\vert \frac{\tr(\bS_i \vw\vv^\top \bS_i \bSigma_\vg)}{1 + \tr(\bS_i \bSigma_\vg)} - \frac{\tr(\bS_i \vw\vv^\top \bS_i \bSigma_{\tilde \vg})}{1 + \tr(\bS_i \bSigma_{\tilde \vg})} \right\vert + O_\sP(n^{-\frac{1}{2}})\\
 &\leq \sum_{i = 1}^N \left(\frac{|\tr(\bS_i \vw \vv^\top \bS_i (\bSigma_\vg - \bSigma_{\tilde \vg}))|}{1 + \tr(\bS_i \bSigma_\vg)} +  \frac{|\tr(\bS_i \vw \vv^\top \bS_i \bSigma_{\tilde \vg}) \tr(\bS_i(\bSigma_\vg - \bSigma_{\tilde \vg}))|}{(1 + \tr(\bS_i \bSigma_\vg))(1 + \tr(\bS_i \bSigma_{\tilde \vg}))} \right) + O_\sP(n^{-\frac{1}{2}}).
\end{align*}
Now, we have
\begin{align*}
    |\tr(\bS_i \vw \vv^\top \bS_i (\bSigma_\vg - \bSigma_{\tilde \vg}))| &\leq \Vert \bS_i \vv \Vert_2 \Vert \bS_i \vw \Vert_2 \Vert \bSigma_\vg - \bSigma_{\tilde \vg} \Vert_\text{op} = O_\sP(\Vert \bSigma_\vg - \bSigma_{\tilde \vg} \Vert_\text{op} / n),\\[0.2cm]
    |\tr(\bS_i \vw \vv^\top \bS_i \bSigma_{\tilde \vg})| &\leq \Vert \bS_i \vv \Vert_2 \Vert \bS_i \vw \Vert_2 \Vert \bSigma_{\tilde \vg} \Vert_\text{op}  = O_\sP(1/n),\\[0.2cm]
    |\tr(\bS_i(\bSigma_\vg - \bSigma_{\tilde \vg}))| &\leq \Vert \bS_i \Vert_F \Vert \bSigma_\vg - \bSigma_{\tilde \vg} \Vert_F \leq \sqrt{n} \Vert \bS_i \Vert_F \Vert \bSigma_\vg - \bSigma_{\tilde \vg} \Vert_\text{op} = O_\sP(\Vert \bSigma_\vg - \bSigma_{\tilde \vg} \Vert_\text{op}),
\end{align*}
where the first and second
inequalities
follow from the definition of the operator norm, and the last one follows from the Cauchy-Schwarz inequality.

By \cite[Theorem 2.1]{el2010spectrum}, $\Vert \bSigma_\vg - \bSigma_{\tilde \vg} \Vert_\text{op} \to_P 0$.
Therefore,
\begin{align*}
    |\vv^\top \bar \E \bR_0 \vw - \vv^\top 
 \E \tilde \bR_0 \vw| \leq O_\sP ( \left\Vert \bSigma_\vg - \bSigma_{\tilde \vg} \right\Vert_\text{op}) + O_\sP(n^{-\frac{1}{2}}) \to_P 0.
\end{align*}

\section{Proofs of Results from Section \ref{section:analysis_of_train_loss}}
Here, we will prove the results in Section \ref{section:analysis_of_train_loss}. First, we will provide several lemmas, which will be used in our proofs. The first lemma allows us to approximate linear and quadratic forms of $\vbeta$ in terms of $\vbeta_\star$; the quadratic form result is from \cite{ba2022high}. Its proof is in Section \ref{pflemma:turn_beta_to_beta_star}.

\begin{lemma}
\label{lemma:turn_beta_to_beta_star}
    For any $d\in \mathbb{N}$,
    let $\mathbf{v} \in \R^{d}$ and $\mathbf{D} \in \R^{d \times d}$
    be
    vectors and matrices, 
    fixed or independent of $\bX,\vbeta_\star,\ep_1,\ldots,\ep_n$, and
    satisfy $\Vert \mathbf{v}\Vert_2, \Vert \mathbf{D} \Vert_\textnormal{op} \leq C$ 
    almost surely,
    uniformly for some constant $C > 0$.
    Under Condition \ref{cond:limit}, we have
    \begin{align*}
        \left\vert \mathbf{v}^{\top} \vbeta - c_{\star,1} \mathbf{v}^{\top} \vbeta_\star \right\vert \to 0, \quad \left\vert \vbeta^\top \mathbf{D} \vbeta - \frac{1}{n} (c_\star^2 + \sigma_\ep^2) \operatorname{tr}\mathbf{D} - c_{\star,1}^2 \vbeta_\star^{\top} \mathbf{D} \vbeta_\star \right\vert \to 0
    \end{align*}
    in probability as $d \to \infty$.
\end{lemma}

We will use the expression derived for the training loss in the following lemma; see Section \ref{pflemma:train_loss} for the proof.

\begin{lemma}
\label{lemma:train_loss}
    The training loss $\gL_{\rm tr}(\bF)$ can be written as
$ \gL_{\rm tr}(\bF) = \lambda \tilde\vy^\top (\bF \bF^\top + \lambda n \bI_n)^{-1} \tilde\vy$.
\end{lemma}

The following lemma will be used in proving concentration of certain quadratic forms appearing in the proofs; see Section \ref{pflemma:poly-concentration}
 for the proof.

\begin{lemma}
    \label{lemma:poly-concentration}
    Let $g: \mathbb{R} \to \mathbb{R}$ be a polynomial,  $\bD\in \R^{n \times n}$ be a matrix 
    with $\Vert\bD\Vert_{\rm op} = O_\sP(1/n)$, and $\bZ\in \R^n$ be a vector of i.i.d. Gaussian random variables with bounded variance independent of $\bD$. We have 
    \begin{align*}
    \Big|g(\bZ)^\top \bD \;g(\bZ) - \mathbb{E}[g(\bZ)^\top \bD\; g(\bZ)]\Big| \to_P 0,    
    \end{align*}
    in which $g$ is applied elementwise.
\end{lemma}

The limiting values of two key quadratic forms appearing in the proof are derived in the following lemma, whose proof is deferred to Section \ref{pflemma:l1_limits}. 

\begin{lemma}
\label{lemma:l1_limits}
 Let $m_1$ and $m_2$ be the solutions to  the system of fixed point equations from \eqref{fpe}. 
 Then, the following holds:
    \begin{enumerate}
        \item[(a)] $\vbeta^\top \tilde \bX^\top \bar\bR_0 \tilde \bX\vbeta = \psi (c_\star^2 + \sigma_\ep^2) m_2 + \frac{\psi}{\phi} c_{\star,1}^2 m_2  + o_\sP(1) = \Theta_\sP(1)$.
        \item[(b)] $\va^\top \bF_0^\top \bar\bR_0 \bF_0 \va - \Vert \va\Vert_2^2 = - \lambda \frac{\psi^2}{\phi^2} m_1+ \frac{\psi}{\phi} - 1 + o_\sP(1) = \Theta_\sP(1)$.
    \end{enumerate}
 In particular, $\psi (c_\star^2 + \sigma_\ep^2) m_2 + \frac{\psi}{\phi} c_{\star,1}^2 m_2\neq 0$ and $- \lambda \frac{\psi^2}{\phi^2} m_1+ \frac{\psi}{\phi} - 1 \neq0$.
\end{lemma}

The following lemmas will be used in the computations.
We defer the proofs of these lemmas to Sections \ref{pflemma:general_orthogonality_21}, \ref{pflemma:delta_2_to_zero}, 
\ref{pflemma:s12},
and \ref{pft2rt2} respectively.

\begin{lemma}
\label{lemma:generalized_hermite_pq}
\label{lemma:general_orthogonality_21}
    For any $p, q \in \mathbb{N}_0$, $p \neq q$ and any vector $\vu \in \R^n$, with $\Vert \vu\Vert_2 = 1$ independent of $\bar\bR_0$, we have $H_q(\tilde\bX \vu)^\top \bar\bR_0 H_p(\tilde\bX \vu)= o_\sP(1)$.
\end{lemma}

\begin{lemma}
    \label{lemma:delta_2_to_zero}
    For any $p \in \mathbb{N}$, we have
    \begin{enumerate}
    \item[(a)] $\sqrt{N} H_p(\tilde\bX\vbeta_\star) \bar \bR_0 \bF_0 \va^{\circ 2}= o_\sP(1)$,
    \item[(b)] $\sqrt{N} H_p(\tilde\bX\vbeta) \bar \bR_0 \bF_0 \va^{\circ 2}= o_\sP(1)$.
    \end{enumerate}    
\end{lemma}
\vspace{10pt}

\begin{lemma}
\label{lemma:s12}
    For $s \in \{1,2\}$, $p \in \mathbb{N}$, 
    and $p \neq s$, we have
        $H_p(\tilde\bX\vbeta_\star)^\top \bar\bR_0 (\tilde\bX\vbeta)^{\circ s} = o_\sP(1)$.
     Further,
     \begin{align*}
    \lim_{n, N, d \to \infty} H_2(\tilde\bX\vbeta_\star)^\top \bar\bR_0(\tilde\bX\vbeta)^{\circ 2} = 2c_{\star,1}^2 \frac{\psi m_1}{\phi}.
\end{align*}
\end{lemma}

\begin{lemma}\label{t2rt2}
    We have
    \begin{align*}
    \lim_{n, N, d \to \infty} (\tilde \bX \vbeta)^{\circ 2 \top}\bar\bR_0(\tilde \bX \vbeta)^{\circ 2} 
    = \frac{3 \psi m_1}{\phi}[\phi(c_\star^2 + \sigma_\ep^2) + c_{\star,1}^2]^2.
\end{align*}
\end{lemma}

Now, we will first provide a proof of Theorem \ref{thm:general_ell_simplified} in the case of $\ell =1$ and $\ell = 2$ for a better insight into the proof techniques. We will then prove the general form in Section \ref{sec:general_ell}.

\subsection{Proof for $\ell = 1$}
\label{eq:proof_train_l1}
In the $\ell = 1$ regime, 
due to Theorem \ref{thm:training_equivalence}, we can replace $\bF$ by $\bF_1$ (defined in \eqref{eqn:Fell}) to compute the training loss. 
Hence, from now on we let $\bF = \bF_1$. We can write
$\bF\bF^\top = \bF_0\bF_0^\top + \bU \bK \bU^\top$
where
$\bU = [\,\bF_0\va\;|\;\tilde \bX\vbeta\;]$ and 
\begin{align*}
    \bK = \begin{bmatrix}
       0 & c_1^2\eta \\
       c_1^2\eta & c_1^4 \eta^2 \Vert \va \Vert_2^2
   \end{bmatrix}.
\end{align*}
Based on Lemma \ref{lemma:train_loss},
the training loss depends on 
$\bar\bR = (\bF\bF^\top + \lambda n \bI_n)^{-1}$.
Using the Woodbury formula, this matrix can be written
in terms of $ \bar\bR_0 = (\bF_0\bF_0^\top + \lambda n \bI_n)^{-1}$ 
as
\begin{align}\label{wb}
    \bar\bR = \bar\bR_0 - \bar\bR_0 \bU (\bK^{-1}+ \bU^\top \bar\bR_0 \bU)^{-1}\bU^\top \bar\bR_0.
\end{align}
Defining $\bT = (\bK^{-1}+ \bU^\top \bar\bR_0 \bU)^{-1} \in \R^{2\times 2}$
and substituting
$\bar\bR = \bar\bR_0 - \bar\bR_0 \bU \bT \bU^\top \bar\bR_0$
in the formula for training loss in Lemma \ref{lemma:train_loss}, we find
\begin{align}
\label{eq:diff_loss}
    \gL_{\rm tr}(\bF_0) - \gL_{\rm tr}(\bF) =   \lambda \tilde \vy^\top \bar\bR_0 \bU \bT\bU^\top \bar\bR_0 \tilde\vy.
\end{align}
Using \eqref{eq:diff_loss}
and
$\bU = [\,\bF_0\va\;|\;\tilde \bX\vbeta\;]$,
the loss difference can be written as
\begin{align}
    \label{eq:loss_difference}
    \gL_{\rm tr}(\bF_0) &- \gL_{\rm tr}(\bF)\nonumber\\ &= 
     \lambda \biggl[T_{11}(\tilde\vy^\top \bar\bR_0 \bF_0 \va)^2
    +   (T_{12}+T_{21}) \tilde\vy^\top\bar\bR_0 \tilde \bX \vbeta\cdot\va^\top\bF_0^\top\bar\bR_0\tilde\vy
    + T_{22} (\tilde\vy^\top\bar\bR_0 \tilde \bX \vbeta)^2\biggr],
\end{align}
in which $T_{ij}$ are the elements of the matrix $\bT$. 
Using
\begin{align}
\label{eq:T_ell1}
   \bT =\frac{\begin{bmatrix}
        \vbeta^\top \tilde \bX^\top \bar\bR_0 \tilde \bX \vbeta & -\frac{1}{c_1^2\eta}-\va^\top\bF_0^\top  \bar\bR_0 \tilde \bX \vbeta~\vspace{0.15cm}\\
        -\frac{1}{c_1^2\eta} - \vbeta^\top \tilde \bX^\top\bar\bR_0\bF_0 \va & 
        \va^\top \bF_0^\top \bar\bR_0 \bF_0 \va - \Vert \va\Vert_2^2
    \end{bmatrix}}{\left(\vbeta^\top \tilde \bX^\top \bar\bR_0 \tilde \bX \vbeta\right)\left(\va^\top \bF_0^\top \bar\bR_0 \bF_0 \va - \Vert \va\Vert_2^2\right) - \left(\frac{1}{c_1^2\eta} +\va^\top\bF_0^\top  \bar\bR_0 \tilde \bX \vbeta \right)^2},
\end{align}
we will compute the limit of each term appearing in \eqref{eq:loss_difference} separately:

\paragraph{Term 1.} The first term can be written as
\begin{align*}
     \delta_1 &= \lambda T_{11} (\tilde\vy^\top \bar\bR_0 \bF_0 \va)^2
     = 
     \frac{\lambda\, (\vbeta^\top \tilde \bX^\top \bar\bR_0 \tilde \bX \vbeta)\,(\tilde\vy^\top \bar\bR_0 \bF_0 \va)^2}{\left(\vbeta^\top \tilde \bX^\top \bar\bR_0 \tilde \bX \vbeta\right)\left(\va^\top \bF_0^\top \bar\bR_0 \bF_0 \va - \Vert \va\Vert_2^2\right) - \left(\frac{1}{c_1^2\eta} +\va^\top\bF_0^\top  \bar\bR_0 \tilde \bX \vbeta \right)^2}.
\end{align*}
Based on Lemma \ref{lemma:l1_limits}, we know that
$\vbeta^\top \tilde \bX^\top \bar\bR_0 \tilde \bX \vbeta$ and $\va^\top \bF_0^\top \bar\bR_0 \bF_0 \va - \Vert \va\Vert_2^2$ are $\Theta_\sP(1)$. Also, it can easily be seen that
\begin{align*}
    \Vert\bF_0^\top \bar\bR_0 \tilde \bX\vbeta\Vert_2 \leq \Vert \bF_0 \Vert_\textnormal{op} \Vert \bar \bR_0\Vert_\textnormal{op} \Vert \tilde \bX \vbeta\Vert_2 = O_\sP(1).  
\end{align*}
Hence, 
\begin{equation}
   \label{abs}
\left(\frac{1}{c_1^2\eta} +\va^\top\bF_0^\top  \bar\bR_0 \tilde \bX \vbeta \right)^2 = o_\sP(1) 
\end{equation}
because $\va \independent \bF_0^\top  \bar\bR_0 \tilde \bX \vbeta$. Also, 
using that 
 $ \bar\bR_0 \bF_0\bF_0^\top = (\bF_0\bF_0^\top + \lambda n \bI_n)^{-1} \bF_0\bF_0^\top = \bI - \lambda n \bar\bR_0$,  
\begin{align*}
    &(\tilde\vy^\top \bar\bR_0 \bF_0 \va)^2 = \tilde\vy^\top \bar\bR_0 \bF_0 \E_\va[\va \va^\top] \bF_0^\top \bar\bR_0\tilde\vy + o_\sP(1)\\
    &= \frac{1}{N}\tilde\vy^\top \bar\bR_0 \bF_0  \bF_0^\top \bar\bR_0\tilde\vy +o_\sP(1)
    = \frac{1}{N} \tilde\vy^\top\bar\bR_0\tilde \vy -  \frac{\lambda n}{N} \tilde \vy^\top \bar \bR_0^2 \tilde \vy  + o_\sP(1) = o_\sP(1),
\end{align*}
where in the last inequality, we used that $\tilde\vy^\top\bar\bR_0\tilde\vy \leq \frac{1}{\lambda n} \Vert\tilde \vy \Vert_2^2 = O_\sP(1)$ and $\tilde \vy^\top \bar\bR_0^2 \tilde\vy \leq \frac{1}{(\lambda n)^2} \Vert \tilde \vy \Vert_2^2 = o_\sP(1)$. Putting everything together, it follows that $\delta_1 = o_\sP(1)$ in probability.
\paragraph{Term 2 and Term 3.} The second and third terms can be written as
\begin{align*}
     \delta_2&= \delta_3 = \lambda T_{12}\tilde\vy^\top \bar\bR_0 \tilde \bX \vbeta \va^\top \bF_0^\top \bar\bR_0\tilde\vy\\ &= 
     \frac{\lambda\, \left(-\frac{1}{c_1^2\eta} - \va^\top \bF_0^\top \bar\bR_0 \tilde\bX \vbeta\right)\,(\tilde\vy^\top \bar\bR_0 \tilde\bX \vbeta \va^\top \bF_0^\top \bar\bR_0\tilde\vy)}{\left(\vbeta^\top \tilde\bX^\top \bar\bR_0 \tilde\bX \vbeta\right)\left(\va^\top \bF_0^\top \bar\bR_0 \bF_0 \va - \Vert \va\Vert_2^2\right) - \left(\frac{1}{c_1^2\eta} +\va^\top\bF_0^\top  \bar\bR_0 \tilde\bX \vbeta \right)^2}.
\end{align*}
Recall from the above argument that the denominator is $\Theta_\sP(1)$ and that $\frac{1}{c_1^2\eta} +\va^\top\bF_0^\top  \bar\bR_0 \tilde \bX \vbeta = o_\sP(1)$.
Also, $\tilde \vy^\top \bar \bR_0 \tilde \bX \vbeta \va^\top \bF_0^\top \bar \bR_0 \tilde \vy \leq \frac{1}{(\lambda n)^2} \Vert \tilde \vy\Vert_2^2 \Vert \tilde \bX \vbeta \Vert_2 \Vert \va \Vert_2 \Vert \bF_0 \Vert_\textnormal{op} = O_\sP(1).$
Therefore, we find $\delta_2 = \delta_3 = o_\sP(1)$.

\paragraph{Term 4.} This term can be written as
 \begin{align*}
      &\delta_4 = \lambda T_{22}(\tilde\vy^\top \bar\bR_0 \tilde\bX \vbeta)^2\\ 
      &= 
      \frac{\lambda\, \left(   \va^\top \bF_0^\top \bar\bR_0 \bF_0 \va - \Vert \va\Vert_2^2\right)\,(\tilde \vy^\top \bar \bR_0 \tilde \bX \vbeta)^2}
      {
      \vbeta^\top \tilde\bX^\top \bar\bR_0 \tilde\bX \vbeta\left(\va^\top \bF_0^\top \bar\bR_0 \bF_0 \va - \Vert \va\Vert_2^2\right) - \left(\frac{1}{c_1^2\eta} +\va^\top\bF_0^\top  \bar\bR_0 \tilde\bX \vbeta \right)^2} 
      =\lambda  \frac{(\tilde \vy^\top \bar \bR_0 \tilde \bX \vbeta)^2}{\vbeta^\top \tilde \bX^\top \bar \bR_0 \tilde \bX \vbeta} + o_\sP(1),
 \end{align*}
since $\frac{1}{c_1^2\eta} +\va^\top\bF_0^\top  \bar\bR_0 \tilde \bX \vbeta = o_\sP(1)$ and $\va^\top \bF_0^\top \bar\bR_0 \bF_0 \va - \Vert \va\Vert_2^2 = \Theta_\sP(1) \neq 0$ by Lemma \ref{lemma:l1_limits}. 
By \eqref{eqn:datagen} and Condition \ref{cond:tehe}, we can write $\tilde \vy = \sum_{p = 1}^\infty c_{\star, p} H_p(\tilde \bX \vbeta_\star) + \boldsymbol{\epsilon}$,
where $\boldsymbol{\epsilon} \in \R^n$ is additive Gaussian noise.
Note that
\begin{align*}
    \tilde \vy^\top \bar \bR_0 \tilde \bX \vbeta &=  c_{\star, 1} \vbeta_\star^\top \tilde \bX^\top \bar\bR_0 \tilde\bX \vbeta + o_\sP(1)
\end{align*}
by Lemma \ref{lemma:s12} and since $\Vert \bar\bR_0 \tilde \bX \vbeta \Vert_2 = O_\sP( 1/\sqrt{n})$ and $\boldsymbol{\varepsilon} \independent \bar\bR_0 \tilde \bX \vbeta$.

Further by Lemma \ref{lemma:turn_beta_to_beta_star},
\begin{align}
    \label{eq:tt1}
    c_{\star, 1} \vbeta_\star^\top \tilde \bX^\top \bar\bR_0 \tilde\bX \vbeta = c_{\star, 1}^2 \vbeta_\star^\top \tilde \bX^\top \bar\bR_0 \tilde\bX \vbeta_\star + o_\sP(1) \to_P c_{\star, 1}^2\frac{\psi}{\phi}m_2,
\end{align}
where the final limit follows from the proof of Lemma \ref{lemma:l1_limits}. By summing up the fours terms computed above and using Lemma \ref{lemma:l1_limits}, we find
\begin{align}
    \gL_{\rm tr}(\bF_0) -
         \gL_{\rm tr}(\bF)  \to_P \Delta_1 = \frac{\psi\lambda c_{\star,1}^4 m_2}{\phi[c_{\star,1}^2 + \phi(c_{\star}^2 + \sigma_\ep^2)]} > 0,
\end{align}
which concludes the proof for $\ell = 1$.

\subsection{Proof for $\ell = 2$}
\label{pl2}
In the $\ell = 2$ regime, based on Theorem \ref{thm:training_equivalence}, we can replace $\bF$ with $\bF_2$ (defined in \eqref{eqn:Fell}) to compute the training loss. Hence, from now on we let $\bF = \bF_2$. We can write
$\bF\bF^\top = \bF_0\bF_0^\top + \bU \bK \bU^\top$
where
$\bU = [\,\bF_0\va\;|\;\bF_0 \va^{\circ 2} \sqrt{N}\;|\;\;\tilde\bX\vbeta\;|\;(\tilde\bX\vbeta)^{\circ 2}\;]$ and 

\begin{align*}
    \bK = \begin{bmatrix}
       0 & 0&c_1^2\eta&0 \\[0.3cm]
       0&0&0&c_1^2c_2 \eta^2/\sqrt{N}\\[0.3cm]
       c_1^2 \eta & 0 & c_1^4 \eta^2 \Vert\va\Vert_2^2& c_1^4c_2 \eta^3 \langle \va, \va^{\circ 2}\rangle\\[0.3cm]
       0 & c_1^2c_2\eta^2/\sqrt{N} & c_1^4 c_2 \eta^3 \langle \va^{\circ 2}, \va\rangle & c_1^4c_2^2 \eta^4 \langle \va^{\circ 2}, \va^{\circ 2}\rangle
   \end{bmatrix}.
\end{align*}

Recalling 
$\bar\bR = (\bF\bF^\top + \lambda n \bI_n)^{-1}$
and
$ \bar\bR_0 = (\bF_0\bF_0^\top + \lambda n \bI_n)^{-1}$,
we still have \eqref{wb}.
Defining
$\bT = (\bK^{-1}+ \bU^\top \bar\bR_0 \bU)^{-1} \in \R^{4\times 4}$, we have the following analogue to \eqref{eq:diff_loss}:
\begin{align}
\label{eq:train-l2-first}
    \gL_{\rm tr}(\bF_0) - \gL_{\rm tr}(\bF) =   \lambda \tilde\vy^\top \bar\bR_0 \bU \bT\bU^\top \bar\bR_0 \tilde\vy.
\end{align}

Denoting in what follows
$\bQ = \bF_0^\top \bar\bR_0 \bF_0$,
the inverse $\bT^{-1}$ can be written as follows:

\begin{align}
\label{eq:inv_T_l2}
\scalemath{0.9}{
    \hspace{-0.5cm}\begin{bmatrix}
         \va^\top \bQ \va -\Vert \va\Vert_2^2 & N^{\frac12}\va^\top (\bQ-\bI) \va^{\circ 2} & \va^\top \bF_0^\top \bar\bR_0 \tilde\vtheta + \frac{1}{c_1^2 \eta} & \va^\top \bF_0^\top \bar\bR_0\tilde\vtheta^{\circ 2}\\[0.5cm]
       N^{\frac12}\va^\top (\bQ-\bI) \va^{\circ 2} & {N}\va^{\circ 2\top}\bQ\va^{\circ 2}-N\Vert \va^{\circ 2}\Vert_2^2&N^{\frac12}\va^{\circ 2\top}\bF_0^\top \bar\bR_0\tilde\vtheta&N^{\frac12}\va^{\circ 2\top}\bF_0^\top \bar\bR_0\tilde\vtheta^{\circ 2} + \frac{N^{\frac12}}{c_1^2c_2\eta^2}\\[0.5cm]
         \tilde\vtheta^\top\bar\bR_0 \bF_0 \va + \frac{1}{c_1^2 \eta}&N^{\frac12} \tilde\vtheta^\top\bar\bR_0\bF_0\va^{\circ 2}& \tilde\vtheta^\top\bar\bR_0\tilde\vtheta& \tilde\vtheta^\top\bar\bR_0\tilde\vtheta^{\circ 2}\\[0.5cm]
        \tilde\vtheta^{\circ 2 \top}\bar\bR_0\bF_0 \va&N^{\frac12}\tilde\vtheta^{\circ 2 \top}\bar\bR_0\bF_0\va^{\circ 2} + \frac{N^{\frac12}}{c_1^2c_2\eta^2}&\tilde\vtheta^{\circ 2 \top}\bar\bR_0\tilde\vtheta&\tilde\vtheta^{\circ 2 \top}\bar\bR_0\tilde\vtheta^{\circ 2}
    \end{bmatrix}}.
\end{align}

\subsubsection{Analysis of Terms in \texorpdfstring{$\bT^{-1}$}{T-1} and \texorpdfstring{$\bT$}{T}}
\label{section:l2-Hinverse}
In the following section, we will first analyze the elements of $\bT^{-1}$:

\paragraph{(1,1):} The term $\va^\top \bQ \va -\Vert \va\Vert_2^2$ has already been analyzed in Lemma \ref{lemma:l1_limits} and is $\Theta_\sP(1)$.

\paragraph{(1,2) and (2,1):} 
Recalling $\bQ = \bF_0^\top \bar\bR_0 \bF_0$ and 
 $\bR_0 = (\bF_0^\top\bF_0 + \lambda n \bI_N)^{-1}$,
we can write
    \begin{align*}
     [\bT^{-1}]_{1,2} = [\bT^{-1}]_{2,1} &= \sqrt{N}\va^\top \bQ \va^{\circ 2}- \sqrt{N}\langle\va, \va^{\circ 2} \rangle\\ 
     &=  - \lambda n\sqrt{N}\va^\top  \bR_0 \va^{\circ 2}
     =  -\lambda n \sqrt{N} \va^\top\bR_0\left( \va^{\circ 2}- 1/N \mathbf{1}_N + 1/N \mathbf{1}_N\right).
    \end{align*}
    Introducing $\tilde\va = \sqrt{N} \va \sim \normal(0, \bI_N)$, 
    and as $H_2(x) = x^2-1$ for all $x$,
    we find
    \begin{align*}
         [\bT^{-1}]_{1,2} = [\bT^{-1}]_{2,1} = - \frac{\lambda n}{N} \tilde\va^\top\bR_0 H_2(\tilde\va) - \frac{\lambda n}{\sqrt{N}} \va^\top \bR_0 \mathbf{1}_N.
    \end{align*}
    The second term converges to zero as $n\to\infty$ because $\va \sim \normal(0, \frac{1}{N}\bI_N)$ is independent of $\bR_0$, and
    $\bigl\Vert\frac{n}{\sqrt{N}}\bR_0 \mathbf{1}_N\bigr\Vert_2 = O_\sP(1)$. Moreover, the first term also converges to zero; indeed,
    \begin{align*}
        \tilde\va^\top  \bR_0 H_2(\tilde\va) = \left(\frac{\tilde\va^\top + H_2(\tilde\va)}{2}\right)^\top\bR_0\left(\frac{\tilde\va^\top + H_2(\tilde\va)}{2}\right) - \left(\frac{\tilde\va^\top - H_2(\tilde\va)}{2}\right)^\top\bR_0\left(\frac{\tilde\va^\top - H_2(\tilde\va)}{2}\right).
    \end{align*}
    
    Lemma \ref{lemma:poly-concentration} can be used with $\bD = \bR$ to prove the concentration of both term around their expectation. Note that the expectation of $\tilde\va^\top\bR_0 H_2(\tilde\va)$ is zero because of the orthogonality property of Hermite polynomials and the independence of $\tilde\va$ and $\bR_0$.
    Putting everything together, we conclude that $[\bT^{-1}]_{1,2} = [\bT^{-1}]_{2,1} = o_\sP(1)$.

\paragraph{(1,3) and (3,1):} Recalling that $\tilde\vtheta=\tilde \bX\vbeta$, 
it follows from \eqref{abs} that this term is $o_{\sP}(1)$.  

\paragraph{(1,4) and (4,1):} To bound $\va^\top \bF_0^\top \bar\bR_0\tilde\vtheta^{\circ 2}$, note that 
\begin{align*}
\Vert\bF_0^\top \bar\bR_0\tilde\vtheta^{\circ 2}\Vert_{\rm op} \leq \Vert\bF_0\Vert_{\rm op}\Vert\bar\bR_0\Vert_{\rm op}\Vert\tilde\vtheta^{\circ 2}\Vert_{2}  = O_\sP(1).
\end{align*}

Hence, because $\va \sim \normal(0, \frac{1}{N}\bI_N)$ is independent of $\bF_0^\top \bar\bR_0\tilde\vtheta^{\circ 2}$, we have
    \begin{align*}
    [\bT^{-1}]_{1,4} = [\bT^{-1}]_{4,1}=\va^\top \bF_0^\top \bar\bR_0\tilde\vtheta^{\circ 2} = o_\sP(1).
    \end{align*}

    \paragraph{(2,2):} This term is $O_\sP(1)$, because $\va \sim \normal(0, \frac{1}{N}\bI_N)$, so 
    \begin{align*}
    [\bT^{-1}]_{2,2} &= {N}\va^{\circ 2\top}\bQ\va^{\circ 2}-N\Vert \va^{\circ 2}\Vert_2^2   
    = - \lambda N n\, \va^{\circ 2 \top}(\bF_0^\top \bF_0 + \lambda n \bI_N)^{-1}\va^{\circ 2}\\
    &\leq \lambda N n \,\Vert\va^{\circ 2}\Vert_2^2\cdot \Vert(\bF_0^\top \bF_0 + \lambda n \bI_N)^{-1}\Vert_{\rm op} = O_\sP(1).
    \end{align*}
    
    \paragraph{(2,3) and (3,2):} To bound $\sqrt{N}\va^{\circ 2\top}\bF_0^\top \bar\bR_0\tilde\vtheta$, note that
    \begin{align*}
    \Vert\sqrt{N}\va^{\circ 2\top}\bF_0^\top \bar\bR_0\tilde\bX\Vert_{2} &\leq \Vert\sqrt{N}\va^{\circ 2}\Vert_2\Vert\bF_0\Vert_{\rm op}\Vert\bR_0\Vert_{\rm op}\Vert\tilde\bX\Vert_{\rm op}
    \leq C \cdot \sqrt{N} \cdot \frac{1}{n} \cdot \sqrt{N} = O_\sP(1).
    \end{align*}
    Also, by Lemma \ref{lemma:turn_beta_to_beta_star}, we have
    \begin{align*}
        [\bT^{-1}]_{2,3} = [\bT^{-1}]_{3,2} = \sqrt{N}\va^{\circ 2\top}\bF_0^\top \bar\bR_0\tilde\bX\vbeta = c_{\star,1}\sqrt{N}\va^{\circ 2\top}\bF_0^\top \bar\bR_0\tilde\bX\vbeta_\star + o_\sP(1),
    \end{align*}
    which converges to zero, because $\vbeta_\star \sim \normal(0, \frac{1}{d}\bI_d)$ and is independent of $\sqrt{N}\va^{\circ 2\top}\bF_0^\top \bar\bR_0\tilde\bX$, which has bounded norm in probability.
    
\paragraph{(2,4) and (4,2):} First note that in the regime where $\ell = 2$, we have $ \frac{\sqrt{N}}{\eta^2} \to 0 $. 
Hence, we can write
    \begin{align}
    \label{eq:left-ortho-argument}
    [\bT^{-1}]_{2,4}
    & = \sqrt{N}(\tilde\bX\vbeta)^{\circ 2 \top}\bar\bR_0\bF_0\va^{\circ 2} +o_\sP(1)
    = \sqrt{N} H_2(\tilde\bX\vbeta)^{ \top}\bar\bR_0\bF_0\va^{\circ 2} +
    \sqrt{N} \mathbf{1}_n^\top\bar\bR_0\bF_0\va^{\circ 2} +
    o_\sP(1).
    \end{align}
By Lemma \ref{lemma:delta_2_to_zero}, the first term converges in probability to zero.
Moreover, $\va \sim \normal(0, \frac{1}{N}\bI_N)$ is independent of $\bar\bR_0\bF_0$,
and $\|\mathbf{1}_n^\top\bar\bR_0\bF_0\|_2 = O_\sP(1)$.
Thus, 
we have that $\sqrt{N}\mathbf{1}_n^\top\bar\bR_0\bF_0\left( \va^{\circ 2}- 1/N \mathbf{1}_N\right) \to_P 0$.
Hence, we find
    \begin{align*}
    [\bT^{-1}]_{2,4} =
    \sqrt{N} \mathbf{1}_n^\top\bar\bR_0\bF_0\mathbf{1}_N/N +
    o_\sP(1).
    \end{align*}
Based on the Gaussian equivalence from Appendix \ref{sec:gec}, we
can replace $\bF_0$ with
$\bF_0 = c_1 \tilde\bX \bW_0^\top + c_{>1} \bZ$,
    where $\bZ \in \R^{n \times d}$ is an independent random matrix with $\normal(0, 1)$ entries, without changing the limit. 
Now,  the linearized $\bF_0$ is left-orthogonally invariant, hence $\bF_0$ has the same distribution as $\bO\bF_0$, where $\bO$ is uniformly distributed over the Haar measure of $d$-dimensional orthogonal matrices, independently of all other randomness. 
Hence,
$
    N^{-1/2} \mathbf{1}_n^\top\bar\bR_0\bF_0\mathbf{1}_N 
=_d
N^{-1/2} \mathbf{1}_n^\top \bO \bar\bR_0\bF_0\mathbf{1}_N. 
$
Now,
$ \bO ^\top \mathbf{1}_n =_d \sqrt{n}\vz/\|\vz\|_2$, where $\vz\sim\normal(0,\bI_n)$. Moreover $\|\vz\|_2 = \sqrt{n}(1+o_\sP(1))$, hence 
replacing $ \bO ^\top \mathbf{1}_n$ with $\vz^\top$ introduces negligible error. 
Hence,
$   [\bT^{-1}]_{2,4} =_d
    N^{-1/2} \vz^\top\bar\bR_0\bF_0\mathbf{1}_N +
    o_\sP(1).
$
Now, $\vz^\top\bar\bR_0\bF_0\mathbf{1}_N\sim \normal(0,\|\bar\bR_0\bF_0\mathbf{1}_N\|_2^2)$, and
$\|\bar\bR_0\bF_0\mathbf{1}_N\|_2  = O_{\sP}(1)$,
thus
$[\bT^{-1}]_{2,4} \to_P0$.

\paragraph{(3,3):} We have $\Vert\tilde\vtheta\Vert_2 = O_\sP(\sqrt{N})$ and $\Vert\bar\bR_0\Vert_{\rm op} = O_\sP(1/n)$. Thus, $[\bT^{-1}]_{3,3} = O_\sP(1)$.

\paragraph{(3,4) and (4,3):} First, note that defining $\tilde\vbeta = \frac{\vbeta}{\Vert\vbeta\Vert_2}$,
    and as $H_2(x) = x^2-1$ for all $x$,
    we can write
\begin{align*}
    [\bT^{-1}]_{3,4} = [\bT^{-1}]_{4,3} &=  \tilde\vtheta^\top\bar\bR_0 \tilde\vtheta^{\circ 2}
    = \Vert \vbeta\Vert_2^3 \left((\tilde \bX \tilde\vbeta)^\top\bar\bR_0 (\tilde \bX \tilde\vbeta)^{\circ 2}\right)\\
    &= \Vert \vbeta\Vert_2^3 \left((\tilde \bX \tilde\vbeta)^\top\bar\bR_0 H_2(\tilde \bX \tilde\vbeta)\right) + \Vert \vbeta\Vert_2^2 \left( \tilde\vtheta^\top\bar\bR_0 \mathbf{1}_N \right).
\end{align*}
Now, by Lemma \ref{lemma:turn_beta_to_beta_star}, we have
$ \tilde\vtheta^\top\bar\bR_0 \mathbf{1}_N  = c_{\star,1}\tilde\vtheta_\star^\top \bar\bR_0 \mathbf{1}_N + o_\sP(1)$.
Now, note that $\Vert\tilde \bX \bar\bR_0 \mathbf{1}_N\Vert_2 = O_\sP(1)$ and $\vbeta_{\star} \sim \normal(0, \frac{1}{d}\bI_d)$ is independent of $\tilde \bX \bar\bR_0 \mathbf{1}_N$, which implies that the second term converges to zero. By using Lemma \ref{lemma:general_orthogonality_21} for $\vu = \tilde\vbeta$,
the first term also converges to zero. Putting these together, we have $[\bT^{-1}]_{3,4} = [\bT^{-1}]_{4,3} = o_{\sP}(1)$.
\paragraph{(4,4):} We have $\Vert\tilde\vtheta^{\circ 2}\Vert_2 = O_\sP(\sqrt{N})$ and $\Vert\bar\bR_0\Vert_{\rm op} = O_\sP(1/n)$. Thus, $[\bT^{-1}]_{4,4} = O_\sP(1)$.

Now, putting everything together, the matrix $\bT^{-1}$ can be written as

\begin{align*}
    \bT^{-1} =\begin{bmatrix}
         [\bT^{-1}]_{1,1} & 0& 0& 0\\[0.5cm]
        0& [\bT^{-1}]_{2,2}&0&0\\[0.5cm]
        0&0&[\bT^{-1}]_{3,3}&0\\[0.5cm]
        0&0&0&[\bT^{-1}]_{4,4}
    \end{bmatrix}
    + \mathbf{\Delta}_1,\\
\end{align*}
where the all elements of $\mathbf{\Delta}_1$ are $o_\sP(1)$. Thus the matrix $\bT$ is equal to

\begin{align}\label{t}
    \bT =\begin{bmatrix}
         \frac{1}{[\bT^{-1}]_{1,1}} & 0& 0& 0\\[0.5cm]
        0& \frac{1}{[\bT^{-1}]_{2,2}}&0&0\\[0.5cm]
        0&0&\frac{1}{[\bT^{-1}]_{3,3}}&0\\[0.5cm]
        0&0&0&\frac{1}{[\bT^{-1}]_{4,4}}
    \end{bmatrix}
    + \mathbf{\Delta}_2,\\\nonumber
\end{align}
where the all elements of $\mathbf{\Delta}_2$ are $o_\sP(1)$. 

\subsubsection{Computing the training loss}

Having computed the limit of the matrix $\bT^{-1}$ and $\bT$, we are now ready to put everything together and compute the limiting train loss. One can write the outcome vector $\tilde\vy$ as $\tilde \vy = \sigma_\star (\tilde \bX \vbeta_\star) + \boldsymbol{\epsilon}$, where $\boldsymbol{\epsilon} \in \R^n$ is the noise term. Thus, using \eqref{eq:train-l2-first}, we find

\begin{align}
\label{eq:noise-expansion-l2}
\gL_{\rm tr}(\bF_0) - \gL_{\rm tr}(\bF) &=   \lambda \sigma_\star (\tilde \bX \vbeta_\star)^\top \bar\bR_0 \bU \bT \bU^\top \bar\bR_0\sigma_\star (\tilde \bX \vbeta_\star)\nonumber\\&\hspace{1cm} + 2 \lambda \sigma_\star (\tilde \bX \vbeta_\star)^\top \bar\bR_0 \bU \bT \bU^\top \bar\bR_0\boldsymbol{\epsilon}
+ \lambda \boldsymbol{\epsilon}^\top \bar\bR_0 \bU \bT \bU^\top \bar\bR_0\boldsymbol{\epsilon}.
\end{align}

We will first argue the second and third term will go to zero in probability. To do this, we note that $\Vert\bT\Vert_{\rm op} = O_\sP(1)$ and also $\Vert \bU^\top \bar\bR_0\Vert_2 \leq \Vert \bU\Vert_{\rm op} \Vert \bar \bR_0 \Vert_{\rm op} = O_{\sP}(1/\sqrt{n})$. We have $\boldsymbol{\epsilon} \sim \normal(0, \sigma_\epsilon^2 \bI_n)$ and it is independent of $\bar\bR_0, \bU, \bT, \tilde \bX$, and $\vbeta_\star$. Also note that $\Vert\sigma_\star(\tilde \bX \vbeta_\star)^\top \bar\bR_0 \bU\Vert_2 = O_{\sP}(1).$ Thus, the second and third term in \eqref{eq:noise-expansion-l2} go to zero and we have

\begin{align}
\gL_{\rm tr}(\bF_0) - \gL_{\rm tr}(\bF) &=   \lambda \sigma_\star (\tilde \bX \vbeta_\star)^\top \bar\bR_0 \bU \bT \bU^\top \bar\bR_0\sigma_\star (\tilde \bX \vbeta_\star)\nonumber + o_\sP(1).
\end{align}

If we expand $\sigma_\star(\tilde \bX\vbeta_\star)$ in the Hermite basis as $\sigma_\star(\tilde \bX\vbeta_\star) = \sum_{p = 1}^{\infty}c_{\star,p} H_p(\tilde\vtheta_\star)$, we can write
\begin{align*}
\gL_{\rm tr}(\bF_0) - \gL_{\rm tr}(\bF) =   \lambda \sum_{p, q = 1}^{\infty}c_{\star,p} c_{\star,q} H_p(\tilde\vtheta_\star)^\top \bar\bR_0\bU \bT\bU^\top \bar\bR_0 H_q(\tilde\vtheta_\star) + o_\sP(1).
\end{align*}
We define $\Delta_{p, q} =  H_p(\tilde\vtheta_\star)^\top \bar\bR_0 \bU \bT\bU^\top \bar\bR_0 H_q(\tilde\vtheta_\star)  = \delta_1^{p,q} + \delta_2^{p,q} + \delta_3^{p,q} + \delta_4^{p,q}$  
in which, 
with $T_{i,j}$ being the $(i,j)$-th elements of the matrix $\bT$,

\begin{align}
\label{eq:delta1}
    \delta_1^{p,q} &= T_{1,1}H_p(\tilde\vtheta_\star)^\top \bar\bR_0(\bF_0\va)(\bF_0\va)^\top\bar\bR_0 H_q(\tilde\vtheta_\star)\vspace{0.2cm}\nonumber\\
        &\hspace{1cm}+ T_{1,2} H_p(\tilde\vtheta_\star)^\top \bar\bR_0(\bF_0\va) (\sqrt{N}\bF_0\va^{\circ 2})^{\top}\bar\bR_0 H_q(\tilde\vtheta_\star)\vspace{0.2cm}\nonumber\\
        &\hspace{1cm}+ T_{1,3}H_p(\tilde\vtheta_\star)^\top \bar\bR_0(\bF_0\va) \tilde\vtheta^\top\bar\bR_0 H_q(\tilde\vtheta_\star)\vspace{0.2cm}\nonumber\\
        &\hspace{1cm}+ T_{1,4}H_p(\tilde\vtheta_\star)^\top \bar\bR_0(\bF_0\va)\tilde\vtheta^{\circ 2 \top}\bar\bR_0 H_q(\tilde\vtheta_\star), 
\end{align}
\begin{align}
\label{eq:delta2}
    \delta_2^{p,q}&= T_{2,1}H_p(\tilde\vtheta_\star)^\top \bar\bR_0(\sqrt{N}\bF_0\va^{\circ 2})(\bF_0\va)^\top\bar\bR_0 H_q(\tilde\vtheta_\star)\vspace{0.2cm}\nonumber\\
    &\hspace{1cm}+  T_{2,2}H_p(\tilde\vtheta_\star)^\top \bar\bR_0(\sqrt{N}\bF_0\va^{\circ 2})(\sqrt{N}\bF_0\va^{\circ 2})^\top\bar\bR_0 H_q(\tilde\vtheta_\star)\vspace{0.2cm}\nonumber\\
        &\hspace{1cm}+T_{2,3}H_p(\tilde\vtheta_\star)^\top \bar\bR_0(\sqrt{N}\bF_0\va^{\circ 2})\tilde\vtheta^\top\bar\bR_0 H_q(\tilde\vtheta_\star)\vspace{0.2cm}\nonumber\\
        &\hspace{1cm}+ T_{2,4}H_p(\tilde\vtheta_\star)^\top \bar\bR_0(\sqrt{N}\bF_0\va^{\circ 2})\tilde\vtheta^{\circ 2 \top}\bar\bR_0 H_q(\tilde\vtheta_\star),
\end{align}
\begin{align}
\label{eq:delta3}
    \delta_3^{p,q} &= T_{3,1}H_p(\tilde\vtheta_\star)^\top \bar\bR_0\tilde\vtheta(\bF_0\va)^\top\bar\bR_0 H_q(\tilde\vtheta_\star)\vspace{0.2cm}\nonumber\\
    &\hspace{1cm}+ T_{3,2}H_p(\tilde\vtheta_\star)^\top \bar\bR_0\tilde\vtheta(\sqrt{N}\bF_0\va^{\circ 2})^\top\bar\bR_0 H_q(\tilde\vtheta_\star)\vspace{0.2cm}\vspace{0.2cm}\nonumber\\
    &\hspace{1cm}+ T_{3,3}H_p(\tilde\vtheta_\star)^\top \bar\bR_0\tilde\vtheta\tilde\vtheta^{\top}\bar\bR_0 H_q(\tilde\vtheta_\star)\vspace{0.2cm}\nonumber\\
    &\hspace{1cm}+ T_{3,4}H_p(\tilde\vtheta_\star)^\top \bar\bR_0\tilde\vtheta\tilde\vtheta^{\circ 2 \top}\bar\bR_0 H_q(\tilde\vtheta_\star),
\end{align}
and
\begin{align}
\label{eq:delta4}
    \delta_4^{p,q} &=  T_{4,1}H_p(\tilde\vtheta_\star)^\top \bar\bR_0\tilde\vtheta^{\circ 2}(\bF_0\va)^\top\bar\bR_0 H_q(\tilde\vtheta_\star)\vspace{0.2cm}\nonumber\\
    &\hspace{1cm}+ T_{4,2}H_p(\tilde\vtheta_\star)^\top \bar\bR_0\tilde\vtheta^{\circ 2}(\sqrt{N}\bF_0\va^{\circ 2})^\top\bar\bR_0 H_q(\tilde\vtheta_\star)\vspace{0.2cm}\nonumber\\
        &\hspace{1cm}+ T_{4,3}H_p(\tilde\vtheta_\star)^\top \bar\bR_0\tilde\vtheta^{\circ 2}\tilde\vtheta^{\top}\bar\bR_0 H_q(\tilde\vtheta_\star)\vspace{0.2cm}\nonumber\\
    &\hspace{1cm}+ T_{4,4}H_p(\tilde\vtheta_\star)^\top \bar\bR_0\tilde\vtheta^{\circ 2}\tilde\vtheta^{\circ 2 \top}\bar\bR_0 H_q(\tilde\vtheta_\star).
\end{align}

We will now look at each $\delta_i^{p,q}$ for $i \in \{1,2,3,4\}$.

\paragraph{Term $\delta_1^{p,q}$:} To prove that the term in \eqref{eq:delta1} are asymptotically negligible, note that $\va \sim \normal(0, \frac{1}{N} \bI_N)$ is independent of $H_p(\tilde\vtheta_\star)\bar\bR_0\bF_0 $ and we have $\Vert H_p(\tilde\vtheta_\star)\bar\bR_0\bF_0  \Vert_{\rm 2} = O_\sP(1)$. Thus, $H_p(\tilde\vtheta_\star)\bar\bR_0\bF_0 \va = o_\sP(1)$ and all other terms multiplying this are $O_\sP(1)$. This implies that for any $p, q \in \mathbb{N}$, we have $\delta_1^{p, q} = o_\sP(1)$.

\paragraph{Term $\delta_2^{p,q}$:}  All four terms in \eqref{eq:delta2} converge to zero. To prove this, we will use the Lemma \ref{lemma:delta_2_to_zero}. In \eqref{eq:delta2}, all terms multiplied by $\sqrt{N} H_p(\tilde\vtheta_\star) \bar \bR_0 \bF_0 \va^{\circ 2}$ are $O_\sP(1)$. Thus, $\delta_2^{p,q} = o_\sP(1)$ for any $p, q \in \mathbb{N}$.

\paragraph{Term $\delta_3^{p,q}$:} The first term in \eqref{eq:delta3} converges to zero in probability due to an argument similar to the arguments used for $\delta_1^{p,q}$;
and the same holds for the second term in \eqref{eq:delta3},
by arguing similarly as for $\delta_2^{p,q}$.  
We have shown that $T_{3,4} = o_\sP(1)$, and by a norm argument, we can see that $H_p(\tilde\vtheta_\star)^\top \bar\bR_0\tilde\vtheta$ and $\tilde\vtheta^{\circ 2 \top}\bar\bR_0 H_q(\tilde\vtheta_\star)$ are $O_\sP(1)$. Hence, 
\begin{align*}
    \delta_3^{p,q} &=T_{3,3}\bigl(H_p(\tilde\vtheta_\star)^\top\bar\bR_0 \tilde\vtheta\bigr)\bigl(\tilde\vtheta^{\top}\bar\bR_0 H_q(\tilde\vtheta_\star)\bigr)+ o_\sP(1).
\end{align*}

\paragraph{Term $\delta_4^{p,q}$:} 
The first two terms in \eqref{eq:delta4} converge to zero by the same reasoning used for $\delta_1^{p,q}$ and $\delta_2^{p,q}$, respectively. 
The third term can also be shown to converge to zero by recalling that $T_{4,3} = o_\sP(1)$. Hence, we can write
\begin{align*}
        \delta_4^{p,q} &= T_{4,4}\bigl(H_p(\tilde\vtheta_\star)^\top\bar\bR_0 \tilde\vtheta^{\circ 2}\bigr)\bigl(\tilde\vtheta^{\circ 2\top}\bar\bR_0 H_q(\tilde\vtheta_\star)\bigr)+ o_\sP(1).
\end{align*}
Putting everything together, we find
\begin{align*}
    L_{\rm tr}(\bF_0) - L_{\rm tr}(\bF) 
    &= \lambda T_{3,3}  \sum_{p, q = 1}^{\infty}c_{\star,p} c_{\star,q} (H_p(\tilde\vtheta_\star)^\top\bar\bR_0 \tilde\vtheta\bigr)\bigl(\tilde\vtheta^{\top}\bar\bR_0 H_q(\tilde\vtheta_\star)\bigr)\\
    &+\lambda T_{4,4}\sum_{p, q = 1}^{\infty}c_{\star,p} c_{\star,q} \bigl(H_p(\tilde\vtheta_\star)^\top\bar\bR_0 \tilde\vtheta^{\circ 2}\bigr)\bigl(\tilde\vtheta^{\circ 2\top}\bar\bR_0 H_q(\tilde\vtheta_\star)\bigr) + o_\sP(1).
\end{align*}
Using Lemma \ref{lemma:s12}, we know that in the sums above, the terms corresponding to $(p,q) = (1,1)$ and $(p,q) = (2,2)$ are the only non-negligible terms in the first and second sum respectively.

Hence,
as $T_{3,3} = 1/{(\tilde\vtheta^\top\bar\bR_0\tilde\vtheta)} + o_\sP(1)$
and 
$T_{4,4} =1/(\tilde\vtheta^{\circ 2 \top}\bar\bR_0\tilde\vtheta^{\circ 2}) + o_{\sP}(1)$,
from Lemmas 
\ref{lemma:turn_beta_to_beta_star}, \ref{lemma:l1_limits},
\ref{lemma:s12} and \ref{t2rt2},
we can write,
\begin{align*}
\gL_{\rm tr}(\bF) - \gL_{\rm tr}(\bF_0) 
&= \lambda T_{3,3} c_{\star,1}^2 \bigl(\tilde\vtheta_\star^\top \bar\bR_0\tilde\vtheta\bigr)^2
    + \lambda T_{4,4}c_{\star,2}^2 \bigl(H_2(\tilde\vtheta_\star)^\top \bar\bR_0\tilde\vtheta^{\circ 2}\bigr)^2 + o_\sP(1)
    \\&=
    \lambda  \frac{c_{\star,1}^2 \bigl(\tilde\vtheta_\star^\top \bar\bR_0\tilde\vtheta\bigr)^2}{\tilde \vtheta^\top \bar\bR_0 \tilde\vtheta}
    + \lambda c_{\star,2}^2 \frac{\bigl(H_2(\tilde\vtheta_\star)^\top \bar\bR_0\tilde\vtheta^{\circ 2}\bigr)^2}
    {\tilde\vtheta^{\circ 2 \top}\bar\bR_0\tilde\vtheta^{\circ 2}}+ o_\sP(1)\\
    &\to_P \Delta_2 =  \frac{\psi\lambda c_{\star,1}^4 m_2}{\phi[c_{\star,1}^2 + \phi(c_{\star}^2 + \sigma_\ep^2)]} + \frac{4\psi\lambda c_{\star,1}^4c_{\star,2}^2m_1}
    {3\phi[\phi(c_\star^2 + \sigma_\ep^2) + c_{\star,1}^2]^2},
\end{align*}
proving the theorem for $\ell =2$.

\section{Asymptotics of the Training Loss for General \texorpdfstring{$\ell$}{ell}}
\label{sec:general_ell}
We define the values $\xi_{i,j}$ 
for all $i,j \in \{0,1,\ldots\}$
such that for any $p \in \mathbb{N}$ and $x\in \R$, we have $x^p = \sum_{i = 0}^{p} \xi_{p,i} H_i(x)$. 
\begin{theorem}
\label{thm:general_ell}
    Let $\ell \in \mathbb{N}$. If Conditions \ref{cond:limit}-\ref{cond:tehe} hold, 
    while we also have $c_1,\cdots,c_\ell \neq 0$, and
    $\eta \asymp n^\alpha$ with $\frac{\ell - 1}{2\ell} < \alpha < \frac{\ell}{2\ell + 2}$,
     then
     for the learned feature map $\bF$ and the untrained feature map $\bF_0$, we have $\gL_{\rm tr}(\bF_0) - 
         \gL_{\rm tr}(\bF) \to_P \Delta_\ell>0$, where 
     \begin{align*}
         \Delta_\ell
         =\lambda \sum_{p = 1}^{\ell}\sum_{q = 1}^{\ell} c_{\star, p} c_{\star, q}r_pr_q\sum_{i = 1}^{\ell}\sum_{j = 1}^{\ell}\; \mathbf{\Omega}_{i, j}\left({\phi(c_\star^2 + \sigma_\ep^2) + c_{\star,1}^2}\right)^{(i+j)/2}
    \xi_{i,p}\xi_{j,q} +o_\sP(1),
    \end{align*}
    in which $\mathbf{\Omega}$ is an invertible matrix with
    \begin{align*}
        [\mathbf{\Omega}^{-1}]_{i,j} = \left(c_{\star,1}^2 + \phi(c_\star^2 + \sigma_\ep^2)\right)^{(i+j)/2} 
        \frac{\psi}{\phi}
        \left[m_2\xi_{i,1}\xi_{j,1} +m_1\sum_{k = 0, \; k\neq 1}^{\min(i,j)} k!\;\xi_{i,k}\xi_{j,k}\right], \quad \forall i, j \in [\ell],
    \end{align*}
    and for $p \in \mathbb{N}$,
\begin{align*}
    r_p = \begin{cases}
            \frac{p!\psi m_1}{\phi}\left(\frac{c_{\star,1}}{\sqrt{\phi(c_\star^2 + \sigma_\ep^2) + c_{\star,1}^2}}\right)^p    & p \neq 1\\
    \frac{\psi m_2}{\phi}\frac{c_{\star,1}}{ \sqrt{\phi(c_\star^2 + \sigma_\ep^2) + c_{\star,1}^2}}& p = 1
        \end{cases}
\end{align*}
\end{theorem}

\begin{proof}[Proof of Theorem~\ref{thm:general_ell}]
In the regime where $\eta \asymp n^\alpha$ with $\frac{\ell - 1}{2\ell} < \alpha < \frac{\ell}{2\ell + 2}$, according to the equivalence theorem \ref{thm:training_equivalence}, we can replace $\bF$ with $\bF_\ell$ when computing the limiting training loss. 
To compute the limiting training loss difference according to lemma \ref{lemma:train_loss}, we study the matrix 
$\bar\bR = (\bF\bF^\top + \lambda n \bI_n)^{-1}$.
Due to \eqref{eqn:Fell}, we can write
\begin{align*}
    \bF\bF^\top = \bF_0\bF_0^\top &+ \sum_{k = 1}^{\ell}c_1^kc_k \eta^k \tilde\vtheta^{\circ k}(\bF_0 \va^{\circ k})^\top\\
    &+ \sum_{k = 1}^{\ell}c_1^kc_k \eta^k (\bF_0 \va^{\circ k})\tilde\vtheta^{\circ k\top} + \sum_{j= 1}^{\ell}\sum_{i= 1}^{\ell} c_1^{i+j}c_ic_j \eta^{i + j} (\va^{\circ i})^\top(\va^{\circ j})\tilde\vtheta^{\circ i}\tilde\vtheta^{\circ j\top}.
\end{align*}
Defining  the matrix $\bU$ as 
\begin{align*}
    \bU = \left[\underbrace{\;\bF_0\va\;\big|\cdots\big|\;N^{(\ell - 1)/2}\bF_0\va^{\circ \ell}}_{\ell \text{ columns}}\;\bigg|\underbrace{\tilde\vtheta\;\big|\cdots\big|\;\tilde\vtheta^{\circ \ell}}_{\ell \text{ columns}}\right] \in \R^{n \times 2\ell},
\end{align*}
we can write
\begin{align*}
    \bF\bF^\top = \bF_0\bF_0^\top + \bU \bK \bU^\top, \text{ in which  } \bK = \begin{bmatrix}
        \mathbf{0}_{\ell\times \ell} & \bK_o \\ \bK_o & \tilde\bK
    \end{bmatrix} \in \R^{2\ell \times 2\ell},
\end{align*}
where $\bK_o = \text{diag}\left(\frac{c_1 c_1\eta}{N^0}, \dots, \frac{c_1^\ell c_\ell \eta^\ell}{N^{(\ell - 1)/2}}\right)\in \R^{\ell \times \ell},$ and $\tilde\bK \in \R^{\ell \times \ell}$ with  
$[\tilde\bK]_{i, j} = {c_1^{i+j}c_i c_j \eta^{i+j} \langle \va^{\circ i}, \va^{\circ j} \rangle}$, for all $i, j \in [\ell]$. 

Using the Woodbury formula, the matrix $\bar\bR$ can be written in terms of $ \bar\bR_0 = (\bF_0\bF_0^\top + \lambda n \bI_n)^{-1}$
and $\bT = (\bK^{-1}+ \bU^\top \bar\bR_0 \bU)^{-1} \in \R^{2\ell\times 2\ell}$ as
$\bar\bR = \bar\bR_0 - \bar\bR_0 \bU \bT \bU^\top \bar\bR_0$.
Now
\begin{align*}
    \bK^{-1} = \begin{bmatrix}
        \hat\bK& \bK_o^{-1} \\ \bK_o^{-1} & \mathbf{0}_{\ell\times \ell} 
    \end{bmatrix},
    \text{ where } \quad \bK_o^{-1} = \text{diag}\left(\frac{N^0}{c_1 c_1\eta}, \dots, \frac{N^{\frac{\ell - 1}{2}}}{c_1^\ell c_\ell \eta^\ell}\right), 
\end{align*}
and $[\hat \bK]_{i, j} = -N^{(i-1)/2}N^{(j-1)/2}\langle\va^{\circ i}, \va^{\circ j}  \rangle$, for all $i,j\in [\ell]$.
We define $\bM_1, \bM_2, \bM_o \in \R^{\ell \times \ell}$ as the following blocks of $\bT^{-1}$:
\begin{align*}
    \bT^{-1} = \begin{bmatrix}
        \bM_1 & \bM_o\\
        \bM_o & \bM_2
    \end{bmatrix}.
\end{align*}
Hence, we have
\begin{align*}
    \begin{cases}
        [\bM_1]_{i,j} = N^{(i-1)/2}N^{(j-1)/2}\va^{\circ i \top}(\bF_0^\top\bar\bR_0\bF_0 - \bI) \va^{\circ j},\vspace{0.3cm}\\
        [\bM_o]_{i,j} = N^{(i-1)/2}\va^{\circ i \top}\bF_0^\top\bar\bR_0\tilde\vtheta^{\circ j} + o_{\sP}(1),\vspace{0.3cm}\\
        [\bM_2]_{i,j} = \tilde\vtheta^{\circ i \top}\bar\bR_0\tilde\vtheta^{\circ j}.
    \end{cases}
\end{align*}
We can expand the monomials in terms of the Hermite polynomials, for scalars $\xi_{i,k}$, $k\in[i]$, as follows:
\begin{align*}
        (N^{1/2}\va)^{\circ i} = \sum_{k = 0}^{i}\xi_{i,k} H_k(N^{1/2}\va),
\,\quad\text{and }\quad
(\tilde\bX\vbeta)^{\circ i} = \Vert\vbeta\Vert_2^i\sum_{k = 0}^{i}\xi_{i,k} H_k(\tilde\bX\vbeta/\Vert\vbeta\Vert_2).
\end{align*}
Using these, we will analyze each matrix $\bM_1, \bM_2, \bM_o$ separately.
\paragraph{Analysis of $\bM_1$.} It is easily seen that the elements of this matrix are $O_\sP(1)$.

\paragraph{Analysis of $\bM_2$.} To analyze these terms, we need the following lemma, whose proof is deferred to Section \ref{pflemma:hermite_limit_general}.
\begin{lemma}\label{lemma:hermite_limit_general}
    For any $i, j \in \mathbb{N}_0$, we have
    \begin{align*}
    (\tilde \bX  \vbeta)^{\circ i \top}\bar\bR_0(\tilde \bX \vbeta)^{\circ j} \to_P \left(c_{\star,1}^2 + \phi(c_\star^2 + \sigma_\ep^2)\right)^{(i+j)/2} \left[\xi_{i,1}\xi_{j,1} \frac{\psi m_2}{\phi}+\frac{\psi m_1}{\phi}\sum_{k = 0, \; k\neq 1}^{\min(i,j)} k!\;\xi_{i,k}\xi_{j,k}\right].
\end{align*}
\end{lemma}
Defining the matrix $\bar\bM_2 \in \mathbb{R}^{\ell \times \ell}$ with entries

\begin{align*}
    [\bar\bM_2]_{i,j}  = \left(c_{\star,1}^2 + \phi(c_\star^2 + \sigma_\ep^2)\right)^{(i+j)/2} \left[\xi_{i,1}\xi_{j,1} \frac{\psi m_2}{\phi}+\frac{\psi m_1}{\phi}\sum_{k = 0, \; k\neq 1}^{\min(i,j)} k!\;\xi_{i,k}\xi_{j,k}\right],
\end{align*}
for all $i,j\in[\ell]$,
we have $[\bM_2]_{i,j} \to_P [\bar\bM_2]_{i,j}$. Note that we can write
\begin{align*}
    \bar\bM_2 =\frac{\psi}{\phi} \mathbf{B}\mathbf{Z}\mathbf{M}\mathbf{Z}^\top \mathbf{B} + \frac{\psi m_1}{\phi}\ve\ve^\top,
\end{align*}
where we define $b = (c_{\star,1}^2 + \phi(c_\star^2 + \sigma_\ep^2))^{1/2}$, $\textbf{B} = \text{diag}(b^1, \cdots, b^\ell) \in \R^{\ell\times \ell}$,  $\ve = \mathbf{B}[\xi_{1,0},\cdots,\xi_{\ell,0}]^\top$,
\begin{align*}
 \textbf{M} = 
 \begin{bmatrix}
     1!\, m_2 & 0 & \cdots & 0\\
     0 & 2!\, m_1 & \cdots & 0\\
     \vdots & \vdots & \ddots & \vdots\\
     0 & 0 & \cdots & \ell!\, m_1\\
 \end{bmatrix}
 \in \R^{\ell\times\ell},\;\text{and }\;
    \mathbf{Z} = \begin{bmatrix}
        \xi_{1,1} & \cdots & \xi_{1, \ell}\\
        \vdots & \ddots & \vdots\\
        \xi_{\ell,1} & \cdots & \xi_{\ell, \ell}
    \end{bmatrix}\in \R^{\ell\times \ell}.
\end{align*}
Recalling that 
for all $i,j \in \{0,1,\ldots\}$,
$\xi_{i,j}$ are such that
such that for any $p \in \mathbb{N}$ and $x\in \R$, we have $x^p = \sum_{i = 0}^{p} \xi_{p,i} H_i(x)$, 
it follows that the matrix $\bZ$ is lower-triangular with unit diagonal; hence invertible. Thus, 
since $\mathbf{B},\mathbf{M}$ are diagonal with positive entries,
the matrix $\mathbf{B}\mathbf{Z}\mathbf{M}\mathbf{Z}^\top \mathbf{B}$ is positive definite. This implies that $\bar{\mathbf{M}}_2$ is invertible. We will denote $\mathbf{\Omega} = \bar\bM_2^{-1}$.

\paragraph{Analysis of $\bM_o$.} 
We analyze $[\bM_o]_{i,j}$ by writing $N^{(i-1)/2}\va^{\circ i}$ in the Hermite basis, finding
\begin{align*}
    [\bM_o]_{i,j} = \sum_{k = 0}^{i}\frac{\xi_{i,k}}{\sqrt{N}} H_k(N^{1/2}\va)^\top\bF_0^\top\bar\bR_0\tilde\vtheta^{\circ j} + o_{\sP}(1).
\end{align*}
The terms with $k>0$ are all $o_{\sP}(1)$ because $\frac{H_k(N^{1/2}\va)}{\sqrt{N}}$ is a norm $O_\sP(1)$ vector with mean zero, independent from the vector $\bF_0^\top\bar\bR_0\tilde\vtheta^{\circ j}$ with norm $O_{\sP}(1)$. Thus, $[\bM_o]_{i,j} = o_{\sP}(1)$. The term with $k = 0$ can also be shown to be $o_\sP(1)$ by using that the linearized $\bF_0$ is left-orthogonally invariant, via an argument identical to the one used to analyze \eqref{eq:left-ortho-argument}. 

Hence, putting these together, the matrix $\bT$ can be written as
\begin{align*}
    \bT = \begin{bmatrix}
        \bM_1^{-1} & \mathbf{0}_{\ell \times \ell}\\
        \mathbf{0}_{\ell \times \ell} & \bar\bM_2^{-1}
    \end{bmatrix} + o_{\sP}(1).
\end{align*}

Using lemma \ref{lemma:train_loss}, we can write the training loss difference as
$\gL_{\rm tr}(\bF_0) - \gL_{\rm tr}(\bF) =   \lambda \vy^\top \bar\bR_0 \bU \bT\bU^\top \bar\bR_0 \vy$.
Plugging in the teacher function $f_\star$, we find
\begin{align*}
    \gL_{\rm tr}(\bF_0) - \gL_{\rm tr}(\bF) &=    \sum_{p,q}\lambda c_{\star,p} c_{\star,q} H_p(\tilde\vtheta_\star)^\top \bar\bR_0 \bU \bT\bU^\top \bar\bR_0 H_q(\tilde\vtheta_\star) \\
    &+ 2\lambda \sum_{p} \left(c_{\star, p} H_p(\tilde \vtheta_\star)^\top \bar\bR_0 \bU \bT \bU^\top \bar \bR_0 \boldsymbol{\ep} \right)+  \lambda \boldsymbol{\ep}^\top \bar\bR_0 \bU \bT \bU^\top \bar\bR_0 \boldsymbol{\ep}.
\end{align*}
Note that the second term can be shown to be $o_\sP(1)$ because $\boldsymbol{\ep}\sim \normal(0, \sigma_\ep^2 \bI_n)$ and it is independent from $H_p(\tilde \vtheta_\star)^\top \bar\bR_0 \bU \bT \bU^\top \bar \bR_0$, and $\Vert H_p(\tilde \vtheta_\star)^\top \bar\bR_0 \bU \bT \bU^\top \bar \bR_0\Vert_{\rm op} = O_\sP(1/\sqrt{N})$ with a simple orderwise analysis. The third can also be shown to be $o_\sP(1)$ by noting that $\boldsymbol{\ep}$ is independent from $\bar\bR_0 \bU$,  $\Vert\bar\bR_0 \bU\Vert_{\rm op} = O_\sP(1/\sqrt{n})$ and that the elements of $\bT$ are $O_\sP(1)$.

To analyze the first term, we define $\delta_{p, q} = H_p(\tilde\vtheta_\star)^\top \bar\bR_0 \bU \bT\bU^\top \bar\bR_0 H_q(\tilde\vtheta_\star)$ for all non-negative integers $p,q$. To analyze such terms, we first expand $\bU\bT\bU^\top$ as
\begin{align*}
    \bU \bT \bU^\top &= \sum_{i = 1}^{\ell}\sum_{j = 1}^{\ell} N^{(i+j)/2-1}[\bM_1^{-1}]_{i, j} (\bF_0 \va^{\circ i})(\bF_0 \va^{\circ j})^\top+ \sum_{i = 1}^{\ell}\sum_{j = 1}^{\ell} [\bar\bM_2^{-1}]_{i, j} \tilde\vtheta^{\circ i}\tilde\vtheta^{\circ j\top}.
\end{align*}
Thus, for any $p, q \in \mathbb{N}_0$, the terms $\delta_{p,q}$ can be written as
\begin{align*}
    \delta_{p, q} 
    &= 
    \sum_{i = 1}^{\ell}\sum_{j = 1}^{\ell} N^{(i+j)/2-1}[\bM_1^{-1}]_{i, j} H_p(\tilde\vtheta_\star)^\top \bar\bR_0  (\bF_0 \va^{\circ i})(\bF_0 \va^{\circ j})^\top\bar\bR_0 H_q(\tilde\vtheta_\star)\\
    &\hspace{1cm}+\sum_{i = 1}^{\ell}\sum_{j = 1}^{\ell}\; [\bar\bM_2^{-1}]_{i, j}H_p(\tilde\vtheta_\star)^\top \bar\bR_0\tilde\vtheta^{\circ i}\tilde\vtheta^{\circ j\top}\bar\bR_0 H_q(\tilde\vtheta_\star).
\end{align*}

By an argument identical to the argument for the terms in $\bM_o$, the first sum goes to zero in probability. Denoting $\vbeta/\Vert\vbeta\Vert_2 := \tilde\vbeta$, we can expand $(\tilde\bX\vbeta)^{\circ i} = \Vert\vbeta\Vert_2^i\sum_{k = 0}^{i}\xi_{i,k} H_k(\tilde\bX\vbeta/\Vert\vbeta\Vert_2)$,
To analyze $\delta_{p,q}$, we need the following result, whose proof is deferred to Section \ref{pflemma:mixed_limit_general}.
\begin{lemma}
\label{lemma:mixed_limit_general}
For any $p, q \in \mathbb{N}_0$, we have
    \begin{align*}
        H_p(\tilde \bX \vbeta_\star)^\top \bar\bR_0\, H_q(\tilde \bX \tilde\vbeta) \to_P \begin{cases}
            \frac{p!\psi m_1}{\phi}\left(\frac{c_{\star,1}}{\sqrt{\phi(c_\star^2 + \sigma_\ep^2) + c_{\star,1}^2}}\right)^p    & p = q\neq 1\\
    \frac{\psi m_2}{\phi}\frac{c_{\star,1}}{ \sqrt{\phi(c_\star^2 + \sigma_\ep^2) + c_{\star,1}^2}}& p = q = 1\\
    0 & p\neq q.
        \end{cases}
    \end{align*}
\end{lemma}

We can now use Lemma \ref{lemma:mixed_limit_general} and that $\Vert\vbeta\Vert_2 \to_P \left({\phi(c_\star^2 + \sigma_\ep^2) + c_{\star,1}^2}\right)^{1/2}$ to write
\begin{align*}
    \delta_{p, q} 
    &=\sum_{i = 1}^{\ell}\sum_{j = 1}^{\ell}\; [\bar\bM_2^{-1}]_{i, j}\Vert\vbeta\Vert_2^{i+j}
    \xi_{i,p}\xi_{j,q} H_p(\tilde\vtheta_\star)^\top \bar\bR_0H_{p}(\tilde\bX\tilde\vbeta)\cdot H_{q}(\tilde\bX\tilde\vbeta)^\top\bar\bR_0 H_q(\tilde\vtheta_\star)+o_\sP(1)\\
    &=\sum_{i = 1}^{\ell}\sum_{j = 1}^{\ell}\; [\mathbf{\Omega}]_{i, j}\left({\phi(c_\star^2 + \sigma_\ep^2) + c_{\star,1}^2}\right)^{(i+j)/2}
    \xi_{i,p}\xi_{j,q} r_pr_q+o_\sP(1),
\end{align*}
for $p, q \in [\ell]$, which concludes the proof.
\end{proof}

\section{Infinite sample limit}

In the infinite sample limit, where $n\gg N, d$, we have $\phi \to 0$. In this extreme case, the expressions for $m_1, m_2$ will further simplify as $m_1, m_2  \to \phi / \lambda \psi$. Note that in this limit, we have $\mathcal{L}_{\mathrm{tr}}(\mathbf{F}_0) \to \sigma_\epsilon^2 + c_\star^2$ (see e.g., \cite[Section 6]{mei2022generalization}.
Using Corollary \ref{corollay}, we see that for example when $\ell = 2$, we have $\mathcal{L}(\mathbf{F}) \to \sigma_\epsilon^2 + \frac{2 c_{\star,2}^2}{3} + c_{\star,>2}^2$. 
In particular, the term corresponding to the linear component of the teacher function in $\mathcal{L}(\mathbf{F}_0)$ cancels out with the corresponding term in $\Delta_2$.

\section{Proof of Theorem \ref{thm:test_risk}}
Let $(\vx_{\rm te}, y_{\rm te})$ follow the model from (\ref{eqn:datagen}). Recall that the test error can be written as
\begin{align}
\label{eq:test_decomposition}
    \mathcal{L}_{\rm te}(\hat\va{(\bF)}) &= \E_{\vx_{\rm te}, y_{\rm te}}(y_{\rm te} - \hat\va^\top \sigma(\bW \vx_{\rm te}))^2 = \E_{\vx_{\rm te}, y_{\rm te}}(y_{\rm te}^2) +   \hat\va^\top \bSigma_\vf \hat\va - 2\hat\va^\top \vmu_\vf,
\end{align}
 where $\bSigma_\vf = \E_{\vx_{\rm te}} \left[\sigma(\bW \vx_{\rm te})\sigma(\bW \vx_{\rm te})^\top\right]$ and $\vmu_\vf = \E_{\vx_{\rm te}, y_{\rm te}}\left[y_{\rm te}\sigma(\bW \vx_{\rm te})\right]$. First, we will show that in the definition of $\bSigma_\vf$ and $\vmu_\vf$,  we can replace the test feature $\sigma(\bW \vx_{\rm te})$ with the spiked approximation from Theorem~\ref{thm:spectrum_of_feature_matrix}
\begin{align}
    \label{eq:spiked_test_feature}
    \vf_\ell = \sigma(\bW_0 \vx_{\rm te}) + \sum_{k = 1}^{\ell} c_1^k c_k \eta^k (\vbeta^\top \vx_{\rm te})^k \va^{\circ k},
\end{align}
without changing the test error. To do this, consider an independent test set $\{\vx_{{\rm te}, i}, y_{{\rm te}, i}\}_{i = 1}^{n_{\rm te}}$ with $n_{\rm te}$ test samples following the data generation distribution in \eqref{eqn:datagen} and we define $\bX = [\vx_{\rm{te}, 1}, \dots, \vx_{\rm{te}, n}]^\top \in \mathbb{R}^{n_{\rm te} \times d}$. Let $\bF_{\rm te} = \sigma({\bX_{\rm te} \bW^\top}) \in \R^{n_{\rm te} \times N}$ be the test feature matrix. We can write the test error as 
\begin{align*}
    \mathcal{L}_{\rm te}(\hat\va({\bF})) = \lim_{n_{\rm te} \to \infty} \frac{1}{n_{\rm te}} \Vert\vy_{\rm te} - \bF_{\rm te}\hat\va(\bF)\Vert_{2}^2.
\end{align*}
Now, consider the spiked approximation of the test feature matrix $\bF_{\rm te, \ell} \in \R^{n_{\rm te}, N}$ where each row of $\bF_{\rm te, \ell}$ follows the approximation in \eqref{eq:spiked_test_feature}. Using Theorem~\ref{thm:spectrum_of_feature_matrix}, we have $\Vert\bF_{\rm te} - \bF_{\rm te, \ell}\Vert_{\rm op} = o_\sP(\sqrt{n_{\rm te}})$. Thus, 
\begin{align*}
     \left|\frac{1}{n_{\rm te}} \Vert\vy_{\rm te} - \bF_{\rm te}\hat\va(\bF)\Vert_{2}^2 -  \frac{1}{n_{\rm te}} \Vert\vy_{\rm te} - \bF_{\rm te, \ell}\hat\va(\bF)\Vert_{2}^2\right| &\leq \frac{1}{n_{\rm te}} \left|\Big\langle(\bF_{\rm te} - \bF_{\rm te, \ell})\hat\va(\bF), 2\vy_{\rm te} -  \bF_{\rm te} - \bF_{\rm te, \ell}\Big\rangle\right|\\
     &\leq \frac{C}{\sqrt{n_{\rm te}}} \Vert\bF_{\rm te} - \bF_{\rm te, \ell}\Vert_{\rm op} \cdot \Vert \hat\va(\bF)\Vert_2 \to_\sP 0,
\end{align*}
where the last line is due to the fact that under the assumption of Theorem~\ref{thm:test_equivalence}, we have $ \Vert\vy_{\rm te} - \bF_{\rm te}\hat\va(\bF)\Vert_{2}^2 = O_\sP(\sqrt{n_{\rm te}})$, and that $\Vert\hat\va(\bF)\Vert_2 = O_\sP(1)$ from the proof of Theorem \ref{thm:training_equivalence}. Thus in the test error, we can replace the test features with their spiked approximation without changing the limiting test error. With this, we can write
\begin{align}
    \label{eq:sigma_f}
    \bSigma_\vf &= \bSigma_\vf^0 + \E_{\vx_{\rm te}} \left[\sum_{k = 1}^{\ell} c_1^k c_k \eta^k (\vx_{\rm te}^\top \vbeta)^k \left(\va^{\circ k} \sigma(\bW_0 \vx_{\rm te})^\top + \sigma(\bW_0 \vx_{\rm te})\va^{\circ k\top}\right)\right]
    \nonumber\\&+ \E_{\vx_{\rm te}} \left[\sum_{i = 1}^{\ell}\sum_{j = 1}^{\ell} c_1^{i+j} \eta^{i+j} (\vx_{\rm te}^\top \vbeta)^{i+j}c_i c_j (\va^{\circ i}\va^{\circ j \top})\right]
    \nonumber\\= \bSigma_\vf^0 &+ \E_{\vx_{\rm te}} \left[\sum_{k = 1}^{\ell} c_1^k c_k \eta^k  \left(\va^{\circ k} \vnu_k^\top + \vnu_k\va^{\circ k\top}\right)\right]
    + \E_{\vx_{\rm te}} \left[\sum_{i = 1}^{\ell}\sum_{j = 1}^{\ell} c_1^{i+j} \eta^{i+j} c_i c_j \aleph_{i+j} (\va^{\circ i}\va^{\circ j \top})\right],  
\end{align}
where $\bSigma_\vf^0 = \E_{\vx_{\rm te}} \left[\sigma(\bW_0 \vx_{\rm te})\sigma(\bW_0 \vx_{\rm te})^\top\right]$, $\vnu_k = \E_{\vx_{\rm te}}[(\vx_{\rm te}^\top \vbeta)^k\sigma(\bW_0\vx_{\rm te})]$, and ${\aleph_{k}} = \E_{\vx_{\rm te}}(\vx_{\rm te}^\top \vbeta)^k$ for all $k\in[\ell]$. 
Also,
\begin{align}
    \label{eq:mu_f}
    \vmu_\vf &= \E_{\vx_{\rm te}, y_{\rm te}}\left[\,y_{\rm te}\sigma(\bW \vx_{\rm te})\,\right] = \vmu_\vf^{0} + \sum_{k = 1}^{\ell} c_1^k c_k  \tau_k \eta^k\va^{\circ k},
\end{align}
where $ \vmu_\vf^{0} = \E_{\vx_{\rm te}, y_{\rm te}}\left[\,y_{\rm te}\sigma(\bW_0 \vx_{\rm te})\,\right] $, and $\tau_k = \E_{\vx_{\rm te}}\left[y_{\rm te} (\vx_{\rm te}^\top\vbeta)^k\right]$.

\subsection{Proof for $\ell = 1$}
Without loss of generality, assume that $c_1 = 1$.  First, note that 
\begin{align*}
    &\vnu_1 = \E_{\vx_{\rm te}}[(\vx_{\rm te}^\top\vbeta) \sigma(\bW_0 \vx_{\rm te})] =  \bW_0 \vbeta,\\
    &\aleph_1 = \E_{\vx_{\rm te}}(\vx_{\rm te}^\top \vbeta) = 0, \quad \text{and} \quad \aleph_2 = \E_{\vx_{\rm te}}(\vx_{\rm te}^\top \vbeta)^2 = \Vert\vbeta\Vert_2^2 \to_P c_{\star, 1}^2 + \phi(c_{\star}^2 + \sigma_\ep^2),\\
    &\tau_1 = \E\left[y_{\rm te} (\vx_{\rm te}^\top\vbeta)\right] = c_{\star, 1} \vbeta_\star^\top \vbeta \to_P c_{\star,1}^2,
\end{align*}
where the convergence follows from \eqref{prop:alignment} and the computations are in its proof.

Thus, using \eqref{eq:sigma_f}, we have
\begin{align*}
    &\vmu_\vf = \vmu_\vf^0 +  c_{\star, 1}^2 \eta \va,\quad
    \bSigma_{\vf} =\bSigma_{\vf}^0 + \eta \left(\va (\bW_0\vbeta)^\top + \bW_0\vbeta \va^\top \right) +  \eta^2 \left(c_{\star, 1}^2 + \phi(c_{\star}^2 + \sigma_\ep^2)\right)\va\va^\top.
\end{align*}

From Section~\ref{eq:proof_train_l1}, we have $\bar\bR = \bar\bR_0 - \bar\bR_0 \bU \bT \bU^\top \bar\bR_0$, where $\bT$ is defined in \eqref{eq:T_ell1} and $\bU = [\;\bF_0\va\; |\; \tilde \bX \vbeta\;]$. In Section~\ref{eq:proof_train_l1}  it was shown that
\begin{align}
    \label{eq:mat_T}
    \bT =\frac{\begin{bmatrix}
        \vbeta^\top \tilde \bX^\top \bar\bR_0 \tilde \bX \vbeta & -\frac{1}{\eta}-\va^\top\bF_0^\top  \bar\bR_0 \tilde \bX \vbeta~\vspace{0.15cm}\\
        -\frac{1}{\eta} - \vbeta^\top \tilde \bX^\top\bar\bR_0\bF_0 \va & 
        \va^\top \bF_0^\top \bar\bR_0 \bF_0 \va - \Vert \va\Vert_2^2
    \end{bmatrix}}{\left(\vbeta^\top \tilde \bX^\top \bar\bR_0 \tilde \bX \vbeta\right)\left(\va^\top \bF_0^\top \bar\bR_0 \bF_0 \va - \Vert \va\Vert_2^2\right) - \left(\frac{1}{\eta} +\va^\top\bF_0^\top  \bar\bR_0 \tilde \bX \vbeta \right)^2},
\end{align}
where  $T_{11}, T_{22} = \Theta_\sP(1)$, and $T_{21}, T_{12} = O_\sP(1/\eta)$. 

Now, we are ready to study the terms $E_1 = \E_{\vx_{\rm te}, y_{\rm te}}(y_{\rm te}^2) $, $E_2 =  \hat\va^\top \bSigma_\vf \hat\va$, and $E_3 = - 2\hat\va^\top \vmu_\vf$ 
that appear in the test error in \eqref{eq:test_decomposition}.

\subsubsection{Analysis of $E_1$.} 
This term can be readily computed as 
\begin{align}
    \label{eq:E1}
    E_1 = \E(y_{\rm te}^2) = \sigma_\ep^2 + \E_{\vx_{\rm te}}\left(\sigma_\star(\vbeta_\star^\top\vx_{\rm te})\right)^2 = \sigma_\ep^2 + c_{\star}^2.
\end{align}

\subsubsection{Analysis of $E_2$.} Recall that $\hat\va = (\bF^\top\bF + \lambda n \bI_N)^{-1} \bF^\top \tilde\vy$, $\bar\bR_0 = (\bF_0^\top\bF_0 + \lambda n \bI_n)^{-1}$, and $\bar\bR = (\bF\bF^\top + \lambda n \bI_n)^{-1}$. Using these, we can write
\begin{align}
    \label{eq:l1_E2}
    E_2 &= \hat\va^\top \bSigma_\vf \hat\va = \tilde\vy^\top \bF (\bF^\top\bF + \lambda n \bI_N)^{-1}\bSigma_{\vf}(\bF^\top\bF + \lambda n \bI_N)^{-1}\bF^\top \tilde \vy\nonumber\\
    &= \tilde\vy^\top (\bF\bF^\top + \lambda n \bI_n)^{-1}\bF\bSigma_{\vf}\bF^\top(\bF\bF^\top + \lambda n \bI_n)^{-1} \tilde \vy =\tilde\vy^\top \bar\bR \bF\bSigma_{\vf}\bF^\top\bar\bR \tilde \vy.
\end{align}
We have $\tilde \vy = \vf_\star + \vepsilon$, thus
\begin{align}
    \label{eq:E2-decomp}
    E_2 =\vf_\star^\top \bar\bR \bF\bSigma_{\vf}\bF^\top\bar\bR \vf_\star +  2\vepsilon^\top \bar\bR \bF\bSigma_{\vf}\bF^\top\bar\bR \vf_\star+\vepsilon^\top \bar\bR \bF\bSigma_{\vf}\bF^\top\bar\bR \vepsilon.
\end{align}

We will now analyze the terms in \eqref{eq:E2-decomp}.  First note that $\vepsilon^\top \bar\bR \bF\bSigma_{\vf}\bF^\top\bar\bR \vf_\star = o_\sP(1)$ using a simple order-wise argument. To analyze the third term in \eqref{eq:E2-decomp}, we write
\begin{align*}
    \bF^\top \bar\bR  \,\vepsilon &= \left(\bF_0 +  \eta (\tilde \bX \vbeta)\va^\top\right)^\top\bar\bR\,\vepsilon =  \bF_0^\top \bar\bR  \,\vepsilon + \eta(\vbeta^\top\tilde \bX^\top\bar\bR\,\vepsilon) \va.
\end{align*}
Using a simple order-wise analysis, we have $\vbeta^\top\tilde \bX^\top\bar\bR\,\vepsilon = O_\sP(1/\sqrt{n})$. 
Thus, the third term can be written as
\begin{align*}
    \vepsilon^\top \bar\bR \bF\bSigma_{\vf}\bF^\top\bar\bR\, \vepsilon &= \underbrace{\vepsilon^\top \bar\bR \bF_0\bSigma_{\vf}\bF^\top_0\bar\bR\, \vepsilon}_{q_1} + \underbrace{2 \eta(\vbeta^\top\tilde \bX^\top\bar\bR\,\vepsilon)\, \vepsilon^\top \bar\bR \bF_0\bSigma_{\vf}\va}_{q_2} + \underbrace{\eta^2 (\vbeta^\top\tilde \bX^\top\bar\bR\,\vepsilon)^2 \va^\top \bSigma_\vf\va}_{q_3}.
\end{align*}
The term $q_1$ can be computed as
\begin{align*}
    q_1 &= \vepsilon^\top \bar\bR \bF_0\bSigma_{\vf}\bF^\top_0\bar\bR\, \vepsilon
    = \vepsilon^\top \bar\bR \bF_0\left(\bSigma_{\vf}^0 +  \eta \left(\va (\bW_0\vbeta)^\top + \bW_0\vbeta \va^\top \right) +  \eta^2 \Vert\vbeta\Vert_2^2 \va\va^\top\right)\bF^\top_0\bar\bR\, \vepsilon\\
    &= \vepsilon^\top \bar\bR \bF_0\bSigma_{\vf}^0\bF^\top_0\bar\bR\, \vepsilon + o_\sP(1)
    = \vepsilon^\top \bar\bR_0 \bF_0\bSigma_{\vf}^0\bF^\top_0\bar\bR_0\, \vepsilon + o_\sP(1),
\end{align*}
where in the last line we have used
that $\bar\bR = \bar\bR_0 - \bar\bR_0 \bU\bT\bU^\top \bar\bR_0$ and an order-wise analysis for various terms. 
To analyze $q_2$, note that
\begin{align*}
    \vepsilon^\top \bar\bR \bF_0\bSigma_{\vf}\va = \vepsilon^\top \bar\bR \bF_0\left(\bSigma_{\vf}^0 +  \eta \left(\va (\bW_0\vbeta)^\top + \bW_0\vbeta \va^\top \right) +  \eta^2 \Vert\vbeta\Vert_2^2 \va\va^\top\right)\va = o_\sP(1),
\end{align*}
using a simple order-wise analysis, thus $q_2 = o_\sP(1)$. Similarly, for $q_3$, we have
\begin{align*}
    q_3 &= \eta^2 (\vbeta^\top\tilde \bX^\top\bar\bR\,\vepsilon)^2 \va^\top \left(\bSigma_{\vf}^0 +  \eta \left(\va (\bW_0\vbeta)^\top + \bW_0\vbeta \va^\top \right) +  \eta^2 \Vert\vbeta\Vert_2^2 \va\va^\top\right)\va = o_\sP(1).
\end{align*}
Hence, summing everything up 
\begin{align}
    \label{eq:E_2_de}
    E_2 = \vepsilon^\top \bar\bR_0 \bF_0\bSigma_{\vf}^0\bF^\top_0\bar\bR_0\, \vepsilon + \vf_\star^\top \bar\bR \bF\bSigma_{\vf}\bF^\top\bar\bR \vf_\star + o_\sP(1).
\end{align}

Next, we will study the term $\vf_\star^\top \bar\bR \bF\bSigma_{\vf}\bF^\top\bar\bR \vf_\star$. We can write
\begin{align}
    \label{eq:expand_a_hat}
    \hat\va = \bF^\top \bar\bR \vf_\star &= \left(\bF_0 +  \eta (\tilde \bX \vbeta)\va^\top\right)^\top\left(\bar\bR_0 - \bar\bR_0 \bU \bT \bU^\top \bar\bR_0\right)\vf_\star\nonumber\nonumber\\
    &= \underbrace{\bF_0^\top \bar\bR_0 \vf_\star}_{\vp_1} \underbrace{- \bF_0^\top \bar\bR_0 \bU \bT \bU^\top \bar\bR_0\vf_\star}_{\vp_2} + \underbrace{\eta\left[(\tilde \bX \vbeta)^\top\bar\bR_0\vf_\star - (\tilde \bX\vbeta)^\top \bar\bR_0 \bU \bT \bU^\top \bar\bR_0\vf_\star\right] \va}_{\vp_3}.
\end{align}
Thus, defining $E_2^{ij}=\vp_i^\top \bSigma_{\vf} \vp_j$ for $i, j \in [3]$, we have $E_{2} = \sum_{i, j = 1}^{3} E_{2}^{ij}$.  In the following sections, we will analyze each term in this sum separately.

\paragraph{Preliminary Computations.} Before starting the computation, we define $K$ as 
\begin{align}
\label{def:K}
    K := \vbeta^\top\tilde \bX^\top\bar\bR_0\vf_\star - \vbeta^\top\tilde \bX^\top \bar\bR_0 \bU \bT \bU^\top \bar\bR_0\vf_\star.
\end{align}
Recalling that $T_{12} = T_{21} = O(1/\eta)$ and $\vbeta^\top \tilde\bX^\top \bar\bR_0 \bF_0 \va = O(1/\sqrt{n})$, the variable $K$ can be simplified as follows:
\begin{align*}
    K &= \vbeta^\top\tilde \bX^\top\bar\bR_0\vf_\star- T_{22} (\vbeta^\top\tilde \bX^\top \bar\bR_0 \tilde \bX\vbeta) (\vbeta^\top\tilde\bX^\top\bar\bR_0 \vf_\star) + O_\sP\left(\frac{1}{\eta\sqrt{n}}\right).
\end{align*}
Also, from the definition of the matrix $\bT$ in \eqref{eq:mat_T}, we have
\begin{align*}
     T_{22} 
    &= \frac{1}{\vbeta^\top \tilde \bX^\top \bar\bR_0 \tilde \bX \vbeta} + \frac{1}{ (\vbeta^\top \tilde \bX^\top \bar\bR_0 \tilde \bX \vbeta)^2(\va^\top \bF_0^\top \bar\bR_0 \bF_0 \va - \Vert \va\Vert_2^2)}\cdot \frac{1} {\eta^2} + o_\sP(1/\eta^2).
\end{align*}
Hence, putting everything together, and using Lemma \ref{lemma:l1_limits}, we can write
\begin{align}
    \label{eq:K_limit}
    K &= -\frac{\vbeta^\top \tilde\bX^\top \bar\bR_0 \vf_\star}{  (\vbeta^\top \tilde\bX^\top \bar\bR_0 \tilde\bX\vbeta)(\va^\top \bF_0^\top \bar\bR_0 \bF_0 \va - \Vert \va\Vert_2^2)}\cdot \frac{1} {\eta^2} + o_\sP(1/\eta^2). 
\end{align}

Next, we will study the limit of $\eta T_{12}$. For this, we can use \eqref{eq:mat_T} to write
\begin{align}
    \label{eq:T12}
    T_{12} = \frac{-1}{ (\vbeta^\top \tilde\bX^\top \bar\bR_0 \tilde\bX\vbeta)(\va^\top \bF_0^\top \bar\bR_0 \bF_0 \va - \Vert \va\Vert_2^2)} \cdot \frac{1}{\eta} + O_\sP(1/\eta^3).
\end{align}

\paragraph{Analysis of $E_{2}^{11}$.} Noting that $\eta = o(n^{1/4})$ and $\va^\top\bF_0^\top \bar\bR_0 \vf_\star = O_\sP(1/\sqrt{n})$, we can simplify this term as follows:
\begin{align}
    \label{eq:E2^11}
    E_{2}^{11} &= \vf_\star^\top \bar\bR_0 \bF_0 \bSigma_{\vf}\bF_0^\top \bar\bR_0 \vf_\star\nonumber\\
    &= \vf_\star^\top \bar\bR_0 \bF_0 \left(\bSigma_{\vf}^0 +  \eta \left(\va (\bW_0\vbeta)^\top + \bW_0\vbeta \va^\top \right) +  \eta^2 \Vert\vbeta\Vert_2^2 \va\va^\top\right)\bF_0^\top \bar\bR_0 \vf_\star\nonumber\\
    &= \vf_\star^\top \bar\bR_0 \bF_0 \bSigma_{\vf}^0\bF_0^\top \bar\bR_0 \vf_\star + o_\sP(1).
\end{align}

\paragraph{Analysis of $E_{2}^{12}$ and $E_{2}^{21}$.} By expanding $\bSigma_\vf$ and $\bU\bT\bU^\top$, we have
\begin{align*}
    E_{2}^{21} &= E_{2}^{12} = - \vf_\star^\top \bar\bR_0 \bF_0 \bSigma_{\vf}\bF_0^\top \bar\bR_0 \bU \bT \bU^\top \bar\bR_0\vf_\star\\
    &= - \vf_\star^\top \bar\bR_0 \bF_0 \left(\bSigma_{\vf}^0 +  \eta \left(\va (\bW_0\vbeta)^\top + \bW_0\vbeta \va^\top \right) +  \eta^2 \Vert\vbeta\Vert_2^2 \va\va^\top\right)\bF_0^\top \bar\bR_0 \bU \bT \bU^\top \bar\bR_0\vf_\star\\
    &= E_2^{21(1)}+E_2^{21(2)}+E_2^{21(3)}+E_2^{21(4)},
\end{align*}
in which
\begin{align*}
    E_2^{21(1)} &= - \vf_\star^\top \bar\bR_0 \bF_0 \bSigma_{\vf}^0\bF_0^\top \bar\bR_0\bU \bT \bU^\top \bar\bR_0\vf_\star,\\[0.2cm]
    E_2^{21(2)} &= -\eta\,(\vf_\star^\top \bar\bR_0 \bF_0 \va) (\vbeta^\top\bW_0^\top\bF_0^\top \bar\bR_0\bU \bT \bU^\top \bar\bR_0\vf_\star),\\[0.2cm]
    E_2^{21(3)} &= -\eta(\vf_\star^\top \bar\bR_0 \bF_0 \bW_0\vbeta)(\va^\top \bF_0^\top \bar\bR_0\bU\bT\bU^\top \bar\bR_0 \vf_\star),\\[0.2cm]
    E_2^{21(4)} &= -\eta^2 (\vf_\star^\top \bar\bR_0 \bF_0\va)(\va^\top\bF_0^\top \bar\bR_0\bU \bT \bU^\top \bar\bR_0\vf_\star).
\end{align*} 
These terms can be simplified as follows. By expanding $\bU\bT\bU^\top$, we have
\begin{align*}
 E_2^{21(1)} &= - \vf_\star^\top \bar\bR_0 \bF_0 \bSigma_{\vf}^0\bF_0^\top \bar\bR_0\bU \bT \bU^\top \bar\bR_0\vf_\star\\ 
 &= - \vf_\star^\top \bar\bR_0 \bF_0 \bSigma_{\vf}^0\bF_0^\top \bar\bR_0\times\\&\hspace{1cm}\left[T_{11} (\bF_0\va)(\bF_0\va)^\top + T_{12}(\bF_0\va)(\tilde\bX\vbeta)^\top + T_{21} (\tilde\bX\vbeta)(\bF_0\va)^\top+T_{22}(\tilde\bX\vbeta)(\tilde\bX\vbeta)^\top\right] \bar\bR_0\vf_\star \\
 &= -T_{22}\left( \vf_\star^\top \bar\bR_0 \bF_0 \bSigma_{\vf}^0\bF_0^\top \bar\bR_0\tilde\bX\vbeta\right)\left(\vbeta^\top\tilde\bX^\top\bar\bR_0\vf_\star\right) + o_{\sP}(1),
\end{align*}
where the last line uses
that $\vf_\star^\top \bar\bR_0 \bF_0 \va = o_{\sP}(1)$ and $\vf_\star^\top \bar\bR_0 \bF_0 \bSigma_{\vf}^0\bF_0^\top \bar\bR_0\bF_0\va = o_\sP(1)$. Next, noting that $\eta(\vf_\star^\top \bar\bR_0 \bF_0 \va) = o_{\sP}(1)$ and $\vbeta^\top\bW_0^\top\bF_0^\top \bar\bR_0\bU \bT \bU^\top \bar\bR_0\vf_\star = \Theta_\sP(1)$, we can write
\begin{align*}
    E_2^{21(2)} = -\eta\,(\vf_\star^\top \bar\bR_0 \bF_0 \va) (\vbeta^\top\bW_0^\top\bF_0^\top \bar\bR_0\bU \bT \bU^\top \bar\bR_0\vf_\star) = o_{\sP}(1).
\end{align*}
Using a similar argument, we have
\begin{align*}
E_2^{21(3)} &= -\eta(\vf_\star^\top \bar\bR_0 \bF_0 \bW_0\vbeta)(\va^\top \bF_0^\top \bar\bR_0\bU\bT\bU^\top \bar\bR_0 \vf_\star)\\
&=-\eta T_{12}(\vf_\star^\top \bar\bR_0 \bF_0 \bW_0\vbeta)(\va^\top \bF_0^\top \bar\bR_0\bF_0\va)(\vf_\star^\top \bar\bR_0 \tilde\bX\vbeta) + o_\sP(1).
\end{align*}
Finally, again note that $\eta^2(\vf_\star^\top \bar\bR_0 \bF_0 \va)^2 = o_\sP(1)$. Thus,
\begin{align*}
    E_2^{21(4)} = -\eta^2 (\vf_\star^\top \bar\bR_0 \bF_0\va)(\va^\top\bF_0^\top \bar\bR_0\bU \bT \bU^\top \bar\bR_0\vf_\star) = o_{\sP}(1).
\end{align*}
Putting all together, we arrive at 
\begin{align}
    \label{eq:E2^12}
    E_{2}^{21} = E_{2}^{12} &= -T_{22}\left( \vf_\star^\top \bar\bR_0 \bF_0 \bSigma_{\vf}^0\bF_0^\top \bar\bR_0\tilde\bX\vbeta\right)\left(\vf_\star^\top\bar\bR_0\tilde\bX\vbeta\right)\nonumber\\
    &\hspace{2cm}-\eta T_{12}(\vf_\star^\top \bar\bR_0 \bF_0 \bW_0\vbeta)(\va^\top \bF_0^\top \bar\bR_0\bF_0\va)(\vf_\star^\top \bar\bR_0 \tilde\bX\vbeta)+o_\sP(1).
\end{align}

\paragraph{Analysis of $E_{2}^{22}$.} Once again, by expanding $\bSigma_\vf$ and $\bU\bT\bU^\top$, we have
\begin{align*}
    E_2^{22} &= \vf_\star^\top \bar\bR_0 \bU\bT \bU^\top \bar\bR_0 \bF_0\bSigma_\vf \bF_0^\top \bar\bR_0 \bU \bT \bU^\top \bar\bR_0\vf_\star\\
    &= \vf_\star^\top \bar\bR_0 \left[T_{11} (\bF_0\va)(\bF_0\va)^\top + T_{12}(\bF_0\va)(\tilde\bX\vbeta)^\top + T_{21} (\tilde\bX\vbeta)(\bF_0\va)^\top+T_{22}(\tilde\bX\vbeta)(\tilde\bX\vbeta)^\top\right] \bar\bR_0 
    \bF_0
    \\
    &\hspace{0.5cm}\times\left[\bSigma_{\vf}^0 +  \eta \left(\va (\bW_0\vbeta)^\top + \bW_0\vbeta \va^\top \right) +  \eta^2 \Vert\vbeta\Vert_2^2 \va\va^\top\right]\\
    &\hspace{0.5cm}\times\bF_0^\top \bar\bR_0 \left[T_{11} (\bF_0\va)(\bF_0\va)^\top + T_{12}(\bF_0\va)(\tilde\bX\vbeta)^\top + T_{21} (\tilde\bX\vbeta)(\bF_0\va)^\top+T_{22}(\tilde\bX\vbeta)(\tilde\bX\vbeta)^\top\right] \bar\bR_0\vf_\star.
\end{align*}
Note that $\vf_\star^\top \bar\bR_0 \bF_0 \va = O_\sP(1/\sqrt{n})$ and $\eta^2 = o(\sqrt{n})$. Hence, 
\begin{align}
    \label{eq:E2^22}
    E_2^{22} &= (\vf_\star^\top \bar\bR_0 \tilde\bX\vbeta)^2 \left[T_{21}(\bF_0\va)^\top + T_{22}(\tilde\bX\vbeta)^\top\right]\nonumber\\
    &\hspace{1cm}\times\bar\bR_0 
    \bF_0\left[\bSigma_{\vf}^0 +  \eta \left(\va (\bW_0\vbeta)^\top + \bW_0\vbeta \va^\top \right) +  \eta^2 \Vert\vbeta\Vert_2^2 \va\va^\top\right]\bF_0^\top \bar\bR_0\\ 
    &\hspace{2cm}\times\left[ T_{12}(\bF_0\va) + T_{22}(\tilde\bX\vbeta)\right] + o_{\sP}(1)\nonumber\\
    &= (\vf_\star^\top \bar\bR_0 \tilde\bX\vbeta)^2 \Big[T_{22}^2\left(\vbeta^\top\tilde\bX^\top\bar\bR_0 \bF_0 \bSigma_{\vf}^0 \bF_0^\top \bar\bR_0 \tilde\bX\vbeta\right)\nonumber\\
    &\hspace{0.6cm}+ 2\eta T_{12}T_{22}(\va^\top \bF_0^\top \bar\bR_0 \bF_0 \va)(\vbeta^\top \bW_0^\top \bF_0^\top \bar\bR_0 \tilde\bX\vbeta) + \Vert\vbeta\Vert_2^2\eta^2T_{12}^2(\va^\top\bF_0^\top\bar\bR_0\bF_0\va)^2\Big] + o_\sP(1).
\end{align}

\paragraph{Analysis of $E_{2}^{13}$ and $E_2^{31}$.} Recalling the definition of $K$ in \eqref{def:K}, this term can be written as
\begin{align}
    \label{eq:E2^13}
    E_{2}^{13} &= E_{2}^{31} =  \eta K \vf_\star^\top \bar\bR_0 \bF_0 \bSigma_\vf \va\nonumber\\
    &=  \eta K \vf_\star^\top\bar\bR_0\bF_0\left[\bSigma_{\vf}^0 +  \eta \left(\va (\bW_0\vbeta)^\top + \bW_0\vbeta \va^\top \right) +  \eta^2 \Vert\vbeta\Vert_2^2 \va\va^\top\right]\va\nonumber\\
    &=  \eta^2 K (\vf_\star^\top \bar\bR_0 \bF_0 \bW_0 \vbeta)  + o_\sP(1),
\end{align}
where the last line uses \eqref{eq:K_limit} and that $\eta\,\vf_\star^\top \bar\bR_0 \bF_0 \va = o_\sP(1)$.

\paragraph{Analysis of $E_{2}^{23}$ and $E_2^{32}$.} Again, recalling the definition of $K$ in \eqref{def:K}, we have
\begin{align*}
    E_{2}^{23} &= E_2^{32}
    =-\eta K\; \vf_\star^\top \bar\bR_0 \bU\bT\bU^\top \bar\bR_0 \bF_0 \bSigma_\vf \va\\
    &= -\eta K\; \vf_\star^\top \bar\bR_0 \left[T_{11} (\bF_0\va)(\bF_0\va)^\top + T_{12}(\bF_0\va)(\tilde\bX\vbeta)^\top + T_{21} (\tilde\bX\vbeta)(\bF_0\va)^\top+T_{22}(\tilde\bX\vbeta)(\tilde\bX\vbeta)^\top\right] \\
    &\hspace{3cm}\times  \left[\bar\bR_0 \bF_0\bSigma_{\vf}^0\va +  \eta \left(\vbeta^\top\bW_0^\top\va)\bar\bR_0 \bF_0\va + \bar\bR_0 \bF_0\bW_0\vbeta  \right) +  \eta^2 \Vert\vbeta\Vert_2^2 \bar\bR_0 \bF_0\va\right]\\
    &= E_2^{23(1)} + E_2^{23(2)} +E_2^{23(3)} +E_2^{23(4)},
\end{align*}
in which each term can be written as follows:
\begin{align*}
    E_2^{23(1)}&= - \eta K T_{11} (\vf_\star^\top \bar\bR_0 \bF_0 \va) \Big[\va^\top\bF_0^\top\bar\bR_0 \bF_0\bSigma_\vf^0 \va  + \eta (\va^\top \bW_0\vbeta)(\va^\top \bF_0^\top \bar\bR_0 \bF_0 \va)\\
    &\hspace{3cm}+\eta (\va^\top \bF_0^\top \bar\bR_0 \bF_0 \bW_0 \vbeta) +  \eta^2 \Vert\vbeta\Vert_2^2(\va^\top \bF_0 \bar\bR_0 \bF_0 \va)\Big] = o_\sP(1),
\end{align*}
where we have used that $K = \Theta_\sP(1/\eta^2)$, $\va^\top \bW_0\vbeta = o_\sP(1)$,  $\va^\top \bF_0^\top \bar\bR_0 \bF_0 \bW_0 \vbeta = o_\sP(1)$, and  $\eta\, \vf_\star^\top \bar\bR_0 \va = o_\sP(1)$. Also,
\begin{align*}
    E_2^{23(2)} &= - \eta K T_{12} (\vf_\star^\top \bar\bR_0 \bF_0 \va)\Big[ \vbeta^\top\tilde\bX^\top \bar\bR_0 \bF_0 \bSigma_\vf^0 \va +  \eta (\va^\top \bW_0 \vbeta)(\vbeta^\top\tilde\bX^\top \bar\bR_0 \bF_0 \va)\\
    &\hspace{3cm}+ \eta(\vbeta^\top \tilde\bX^\top \bar\bR_0 \bF_0 \bW_0 \vbeta) +  \eta^2 \Vert\vbeta\Vert_2^2(\vbeta^\top \tilde\bX^\top \bar\bR_0 \bF_0 \va)\Big] = o_\sP(1),
\end{align*}
with a very similar argument to that for $E_2^{23(1)}$. Next,
\begin{align*}
    E_2^{23(3)} &= -\eta K T_{21}(\vf_\star^\top\bar\bR_0\tilde\bX\vbeta)\Big[\va^\top\bF_0^\top \bar\bR_0 \bF_0 \bSigma_\vf^0 \va + \eta (\va^\top \bW_0 \vbeta)(\va^\top \bF_0^\top \bar\bR_0 \bF_0 \va)\\
    &\hspace{4.5cm}+ \eta(\va^\top\bF_0^\top \bar\bR_0 \bF_0 \bW_0 \vbeta) + \eta^2 \Vert\vbeta\Vert_2^2 (\va^\top \bF_0^\top \bar\bR_0 \bF_0 \va)\Big]\\
    &= -\eta^3 K T_{21} \Vert\vbeta\Vert_2^2(\vf_\star^\top\bar\bR_0 \tilde\bX\vbeta)(\va^\top \bF_0^\top \bar\bR_0 \bF_0 \va) + o_\sP(1),
\end{align*}
in which we have used the fact that $ T_{21} = \Theta_\sP(1/\eta)$, and $K = \Theta(1/\eta^2)$. Finally, with a similar argument
\begin{align*}
    E_2^{23(4)} &= - \eta K T_{22} (\vf_\star^\top \bar\bR_0 \tilde\bX\vbeta) \Big[\vbeta^\top \tilde\bX^\top\bar\bR_0 \bF_0 \bSigma_\vf^0 \va + \eta (\va^\top \bW_0\vbeta)(\vbeta^\top \tilde\bX^\top\bar\bR_0 \bF_0 \va)\\
    &\hspace{4.3cm}+\eta(\vbeta^\top \tilde\bX^\top \bar\bR_0 \bF_0 \bW_0 \vbeta) + \eta^2 \Vert\vbeta\Vert_2^2 (\vbeta^\top \tilde\bX^\top \bar\bR_0 \bF_0 \va)\Big]\\
    &= -  \eta^2 K T_{22}(\vf_\star^\top \bar\bR_0 \tilde\bX\vbeta)(\vbeta^\top \tilde\bX^\top \bar\bR_0 \bF_0 \bW_0 \vbeta) + o_\sP(1).
\end{align*}
Putting everything together, we have
\begin{align}
    \label{eq:E2^23}
    E_2^{23} = -\eta^3 K T_{21} \Vert\vbeta\Vert_2^2(\vf_\star^\top\bar\bR_0 &\tilde\bX\vbeta)(\va^\top \bF_0^\top \bar\bR_0 \bF_0 \va)\nonumber\\ &-  \eta^2 K T_{22}(\vf_\star^\top \bar\bR_0 \tilde\bX\vbeta)(\vbeta^\top \tilde\bX^\top \bar\bR_0 \bF_0 \bW_0 \vbeta) + o_\sP(1).
\end{align}

\paragraph{Analysis of $E_{2}^{33}$.} This term can be analyzed by expanding $\bSigma_\vf$ as follows:
\begin{align}
    \label{eq:E2^33}
    E_2^{33} &= \eta^2 K^2 \va^\top \left[\bSigma_{\vf}^0 +  \eta \left(\va (\bW_0\vbeta)^\top + \bW_0\vbeta \va^\top \right) +  \eta^2 \Vert\vbeta\Vert_2^2 \va\va^\top\right]\va\nonumber\\
    &= \eta^2 K^2 \va^\top\bSigma_\vf^0\va + 2 \eta^3 K^2 (\va^\top \bW_0 \vbeta) +  \eta^4K^2 \Vert\vbeta\Vert_2^2
    =  \eta^4K^2 \Vert\vbeta\Vert_2^2 +  o_\sP(1),
\end{align}
where we have used that $K = \Theta_\sP(1/\eta^2)$ and $\eta\,\va^\top \bW_0\vbeta  = o_\sP(1)$.

\paragraph{Putting Everything Together.} Now, we can put 
together
the results from previous sections to derive the limiting value of $E_2$. First, we will explicitly derive the limit of each component. 
To do so, recall that using \eqref{prop:alignment} and Lemma \ref{lemma:l1_limits}, we have 
\begin{align*}
    &\Vert\vbeta\Vert_2^2\to_P \left[\phi(c_\star^2 + \sigma_\ep^2) + c_{\star, 1}^2\right],\quad
    \vf_\star^\top \bar\bR_0 \tilde\bX\vbeta \to_P c_{\star, 1}^2 \psi m_2/\phi\\
    &\vbeta^\top \tilde\bX^\top \bar\bR_0 \tilde\bX\vbeta \to_P \left[\phi(c_\star^2 + \sigma_\ep^2) + c_{\star, 1}^2\right]\psi m_2/{\phi}, \quad \text{and}\quad
    \va^\top \bF_0^\top \bar\bR_0 \bF_0 \va \to_P \psi/\phi - \lambda \psi^2m_1/\phi^2.
\end{align*}
 
Also, using Lemma \ref{lemma:turn_beta_to_beta_star}, we have
\begin{align*}
    &\vf_\star^\top \bar\bR_0 \bF_0 \bSigma_\vf^0 \bF_0^\top \bar\bR_0 \tilde\bX \vbeta \to_P c_{\star, 1}^2 M,\quad \text{and}\quad
    \vbeta^\top \tilde\bX^\top \bar\bR_0 \bF_0 \bSigma_\vf^0 \bF_0^\top \bar\bR_0 \tilde\bX \vbeta \to_P  \left[\phi(c_\star^2 + \sigma_\ep^2) + c_{\star, 1}^2\right] M,
\end{align*}
in which 
$M := \lim_{n, N, d \to \infty} \vbeta_\star^\top \tilde\bX^\top \bar\bR_0 \bF_0 \bSigma_\vf^0 \bF_0^\top \bar\bR_0 \tilde\bX \vbeta_\star$.
This limit has 
been computed in \cite{adlam2020neural}. Using the diagram in the proof of Lemma~\ref{lemma:l1_limits} that shows how the notations of \cite{adlam2020neural} match ours, we find that 
$E_{32}$ in \cite[S148]{adlam2020neural} equals our $M$. 
Thus, we find
\begin{align}
    \label{eq:M_value}
    M = 1 - \frac{2m_2}{m1} -\frac{m_2'}{m_1^2},
\end{align}
where $m_2'$ is the derivative of $m_2$ with respect to $\lambda$.
Also, again using \ref{lemma:turn_beta_to_beta_star}, we have
\begin{align*}
    &\vf_\star^\top \bar\bR_0 \bF_0 \bW_0 \vbeta \to_P c_{\star,1}^2 \bar{M},\quad\text{and}\quad
    \vbeta^\top\tilde\bX^\top \bar\bR_0 \bF_0 \bW_0 \vbeta \to_P  \left[\phi(c_\star^2 + \sigma_\ep^2) + c_{\star, 1}^2\right] \bar{M},
\end{align*}
in which 
$\bar{M} := \lim_{n, N, d \to \infty} \vbeta^\top_\star\tilde\bX^\top \bar\bR_0 \bF_0 \bW_0 \vbeta_\star$.
This limit has been computed in \cite{ba2022high}. 
Specifically, using $T_9^0$ in (C.16), and noting that their $\mathbf{\Phi}_0$ translates to $\bF_0/\sqrt{N}$, 
their $\bW_0$ translates to $\bW_0^\top$, and their $\tilde\lambda$ translates to $\frac{\lambda n}{N}$ in our notation, using \cite[Proposition 29]{ba2022high} we find
\begin{align}
    \label{eq:M_bar}
    \bar{M} = 1 - \frac{m_2}{m_1}.
\end{align}

Now, we can use \eqref{eq:E2^12}, \eqref{eq:E2^22}, \eqref{eq:E2^13}, \eqref{eq:E2^23}, and \eqref{eq:E2^33},
 respectively, to derive the following expressions:

\begin{align*}
    &E_2^{21} = E_2^{12} = \frac{c_{\star,1}^4}{\phi(c_\star^2 + \sigma_\ep^2) + c_{\star,1}^2} \left[-M + \bar{M} \left(\frac{\psi/\phi - \lambda \psi^2 m_1 /\phi^2}{\psi/\phi - \lambda \psi^2 m_1 /\phi^2 - 1}\right)\right],\\[0.2cm]
    &E_2^{22} = \frac{c_{\star,1}^4}{\phi(c_\star^2 + \sigma_\ep^2) + c_{\star,1}^2} \left[M - 2\bar{M} \left(\frac{\psi/\phi - \lambda \psi^2 m_1 /\phi^2}{\psi/\phi - \lambda \psi^2 m_1 /\phi^2 - 1}\right) + \left(\frac{\psi/\phi - \lambda \psi^2 m_1 /\phi^2}{\psi/\phi - \lambda \psi^2 m_1 /\phi^2 - 1}\right)^2\right],\\[0.2cm]
    &E_2^{13} = E_2^{31} = -\frac{c_{\star,1}^4}{\phi(c_\star^2 + \sigma_\ep^2) + c_{\star,1}^2} \left[\frac{\bar M}{\psi/\phi + \lambda \psi^2/\phi^2 - 1}\right],\\[0.2cm]
    &E_2^{23} = E_2^{32} = \frac{c_{\star,1}^4}{\phi(c_\star^2 + \sigma_\ep^2) + c_{\star,1}^2} \left[-\frac{\psi/\phi - \lambda \psi^2 m_1 /\phi^2}{(\psi/\phi - \lambda \psi^2 m_1 /\phi^2 - 1)^2} + \frac{\bar M}{\psi/\phi - \lambda \psi^2 m_1 /\phi^2 - 1}\right],\\[0.2cm]
    &E_2^{33}  =\frac{c_{\star,1}^4}{\phi(c_\star^2 + \sigma_\ep^2) + c_{\star,1}^2} \left[\frac{1}{(\psi/\phi - \lambda\psi^2m_1/\phi^2-1)^2}\right].
\end{align*}

Thus, summing these terms up, we conclude that
\begin{align}
    \label{eq:E2}
    E_2 - \tilde\vy^\top \bar\bR_0 \bF\bSigma_{\vf}^0\bF_0^\top\bar\bR_0 \tilde\vy  \to_P \frac{c_{\star,1}^4\left(1 - M\right)}{\phi(c_\star^2 + \sigma_\ep^2) + c_{\star,1}^2},
\end{align}
wrapping up the derivation of the limiting value of $E_2$.

\subsubsection{Analysis of $E_3$.} To analyze this term, first note that
\begin{align*}
    E_3 &= -2 \hat \va^\top \vmu_\vf = -2 \tilde \vy^\top  \bar\bR \bF\vmu_\vf = -2 \vf_\star^\top  \bar\bR \bF\vmu_\vf - 2 \vepsilon^\top \bF \bar\bR \vmu_\vf.
\end{align*}
With a simple order-wise analysis, the second term can be shown to be $o_\sP(1).$ To analyze the first term, we will again use the decomposition from \eqref{eq:expand_a_hat}. 
We can write
$-2 \vf_\star^\top  \bar\bR \bF\vmu_\vf =  E_3^{(1)} + E_3^{(2)} + E_3^{(3)}$, 
in which $E_3^{(i)} = - 2 \vp_i^\top \vmu_\vf$. We will analyze these terms separately. For the first term, we have
\begin{align*}
    E_3^{(1)} &=  -2 \, \vp_1^\top \vmu_\vf = -2\, \vf_\star^\top \bar\bR_0 \bF_0  \left( c_{\star, 1} \bW_0 \vbeta_\star +  c_{\star,1}^2 \eta \,\va\right)
    = -2c_{\star, 1} \vf_\star^\top \bar\bR_0 \bF_0\bW_0\vbeta_\star + o_\sP(1),
\end{align*}
where we have used that  $\eta\, \vf_\star^\top\bar\bR_0 \bF_0 \va = o_\sP(1)$.  Similarly, the second term can be written as 
\begin{align*}
    E_3^{(2)} &= -2\, \vp_2^\top \vmu_\vf = 2\, \vf_\star^\top  \bar\bR_0 \bU\bT\bU^\top \bar\bR_0 \bF_0 \left( c_{\star, 1} \bW_0 \vbeta_\star +  c_{\star,1}^2 \eta \,\va\right)\\
     &= 2\, \vf_\star^\top  \bar\bR_0 \left[T_{11} (\bF_0\va)(\bF_0\va)^\top + T_{12}(\bF_0\va)(\tilde\bX\vbeta)^\top + T_{21} (\tilde\bX\vbeta)(\bF_0\va)^\top+T_{22}(\tilde\bX\vbeta)(\tilde\bX\vbeta)^\top\right]\\  &\hspace{1cm}\times\bar\bR_0 \bF_0 \left( c_{\star, 1} \bW_0 \vbeta_\star +  c_{\star,1}^2 \eta \,\va\right)\\
    &= 2\, (\vf_\star^\top  \bar\bR_0\tilde\bX\vbeta) \left[T_{21} (\bF_0\va)^\top+T_{22}(\tilde\bX\vbeta)^\top\right]  \bar\bR_0 \bF_0 \left( c_{\star, 1} \bW_0 \vbeta_\star +  c_{\star,1}^2 \eta \,\va\right) + o_\sP(1)\\
    &= 2\, (\vf_\star^\top  \bar\bR_0\tilde\bX\vbeta) \left[ c_{\star,1}^2 \eta T_{21} \va^\top\bF_0^\top \bar\bR_0 \bF_0 \va + c_{\star, 1}T_{22}\vbeta^\top\tilde\bX^\top\bar\bR_0 \bF_0 \bW_0\vbeta_\star\right]   + o_\sP(1).
\end{align*} 
Also, for the third term we can write
\begin{align*}
    E_3^{(3)}&= -2\, \vp_3^\top \vmu_\vf = -2   \eta K \va^\top \left( c_{\star, 1} \bW_0 \vbeta_\star +  c_{\star,1}^2 \eta \,\va\right) = -2 c_{\star,1}^2\eta^2 K + o_\sP(1).
\end{align*}
Summing up, the limiting value for $E_3$ is
\begin{align}
    \label{eq:E3}
    E_3  + 2 c_{\star, 1} \tilde\vy^\top \bF_0 \bar\bR_0 \bW_0 \vbeta_\star +  \to_P \frac{c_{\star,1}^4}{\phi(c_\star^2 + \sigma_\ep^2) + c_{\star,1}^2}\left[2 \bar M - 2\right].
\end{align}

\subsubsection{The Final Result}
Putting equations \ref{eq:E1}, \ref{eq:E2}, \ref{eq:E3} together, we have
\begin{align*}
    \gL_{\rm te}&(\hat\va(\bF)) - \Big(\sigma_\ep^2 + c_{\star}^2 + \tilde\vy^\top \bar\bR_0 \bF\bSigma_{\vf}^0\bF_0^\top\bar\bR_0 \tilde\vy  - 2 c_{\star, 1} \tilde\vy^\top \bF_0 \bar\bR_0 \bW_0 \vbeta_\star \Big)\\
    &\hspace{2cm}\to_P \frac{c_{\star,1}^4}{\phi(c_\star^2 + \sigma_\ep^2) + c_{\star,1}^2}\left[2 \bar M -1 - M\right]. 
\end{align*}
The test error of the untrained random features model can be written as
\begin{align*}
    \gL_{\rm te}(\hat\va(\bF_0)) = \sigma_\ep^2 + c_{\star}^2 &+ \tilde\vy^\top \bar\bR_0 \bF\bSigma_{\vf}^0\bF_0^\top\bar\bR_0 \tilde\vy  - 2 c_{\star, 1} \tilde\vy^\top \bF_0 \bar\bR_0 \bW_0 \vbeta_\star.
\end{align*}
Hence, the improvement over the untrained random features model in terms of test error is equal to
\begin{align*}
     \gL_{\rm te}(\hat\va(\bF_0)) - \gL_{\rm te}(\hat\va(\bF)) \to_P 
    &\frac{c_{\star,1}^4}{\phi(c_\star^2 + \sigma_\ep^2) + c_{\star,1}^2}\left[1  +M  - 2 \bar M\right]. 
\end{align*}
Hence, using \eqref{eq:M_bar} and \eqref{eq:M_value}, we find
\begin{align}
    \label{eq:final_l1_test}
    \gL_{\rm te}(\hat\va(\bF_0)) - \gL_{\rm te}(\hat\va(\bF)) \to_P \frac{-c_{\star,1}^4 }{(\phi(c_\star^2 + \sigma_\ep^2) + c_{\star,1}^2)m_1^2}\frac{\partial m_2}{\partial \lambda},
\end{align}
where $\frac{\partial m_1}{\partial \lambda} \leq 0$ using that $\tr(\tilde \bX^\top (\bF_0\bF_0^\top + \lambda n \bI_n)^{-1} \tilde \bX)/d =  \psi m_2/\phi$. This concludes the proof for $\ell = 1$.

\subsection{Proof for $\ell = 2$}
First, 
similar to the proof for $\ell = 1$, we have
\begin{align*}
    &\vnu_2 = \E_{\vx_{\rm te}}[(\vx_{\rm te}^\top\vbeta)^2 \sigma(\bW_0 \vx_{\rm te})] = 2 c_2(\bW_0 \vbeta)^{\circ 2},\\
    &\aleph_3 = \E_{\vx_{\rm te}}(\vx_{\rm te}^\top \vbeta)^3 = 0, \text{ and } \quad \aleph_4 = \E_{\vx_{\rm te}}(\vx_{\rm te}^\top \vbeta)^4 = 3\Vert\vbeta\Vert_2^2 \to_P 3[c_{\star, 1}^2 + \phi(c_{\star}^2 + \sigma_\ep^2)],\\
    &\tau_2 = \E\left[y_{\rm te} (\vx_{\rm te}^\top\vbeta)^2\right] = 2c_{\star, 2} (\vbeta_\star^\top \vbeta)^2 \to_P 2c_{\star,1}^2 c_{\star, 2},
\end{align*}
in which we have used Lemma~\ref{lem:twohermite}. Thus, using \eqref{eq:sigma_f} and \eqref{eq:mu_f}, we have
\begin{align*}
    \bSigma_\vf &= \bSigma_\vf^0 +  \eta \left(\va (\bW_0 \vbeta)^\top + (\bW_0 \vbeta) \va^\top\right) +  \frac{2 c_2^2\eta^2}{\sqrt{N}} \left((\sqrt{N}\va^{\circ 2}) (\bW_0 \vbeta)^{\circ 2 \top} + (\bW_0 \vbeta)^{\circ 2} (\sqrt{N}\va^{\circ 2\top})\right)\\
    &\hspace{1cm} +  \eta^2 \Vert\vbeta\Vert_2^2 \va \va^\top + \frac{3 c_2^2 \eta^4}{N} \Vert \vbeta\Vert_2^4 (\sqrt{N}\va^{\circ 2})(\sqrt{N}\va^{\circ 2 \top}),\quad\text{and}\quad \\
    \vmu_\vf &= \vmu_\vf^0 + c_{\star, 1}^2 \eta \va + 2 c_2 c_{\star, 1}^2 c_{\star, 2} \eta^2 \va^{\circ 2}.
\end{align*}
Also, from Section~\ref{pl2}, we have $\bar\bR = \bar\bR_0 - \bar\bR_0 \bU \bT \bU^\top \bar\bR_0$, where the matrix $\bT^{-1}$ is defined in \eqref{eq:inv_T_l2} and $\bU = [\,\bF_0\va\;|\;\bF_0 \va^{\circ 2} \sqrt{N}\;|\;\;\tilde\bX\vbeta\;|\;(\tilde\bX\vbeta)^{\circ 2}\;]$. Using the analysis in Section \ref{section:l2-Hinverse}, we note that 

\begin{align*}
 \bT^{-1} =
    \begin{bmatrix}
         \va^\top (\bF_0^\top \bar\bR_0 \bF_0 - \bI) \va & 0 & \frac{1}{c_1^2 \eta} & 0\\
       0& {N}\va^{\circ 2\top}(\bF_0^\top \bar\bR_0 \bF_0-\bI)\va^{\circ 2}&0&\frac{N^{\frac12}}{c_1^2c_2\eta^2}\\
         \frac{1}{c_1^2 \eta}&0& (\tilde\bX \vbeta)^\top\bar\bR_0(\tilde\bX \vbeta)& 0\\[0.5cm]
        0& \frac{N^{\frac12}}{c_1^2c_2\eta^2}&0&(\tilde\bX \vbeta)^{\circ 2 \top}\bar\bR_0(\tilde\bX \vbeta)^{\circ 2}
    \end{bmatrix} + \bDelta,
\end{align*}

in which the elements of $\bDelta$ are $O_\sP(1/\sqrt{n})$. Hence, the matrix $\bT = [T_{i,j}]$ has entries
\begin{align*}
    &T_{11} = \frac{1}{\va^\top (\bF_0^\top \bar\bR_0 \bF_0 - \bI) \va} + \frac{1}{\eta^2} \cdot \frac{1}{(\va^\top(\bF_0^\top \bar\bR_0 \bF_0 - \bI) \va)^2((\tilde\bX \vbeta)^{\top}\bar\bR_0(\tilde\bX \vbeta))} + o_\sP(1/\eta^2),\\
    &T_{13} = T_{31} = - \frac{1}{\eta}\cdot \frac{1}{(\va^\top (\bF_0^\top \bar\bR_0 \bF_0 - \bI) \va)((\tilde\bX \vbeta)^\top\bar\bR_0(\tilde\bX \vbeta))} + o_\sP(1/\eta),\\
    &T_{22} = \frac{1}{{N}\va^{\circ 2\top}(\bF_0^\top \bar\bR_0 \bF_0-\bI)\va^{\circ 2}} + \frac{N}{c_2^2\eta^4}\cdot\frac{1}{({N}\va^{\circ 2\top}(\bF_0^\top \bar\bR_0 \bF_0-\bI)\va^{\circ 2})^2((\tilde\bX \vbeta)^{\circ 2\top}\bar\bR_0(\tilde\bX \vbeta)^{\circ 2})} + o_\sP(N/\eta^4),\\
    &T_{24} = T_{42} = - \frac{\sqrt{N}}{c_2 \eta^2} \cdot \frac{1}{( {N}\va^{\circ 2\top}(\bF_0^\top \bar\bR_0 \bF_0-\bI)\va^{\circ 2})((\tilde\bX \vbeta)^{\circ 2 \top}\bar\bR_0(\tilde\bX \vbeta)^{\circ 2})}+o_\sP(1/\eta^2),\\
    &T_{33} = \frac{1}{(\tilde\bX\vbeta)^\top \bar\bR_0 (\tilde\bX\vbeta)} + \frac{1}{\eta^2} \cdot \frac{1}{(\va^{\top}(\bF_0^\top \bar\bR_0 \bF_0-\bI)\va)((\tilde\bX\vbeta)^\top \bar\bR_0 (\tilde\bX\vbeta))^2}+o_\sP(1/\eta^2),\\
    &T_{44} = \frac{1}{(\tilde\bX \vbeta)^{\circ 2 \top}\bar\bR_0(\tilde\bX \vbeta)^{\circ 2}} + \frac{N}{c_2^2\eta^4} \cdot \frac{1}{(N\va^{\circ 2\top} (\bF_0^\top \bar\bR_0 \bF_0 - \bI) \va^{\circ 2\top})((\tilde\bX \vbeta)^{\circ 2 \top}\bar\bR_0(\tilde\bX \vbeta)^{\circ 2})^2} + o_\sP(N/\eta^4),\\
\end{align*}
and its other elements are $O_\sP(1/\sqrt{n})$.

Next, we will study the terms $E_1 = \E_{\vx_{\rm te}, y_{\rm te}}(y_{\rm te}^2) $, $E_2 =  \hat\va^\top \bSigma_\vf \hat\va$, and $E_3 = - 2\hat\va^\top \vmu_\vf$ 
that appear in the decomposition of the test error in \eqref{eq:test_decomposition}.

\subsubsection{Analysis of $E_1$} 
Similar to the $\ell = 1$ case, we have
\begin{align}
    \label{eq:E1_l2}
    E_1 = \E(y_{\rm te}^2) = \sigma_\ep^2 + \E_{\vx_{\rm te}}\left(\sigma_\star(\vbeta_\star^\top\vx_{\rm te})\right)^2 = \sigma_\ep^2 + c_{\star}^2.
\end{align}

\subsubsection{Analysis of $E_2$}
Recall that using \eqref{eq:E2-decomp}, we have
\begin{align*}
    E_2 =\vf_\star^\top \bar\bR \bF\bSigma_{\vf}\bF^\top\bar\bR \vf_\star +  2\vepsilon^\top \bar\bR \bF\bSigma_{\vf}\bF^\top\bar\bR \tilde\vf_\star+\vepsilon^\top \bar\bR \bF\bSigma_{\vf}\bF^\top\bar\bR \vepsilon.
\end{align*}
First, note that $\vepsilon^\top \bar\bR \bF\bSigma_{\vf}\bF^\top\bar\bR \vf_\star = o_\sP(1)$ using a simple order-wise argument. 
Using an argument similar to the one for $\ell = 1$, 
the third term can be written as $\vepsilon^\top \bar\bR \bF\bSigma_{\vf}\bF^\top\bar\bR\, \vepsilon = \vepsilon^\top \bar\bR_0 \bF_0\bSigma_{\vf}^0\bF_0^\top\bar\bR_0\, \vepsilon + o_\sP(1).$

Next, we will study the term $\vf_\star^\top \bar\bR \bF\bSigma_{\vf}\bF^\top\bar\bR \vf_\star$. We can write
\begin{align}
    \label{eq:expand_a_hat_l2}
    \bF^\top \bar\bR \vf_\star &= \left(\bF_0 +  \eta (\tilde \bX \vbeta)\va^\top + c_2 \eta^2 (\tilde \bX \vbeta)^{\circ 2}\va^{\circ 2\top}\right)^\top\left(\bar\bR_0 - \bar\bR_0 \bU \bT \bU^\top \bar\bR_0\right)\vf_\star\nonumber\nonumber\\
    &= \underbrace{\bF_0^\top \bar\bR_0 \vf_\star}_{\vp_1} \underbrace{- \bF_0^\top \bar\bR_0 \bU \bT \bU^\top \bar\bR_0\vf_\star}_{\vp_2} + \underbrace{\eta K_1 \va}_{\vp_3} +\underbrace{c_2\eta^2 K_2 \va^{\circ 2}}_{\vp_4},
\end{align}
 in which $K_1$ and $K_2$ are defined as
 \begin{align*}
     &K_1 = \left[(\tilde \bX \vbeta)^\top\bar\bR_0\vf_\star - (\tilde \bX\vbeta)^\top \bar\bR_0 \bU \bT \bU^\top \bar\bR_0\vf_\star\right],\\
     &K_2 = \left[(\tilde \bX \vbeta)^{\circ 2\top}\bar\bR_0\vf_\star - (\tilde \bX\vbeta)^{\circ 2\top} \bar\bR_0 \bU \bT \bU^\top \bar\bR_0\vf_\star\right].
 \end{align*}
Usign this notation, we have
$\vf_\star^\top \bar\bR \bF\bSigma_{\vf}\bF^\top\bar\bR \vf_\star = \sum_{i, j \in [4]} E_2^{ij}$,
where $E_2^{ij} = \vp_i^\top \bSigma_\vf \vp_j^\top$. In the following sections, we will compute each term separately.

\paragraph{Preliminary Computations.} First, we will analyze $K_1$ and $K_2$. Recall that 
 \begin{align*}
     \bU\bT\bU^\top = T_{11} (\bF_0\va)&(\bF_0\va)^\top + T_{13} (\bF_0\va)(\tilde \bX \vbeta)^\top \\ &+ N T_{22}\left(\bF_0\va^{\circ 2}\right)\left(\bF_0\va^{\circ 2}\right)^\top + \sqrt{N}T_{24}\left(\bF_0\va^{\circ 2}\right)(\tilde \bX \vbeta)^{\circ 2 \top}\\
     &+ T_{31}  (\tilde \bX \vbeta)(\bF_0\va)^\top + T_{33}  (\tilde \bX \vbeta)(\tilde \bX \vbeta)^\top + \sqrt{N}T_{42}(\tilde \bX \vbeta)^{\circ 2}(\bF_0\va^{\circ 2 \top})\\ &+ T_{44}(\tilde \bX \vbeta)^{\circ 2}(\tilde \bX \vbeta)^{\circ 2 \top}.
 \end{align*}
 Thus, we have
 \begin{align*}
     (\tilde \bX\vbeta)^\top \bar\bR_0 \bU \bT \bU^\top \bar\bR_0\vf_\star &= T_{33} \left[(\tilde \bX\vbeta)^\top \bar\bR_0 (\tilde \bX \vbeta)\right]\cdot\left[(\tilde \bX\vbeta)^\top \bar\bR_0 \vf_\star\right] + O_\sP(1/\sqrt{n}),
 \end{align*}
 which gives
 \begin{align}
    \label{eq:K1}
     K_1 = -\frac{1}{\eta^2} \cdot \frac{\vbeta^\top \tilde\bX^\top \bar\bR_0^\top \vf_\star}{(\va^{\top}(\bF_0^\top \bar\bR_0 \bF_0-\bI)\va)((\tilde\bX\vbeta)^\top \bar\bR_0 (\tilde\bX\vbeta))} + o_\sP(1/\eta^2).
 \end{align}
 Similarly, for $K_2$ we have 
  \begin{align*}
     (\tilde \bX\vbeta)^{\circ 2 \top} \bar\bR_0 \bU \bT \bU^\top \bar\bR_0\vf_\star &= T_{44} \left[(\tilde \bX\vbeta)^{\circ 2 \top} \bar\bR_0 (\tilde \bX \vbeta)^{\circ 2}\right]\cdot\left[(\tilde \bX\vbeta)^{\circ 2\top} \bar\bR_0 \vf_\star\right] + O_\sP(1/\sqrt{n}),
 \end{align*}
 which gives
 \begin{align}
    \label{eq:K2}
     K_2 = -\frac{N}{c_2^2\eta^4} \cdot \frac{(\tilde\bX\vbeta)^{\circ 2\top} \bar\bR_0^\top \vf_\star}{(N\va^{\circ 2\top} (\bF_0^\top \bar\bR_0 \bF_0 - \bI) \va^{\circ 2\top})((\tilde\bX \vbeta)^{\circ 2 \top}\bar\bR_0(\tilde\bX \vbeta)^{\circ 2})} + o_\sP(1/\eta^4).
 \end{align}
\paragraph{Analysis of $E_2^{11}$.} For this term, by expanding $\bSigma_\vf$, we can write
\begin{align}
    \label{eq:E2_11_l2}
    E_2^{11} &= \vf_\star^\top \bar\bR_0 \bF_0 \bSigma_\vf \bF_0^\top\bar\bR_0 \vf_\star = \vf_\star^\top \bar\bR_0 \bF_0 \bSigma_\vf^0 \bF_0^\top\bar\bR_0 \vf_\star + o_\sP(1).
\end{align}
This holds because $\eta \vf_\star^\top \bar\bR_0 \bF_0 \va = o_\sP(1)$, and $\frac{\eta^2}{\sqrt{N}} \vf_\star^\top \bar\bR_0  \bF_0 (\sqrt{N}\va^{\circ 2}) = o_\sP(1)$.
\paragraph{Analysis of $E_2^{12}$ and $E_2^{21}$.} We have
$E_2^{21} = E_2^{12} = -\vf_\star^\top \bar\bR_0 \bF_0 \bSigma_\vf \bF_0^\top \bar\bR_0 \bU \bT \bU^\top \bar\bR_0\vf_\star$.
Using the expression for $\bT$, we can write
\begin{align}
    \label{eq:TURf}
    \bT\bU^\top \bar\bR_0 \vf_\star = \begin{bmatrix}
        T_{11} \va^\top \bF_0^\top \bar\bR_0 \vf_\star + T_{13} (\tilde \bX\vbeta)^\top \bar\bR_0 \vf_\star\\[0.2cm]
        T_{22} \sqrt{N} \va^{\circ 2 \top}\bF_0^\top \bar\bR_0 \vf_\star + T_{24} (\tilde \bX \vbeta)^{\circ 2 \top} \bar\bR_0 \vf_\star\\[0.2cm]
        T_{31} \va^\top \bF_0^\top \bar\bR_0 \vf_\star + T_{33} (\tilde \bX\vbeta)^\top \bar\bR_0 \vf_\star\\[0.2cm]
        T_{42} \sqrt{N} \va^{\circ 2 \top} \bF_0^\top \bar\bR_0 \vf_\star + T_{44} (\tilde \bX \vbeta)^{\circ 2 \top} \bar\bR_0 \vf_\star
    \end{bmatrix}.
\end{align}
On the other hand, by expanding $\bU$ and $\bSigma_\vf$, we similarly have
\begin{align}
    \label{eq:fRFSFRU}
    (\vf_\star^\top \bar\bR_0 \bF_0 \bSigma_\vf \bF_0^\top \bar\bR_0 \bU) = \begin{bmatrix}
    \eta (\vf_\star^\top \bar\bR_0 \bF_0 \bW_0 \vbeta)(\va^\top \bF_0^\top \bar\bR_0 \bF_0 \va)\\[0.2cm]
    \frac{2c_2^2\eta^2}{\sqrt{N}}(\vf_\star^\top \bar\bR_0 \bF_0 (\bW_0 \vbeta)^{\circ 2})((\sqrt{N} \va^{\circ 2})^{\top}\bF_0^\top \bar\bR_0 \bF_0(\sqrt{N}\va^{\circ 2}))\\[0.2cm]
    \vf_\star^\top\bar\bR_0\bF_0\bSigma_\vf^0 \bF_0^\top \bar\bR_0 \tilde \bX\vbeta\\[0.2cm]
    \vf_\star^\top \bar\bR_0 \bF_0 \bSigma_\vf^0 \bF_0^\top \bar\bR_0 (\tilde \bX\vbeta)^{\circ 2}
    \end{bmatrix}^\top.
\end{align}

This gives
\begin{align}
    \label{eq:E2_21_l2}
    E_2^{21} = E_2^{12} = &- T_{13}\eta(\vf_\star \bar\bR_0 \bF_0 \bW_0 \vbeta)(\va^\top \bF_0^\top \bar\bR_0 \bF_0 \va)((\tilde \bX\vbeta)^\top \bar\bR_0 \vf_\star)\nonumber\\[0.1cm]
    &-\frac{2c_2^2\eta^2 T_{24}}{\sqrt{N}}(\vf_\star^\top \bar\bR_0 \bF_0 (\bW_0 \vbeta)^{\circ 2})((\sqrt{N} \va^{\circ 2})^{\top}\bF_0^\top \bar\bR_0 \bF_0(\sqrt{N}\va^{\circ 2}))((\tilde \bX \vbeta)^{\circ 2 \top} \bar\bR_0 \vf_\star)\nonumber\\[0.1cm]
    &- T_{33}(\vf_\star\bar\bR_0\bF_0\bSigma_\vf^0 \bF_0^\top \bar\bR_0 \tilde \bX\vbeta)((\tilde \bX\vbeta)^\top \bar\bR_0 \vf_\star)\nonumber\\[0.2cm]
    &- T_{44} (\vf_\star^\top \bar\bR_0 \bF_0 \bSigma_\vf^0 \bF_0^\top \bar\bR_0 (\tilde \bX\vbeta)^{\circ 2})((\tilde \bX \vbeta)^{\circ 2 \top} \bar\bR_0 \vf_\star) + o_\sP(1).
\end{align}

\paragraph{Analysis of $E_2^{13}$ and $E_2^{31}$.} Recalling \eqref{eq:K1}, by expanding $\bSigma_\vf$ we have
\begin{align}
    \label{eq:E2_13_l2}
    E_2^{13} = E_2^{31} &= \eta K_1 \vf_\star^\top \bar\bR_0 \bF_0\bSigma_\vf\va
    = \eta^2 K_1 (\vf_\star^\top \bar\bR_0 \bF_0 \bW_0\vbeta) + o_\sP(1),
\end{align}
in which we have used that $K_1 = O_\sP(1/\eta^2)$.

\paragraph{Analysis of $E_2^{14}$ and $E_2^{41}$.} Recalling \eqref{eq:K2}, by expanding $\bSigma_\vf$ we have
\begin{align}
    \label{eq:E2_14_l2}
    E_2^{14} = E_2^{41} &= c_2 \eta^2 K_2 \vf_\star^\top \bar\bR_0 \bF_0\bSigma_\vf\va^{\circ 2}  = o_\sP(1),
\end{align}
in which we have used that $\eta^4 K_2/N = O_\sP(1)$.
\paragraph{Analysis of $E_2^{22}$.} This term is equal to
$E_{2}^{22} = \vf_\star^\top \bar\bR_0 \bU\bT \bU^\top \bar\bR_0 \bF_0\bSigma_\vf \bF_0^\top \bar\bR_0 \bU \bT \bU^\top \bar\bR_0\vf_\star$.
Using \eqref{eq:TURf}, we can write
\begin{align}
    \label{eq:FRUTURf}
    \bF_0^\top \bar\bR_0 \bU \bT \bU^\top \bar\bR_0\vf_\star &= [T_{11} \va^\top \bF_0^\top \bar\bR_0 \vf_\star + T_{13} (\tilde \bX\vbeta)^\top \bar\bR_0 \vf_\star]\; \bF_0^\top \bar\bR_0 \bF_0 \va\nonumber\\
    &\hspace{2cm} + [T_{22} \sqrt{N} \va^{\circ 2 \top}\bF_0^\top \bar\bR_0 \vf_\star + T_{24} (\tilde \bX \vbeta)^{\circ 2 \top} \bar\bR_0 \vf_\star]\; \bF_0^\top \bar\bR_0 \bF_0 (\sqrt{N}\va^{\circ 2})\nonumber\\
    &\hspace{2cm} + [T_{31} \va^\top \bF_0^\top \bar\bR_0 \vf_\star + T_{33} (\tilde \bX\vbeta)^\top \bar\bR_0 \vf_\star]\; \bF_0^\top \bar\bR_0\tilde \bX \vbeta\nonumber\\
    &\hspace{2cm} + [T_{42} \sqrt{N} \va^{\circ 2 \top} \bF_0^\top \bar\bR_0 \vf_\star + T_{44} (\tilde \bX \vbeta)^{\circ 2 \top} \bar\bR_0 \vf_\star]\; \bF_0^\top \bar\bR_0   (\tilde\bX \vbeta)^{\circ 2}.
\end{align}
Further, by expanding $\bSigma_\vf$, we have
\begin{align}
    \label{eq:E2_22_l2}
    E_{2}^{22} = E_2^{22(1)} +E_2^{22(2)} +E_2^{22(3)}+E_2^{22(4)} + o_\sP(1),
\end{align}
in which
\begin{align*}
    E_2^{22(1)} &= [T_{33} (\tilde \bX\vbeta)^\top \bar\bR_0 \vf_\star]^2\; (\vbeta^\top \tilde \bX^\top \bar\bR_0 \bF_0 \bSigma_\vf^0 \bF_0^\top \bar\bR_0\tilde \bX \vbeta)\\ 
    &\hspace{2cm}+ [T_{44} (\tilde \bX \vbeta)^{\circ 2 \top} \bar\bR_0 \vf_\star]^2\;  ( (\tilde \bX\vbeta)^{\circ 2 \top} \bar\bR_0 \bF_0 \bSigma_\vf^0 \bF_0^\top \bar\bR_0(\tilde\bX \vbeta)^{\circ 2}),\\[0.2cm]
    E_2^{22(2)} &= 2\eta T_{13} T_{33}\cdot ( (\tilde \bX\vbeta)^\top \bar\bR_0 \vf_\star)^2 (\va^\top\bF_0^\top \bar\bR_0 \bF_0 \va)(\vbeta^\top \bW_0^\top\bF_0^\top \bar\bR_0\tilde \bX \vbeta),\\[0.2cm]
    E_2^{22(3)} &= \frac{4c_2^2\eta^2}{\sqrt{N}} T_{24}T_{44} ((\tilde \bX \vbeta)^{\circ 2 \top} \bar\bR_0 \vf_\star)^2  \left((\sqrt{N}\va^{\circ 2})^\top \bF_0^\top \bar\bR_0 \bF_0 (\sqrt{N}\va^{\circ 2})\right) \left((\bW_0\vbeta)^{\circ 2}\bF_0^\top \bar\bR_0 (\tilde\bX\vbeta)^{\circ 2}\right),\\[0.2cm]
    E_2^{22(4)} &= \eta^2 T_{13}^2 \Vert\vbeta\Vert_2^2 \; ((\tilde \bX\vbeta)^\top \bar\bR_0 \vf_\star)^2 \, (\va^\top \bF_0^\top \bar\bR_0 \bF_0 \va)^2\\ 
    &\hspace{2cm}+ \frac{3c_2^2\eta^4}{N} T_{24}^2 \Vert\vbeta\Vert_2^4 \left((\tilde \bX \vbeta)^{\circ 2 \top} \bar\bR_0 \vf_\star\right)^2 \, \left((\sqrt{N}\va^{\circ 2})^\top \bF_0^\top \bar\bR_0 \bF_0 (\sqrt{N}\va^{\circ 2})\right)^2.
\end{align*}

\paragraph{Analysis of $E_2^{23}$ and $E_2^{32}$.} We have
$E_2^{23} = E_2^{32} = -\eta K_1 \va^\top \bSigma_\vf \bF_0^\top \bar\bR_0 \bU\bT\bU^\top \bar\bR_0 \vf_\star$.
Recall that the vector $\bF_0^\top \bar\bR_0 \bU\bT\bU^\top \bar\bR_0 \vf_\star$ has been computed in \eqref{eq:FRUTURf}. 
With this, we have
\begin{align}
    \label{eq:E2_23_l2}
    E_2^{23} = E_2^{32} = &-  \eta^2 K_1 T_{33}  (\vbeta^\top \tilde\bX^\top \bar\bR_0 \vf_\star)(\vbeta^\top \bW_0^\top \bF_0^\top \bar\bR_0 \tilde \bX \vbeta)\nonumber\\[0.1cm] &- \eta^3 K_1 T_{13} \Vert\vbeta\Vert_2^2 (\va^\top \bF_0^\top \bar\bR_0 \bF_0 \va)(\vbeta^\top \tilde\bX^\top \bar\bR_0 \vf_\star) + o_\sP(1),
\end{align}
in which we used that $\eta^2K_1 = O_\sP(1)$ and $\eta T_{13} = O_\sP(1)$.
\paragraph{Analysis of $E_2^{24}$ and $E_2^{42}$.} This term can be written as
\begin{align*}
E_{24} = E_{42} =  -c_2 \eta^2 K_2  \va^{\circ 2 \top}\bSigma_\vf \bF_0^\top \bar\bR_0 \bU\bT\bU^\top \bar\bR_0 \vf_\star.    
\end{align*}

The vector $\bF_0^\top \bar\bR_0 \bU\bT\bU^\top \bar\bR_0 \vf_\star$ has been computed in \eqref{eq:FRUTURf}. By expanding $\bSigma_\vf$, we have
\begin{align*}
    &E_2^{24} = E_2^{42} = -\frac{2c_2^3 \eta^4 K_2 T_{44}}{N} \cdot \left((\bW_0\vbeta)^{\circ 2 \top}\bF_0^\top \bar\bR_0 (\tilde \bX\vbeta)^{\circ 2}\right)\cdot\left((\tilde \bX \vbeta)^{\circ 2 \top} \bar\bR_0 \vf_\star\right) \\&\hspace{2cm}- \frac{3c_2^3 \eta^6 K_2 T_{24} \Vert\vbeta\Vert_2^4}{N^{3/2}}\left((\sqrt{N}\va^{\circ 2})^\top \bar\bF_0^\top \bar\bR_0 \bF_0 (\sqrt{N}\va^{\circ 2})\right)\cdot \left((\tilde \bX\vbeta)^{\circ 2 \top}\bar\bR_0 \vf_\star\right) + o_\sP(1).
\end{align*}
Now, noting that $K_2 = O_\sP(N/\eta^4)$ from \eqref{eq:K2} and $\Vert(\bW_0\vbeta)^{\circ 2}\Vert_2 = O_\sP(1/\sqrt{N})$, we find that the first term is $o_\sP(1)$ and we have
\begin{align}
    \label{eq:E2_24_l2}
    &E_2^{24} = E_2^{42} = - \frac{3c_2^3 \eta^6 K_2 T_{24} \Vert\vbeta\Vert_2^4}{N^{3/2}}\left((\sqrt{N}\va^{\circ 2})^\top \bar\bF_0^\top \bar\bR_0 \bF_0 (\sqrt{N}\va^{\circ 2})\right)\cdot \left((\tilde \bX\vbeta)^{\circ 2 \top}\bar\bR_0 \vf_\star\right) + o_\sP(1).
\end{align}

\paragraph{Analysis of $E_2^{33}$.} Similar to the $\ell = 1$ case, this term can be written as
\begin{align}
    \label{eq:E2_33_ell2}
    E_2^{33} &= \eta^2 K_1^2 \va^\top \bSigma_\vf\va = \eta^4 K_1^2 \Vert\vbeta\Vert_2^2 +  o_\sP(1),
\end{align}
noting that $K_1 = O_\sP(1/\eta^2)$.

\paragraph{Analysis of $E_2^{34}$ and $E_2^{43}$.} By expanding $\bSigma_\vf$, we readily arrive at
\begin{align}
    \label{eq:E2_34_ell2}
    E_2^{34} = E_2^{43} = \eta^3 K_1 K_2 c_2 \va^\top \bSigma_\vf \va^{\circ 2} = o_\sP(1).
\end{align}
\paragraph{Analysis of $E_2^{44}$.} This term can be written as
$E_2^{44} = c_2^2 \eta^4 K_2^2 \va^{\circ 2 \top} \bSigma_\vf \va^{\circ 2}$.
By expanding $\bSigma_\vf$ and noting that $K_2 = O_\sP(N/\eta^4)$, we can write
\begin{align}
    \label{eq:E2_44_ell2}
    E_2^{44} = \frac{3c_2^4 \eta^8 K_2^2}{N^2} + o_\sP(1).
\end{align}
\paragraph{Putting Everything Together.}  Now, we can use the results derived above to compute the limiting value of $E_2$. Recall that from Lemma~\ref{lemma:s12} and Lemma~\ref{t2rt2} we have
\begin{align*}
    H_2(\tilde\bX\vbeta_\star)^\top \bar\bR_0(\tilde\bX\vbeta)^{\circ 2} \to_P 2c_{\star,1}^2 \frac{\psi m_1}{\phi}, \quad \text{and} \quad (\tilde \bX \vbeta)^{\circ 2 \top}\bar\bR_0(\tilde \bX \vbeta)^{\circ 2} 
    \to_P \frac{3 \psi m_1}{\phi}[\phi(c_\star^2 + \sigma_\ep^2) + c_{\star,1}^2]^2.
\end{align*}
Also using an argument similar to the argument in the proof of Lemma \ref{lemma:s12} and Lemma \ref{t2rt2}, we have
\begin{align*}
    &H_2(\tilde \bX\vbeta_\star)^\top \bar\bR_0 \bF_0 \bSigma_\vf^0 \bF_0^\top \bar\bR_0 (\tilde \bX\vbeta)^{\circ 2} \to_p 2c_{\star,1}^2\hat{M},\\[0.1cm]
    &(\tilde \bX\vbeta)^{\circ 2 \top} \bar\bR_0 \bF_0 \bSigma_\vf^0 \bF_0^\top \bar\bR_0 (\tilde \bX\vbeta)^{\circ 2} \to 3\hat M [\phi(c_\star^2 + \sigma_\ep^2) + c_{\star,1}^2]^2,
\end{align*}
in which
$\hat{M} = \lim_{n,N,d \to \infty} \frac{1}{n}\tr\left[\bar\bR_0 \bF_0 \bSigma_\vf^0 \bF_0^\top \bar\bR_0\right]$.
This term has been computed in \cite{adlam2020neural}.
Using the diagram in the proof of Lemma~\ref{lemma:l1_limits} that shows how the notations of \cite{adlam2020neural} match ours, we find that we can use (S142) with $\sigma_\ep = 0$ in \cite{adlam2020neural}, to find
\begin{align}
    \label{eq:M_hat}
    \hat{M} = - \frac{m_1'}{m_1^2} - 1,
\end{align}
where $m_1'$ is the derivative of $m_1$ with respect to $\lambda$.

For brevity, we will define
$\tilde A := \lim_{n,N,d \to \infty} \left[(\sqrt{N} \va^{\circ 2\top}) \bF_0^\top \bar\bR_0 \bF_0 (\sqrt{N} \va^{\circ 2})\right]$.
With this, equations~\ref{eq:E2_21_l2}, \ref{eq:E2_13_l2}, \ref{eq:E2_14_l2}, \ref{eq:E2_22_l2}, \ref{eq:E2_23_l2}, \ref{eq:E2_24_l2}, \ref{eq:E2_33_ell2}, \ref{eq:E2_34_ell2}, and \ref{eq:E2_44_ell2} give
\begin{align*}
    &E_2^{12} = E_2^{21} \to_P \frac{c_{\star,1}^4}{\phi(c_\star^2 + \sigma_\ep^2) + c_{\star,1}^2}\left[\bar M \left(\frac{\psi/\phi - \lambda \psi^2 m_1/ \phi^2}{\psi/\phi - \lambda \psi^2 m_1/ \phi^2 - 1}\right) - M\right] - \frac{4c_{\star,1}^4c_{\star,2}^2 \hat M}{3[\phi(c_\star^2 + \sigma_\ep^2) + c_{\star,1}^2]^2},\\[0.2cm]
    &E_2^{13} = E_2^{31} \to_P -\frac{c_{\star,1}^4}{\phi(c_\star^2 + \sigma_\ep^2) + c_{\star,1}^2}\cdot \frac{\bar M}{\psi/\phi - \lambda \psi^2 m_1/ \phi^2 - 1},\quad
    E_2^{14} = E_2^{41} \to_P 0,
\end{align*}
as well as 
\begin{align*}
    &E_2^{22} \to_P \frac{c_{\star,1}^4}{\phi(c_\star^2 + \sigma_\ep^2) + c_{\star,1}^2} \left[ M  - 2\bar{M} \left(\frac{\psi/\phi - \lambda \psi^2 m_1/ \phi^2}{\psi/\phi - \lambda \psi^2 m_1/ \phi^2-1}\right) + \left(\frac{\psi/\phi - \lambda \psi^2 m_1/ \phi^2}{\psi/\phi - \lambda \psi^2 m_1/ \phi^2-1}\right)^2 \right]\\ 
    &\hspace{2cm}+ \frac{c_{\star,1}^4c_{\star,2}^2 }{[\phi(c_\star^2 + \sigma_\ep^2) + c_{\star,1}^2]^2}\left[4 {\hat M/3} + \frac{4\tilde A^2}{3(\tilde A - 1)^2}\right],\\[0.2cm]
    &E_2^{23} = E_2^{32} \to_P \frac{c_{\star,1}^4}{\phi(c_\star^2 + \sigma_\ep^2) + c_{\star,1}^2} \left[ \frac{\bar M}{\psi/\phi - \lambda \psi^2 m_1/ \phi^2 - 1} - \frac{\psi/\phi - \lambda \psi^2 m_1/ \phi^2}{(\psi/\phi - \lambda \psi^2 m_1/ \phi^2 - 1)^2}\right],\\[0.2cm]
    \end{align*}
and also \begin{align*}
    &E_2^{24} = E_2^{42} \to_P -\frac{c_{\star,1}^4c_{\star,2}^2 }{[\phi(c_\star^2 + \sigma_\ep^2) + c_{\star,1}^2]^2} \left[\frac{4\tilde A}{3(\tilde A - 1)^2}\right],\\
    &E_2^{33} \to_P \frac{c_{\star,1}^4}{\phi(c_\star^2 + \sigma_\ep^2) + c_{\star,1}^2} \left[\frac{1}{(\psi/\phi - \lambda \psi^2 m_1/ \phi^2 - 1)^2}\right],\\[0.2cm]
    &E_2^{34} = E_2^{43} \to_P 0,\quad
    E_2^{44} = E_2^{44} \to_P \frac{c_{\star,1}^4c_{\star,2}^2 }{[\phi(c_\star^2 + \sigma_\ep^2) + c_{\star,1}^2]^2} \left[\frac{4}{3(\tilde A - 1)^2}\right],
\end{align*}
respectively. Putting these together, the component $E_2$ can be written as
\begin{align}
    \label{eq:E2_ell2}
    E_2 - \tilde\vy^\top \bar\bR_0 \bF\bSigma_{\vf}^0\bF_0^\top\bar\bR_0 \tilde\vy \to_P \frac{c_{\star,1}^4 (1-M)}{\phi(c_\star^2 + \sigma_\ep^2) + c_{\star,1}^2} + \frac{4c_{\star,1}^4 c_{\star,2}^2 ( 1 - \hat M)}{3[\phi(c_\star^2 + \sigma_\ep^2) + c_{\star,1}^2]^2}.
\end{align}

\subsubsection{Analysis of $E_3$.}
To analyze this component, first note that
\begin{align*}
    E_3 &= -2 \hat \va^\top \vmu_\vf = -2 \tilde \vy^\top  \bar\bR \bF\vmu_\vf = -2 \vf_\star^\top  \bar\bR\bF \vmu_\vf - 2 \vepsilon^\top \bar\bR\bF \vmu_\vf.
\end{align*}
Using a simple order-wise analysis, it can be shown that the second term is $o_\sP(1)$. Now, recalling \eqref{eq:expand_a_hat_l2} we write $E_3 =  \sum_{i = 1}^{4} E_3^{(i)}$, where $E_3^{(i)} = -2\vp_i^\top \vmu_\vf$. By expanding $\vmu_\vf$ and recalling that $\vmu_\vf^0 = c_{\star,1}\bW_0\vbeta$, we have
\begin{align}
    \label{eq:E3_1_l2} 
    E_3^{(1)} = -2\vf_\star^\top \bar\bR_0 \bF_0 \vmu_\vf^0 + o_\sP(1).
\end{align}
Next, recall that the matrix $\bF_0^\top \bar\bR_0 \bU\bT\bU^\top \bar\bR_0 \vf_\star$ is analyzed in \eqref{eq:FRUTURf}. Using this, and by expanding $\vmu_\vf$, we get
\begin{align}
    \label{eq:E3_2_l2}
    E_3^{(2)} &= 2(\vf_\star^\top \bar\bR_0 \tilde \bX\vbeta)\left[T_{33} ( \vmu_\vf^{0 \top} \bF_0^\top \bar\bR_0 \tilde \bX\vbeta) + c_{\star,1}^2 T_{13}\eta (\va^\top \bF_0^\top \bar\bR_0 \bF_0 \va)\right]\nonumber\\
    &\hspace{1.3cm} +\frac{4 c_2 c_{\star,1}^2 c_{\star,2} T_{24} \eta^2}{\sqrt{N}} \left[(\sqrt{N} \va^{\circ 2})^\top \bF_0^\top \bar\bR_0 \bF_0 (\sqrt{N} \va^{\circ 2})\right] 
    \left((\tilde \bX\vbeta)^{\circ 2\top} \bar\bR_0 \vf_\star\right) + o_\sP(1).
\end{align}
Similarly, by expanding $\vmu_\vf$, we arrive at
\begin{align*}
    E_3^{(3)} = - 2c_{\star,1}^2 \eta^2 K_1 + o_\sP(1), \quad \text{and}\quad E_{3}^{(4)} = -\frac{4c_2^2c_{\star,1}^2 c_{\star,2} \eta^4 K_2 }{N} + o_\sP(1).
\end{align*}
Similar to the computation for $E_2$, we can derive the limiting values of the components in $E_3$ as 
\begin{align*}
    &E_3^{(2)} \to_P \frac{c_{\star, 1}^4}{\phi(c_\star^2 + \sigma_\ep^2) + c_{\star,1}^2} \left(2 \bar M - \frac{2 \tilde A}{\tilde A - 1}\right) - \frac{c_{\star, 1}^4 c_{\star,2}^2}{[\phi(c_\star^2 + \sigma_\ep^2) + c_{\star,1}^2]^2} \left(\frac{8 \tilde A}{3(\tilde A - 1)}\right),\\[0.2cm]
    &E_3^{(3)} \to_P \frac{c_{\star, 1}^4}{\phi(c_\star^2 + \sigma_\ep^2) + c_{\star,1}^2} \left(\frac{2}{\tilde A - 1}\right), \quad
    E_3^{(4)} \to_P  \frac{c_{\star, 1}^4 c_{\star,2}^2}{[\phi(c_\star^2 + \sigma_\ep^2) + c_{\star,1}^2]^2} \left(\frac{8 }{3(\tilde A - 1)}\right).
\end{align*}
Putting these together, we have
\begin{align}
    \label{eq:E3_l2}
    E_3 +2\vf_\star^\top \bar\bR_0 \bF_0 \vmu_\vf^0 \to_P \frac{c_{\star, 1}^4 \left(2 \bar M - 2\right)}{\phi(c_\star^2 + \sigma_\ep^2) + c_{\star,1}^2}  - \frac{8c_{\star, 1}^4 c_{\star,2}^2}{3[\phi(c_\star^2 + \sigma_\ep^2) + c_{\star,1}^2]^2} .
\end{align}
\subsubsection{The Final Result}
Putting equations \ref{eq:E1_l2}, \ref{eq:E2_ell2}, and \ref{eq:E3_l2}, we have
\begin{align*}
     \gL_{\rm te}(\hat\va(\bF_0)) - \gL_{\rm te}(\hat\va(\bF)) \to_P 
    &\frac{c_{\star,1}^4 (1  +M  - 2 \bar M)}{\phi(c_\star^2 + \sigma_\ep^2) + c_{\star,1}^2} + \frac{4c_{\star,1}^4 c_{\star,2}^2 ( 1 + \hat M )}{3[\phi(c_\star^2 + \sigma_\ep^2) + c_{\star,1}^2]^2},
\end{align*}
where $\gL_{\rm te}(\hat\va(\bF_0))$ is the test error of the untrained random feature model. Further, using \eqref{eq:M_value}, \eqref{eq:M_bar}, and \eqref{eq:M_hat}, we get
\begin{align}
    \label{eq:final_l2_test}
     \gL_{\rm te}(\hat\va(\bF_0)) - \gL_{\rm te}(\hat\va(\bF)) \to_P 
    &-\,\frac{c_{\star,1}^4 }{(\phi(c_\star^2 + \sigma_\ep^2) + c_{\star,1}^2)m_1^2}\frac{\partial m_2}{\partial \lambda} - \frac{4c_{\star,1}^4 c_{\star,2}^2}{3[\phi(c_\star^2 + \sigma_\ep^2) + c_{\star,1}^2]^2 m_1^2} \frac{\partial m_1}{\partial \lambda}.
\end{align}
Note that $\frac{\partial m_1}{\partial \lambda}, \frac{\partial m_2}{\partial \lambda} \leq 0$, concluding the proof.

\section{Proofs of Supplementary Lemmas}

\subsection{Proof of Lemma \ref{lem:Ms}}
\label{pflem:Ms}

Recalling $a_i \stackrel{i.i.d.}{\sim} \normal(0, 1/N)$ and $\langle \tilde \vx_i, \vbeta \rangle | \vbeta \stackrel{i.i.d.}{\sim} \normal(0, \Vert \vbeta \Vert_2^2)$, claims (a) and (b) follow from standard Gaussian maximal inequalities \cite[Section 2.2]{wellner2013weak} 
and from $\Vert \vbeta \Vert_2^2 = O_\sP(1)$; the latter follows 
by writing  
$\vbeta = n^{-1} \bX^\top (\sigma_\star(\bX \vbeta_\star ) + \ep)$, 
where $\ep = (\ep_1, \ldots,\ep_n)^\top$ and using our distributional assumptions on $\bX, \ep$, 
as well as Condition \ref{cond:te}.

By \cite[Theorem 5.39]{vershynin2010introduction} and \cite[Corollary A.21]{bai2010spectral}, we have $\Vert \bW_0 \bW_0^\top \Vert_\textnormal{op}, \Vert (\bW_0 \bW_0^\top)^{\circ 2} \Vert_\textnormal{op} = O_\sP(1)$.
Also, by \cite[Theorem 3.4.6]{vershynin2018high} and Gaussian maximal inequalities \cite[Section 2.2]{wellner2013weak}, we have $\max_{1 \leq i \neq \leq j \leq N} \langle \vw_{0,i}, \vw_{0,j} \rangle = O_\sP(n^{-\frac{1}{2}} \log^\frac{1}{2} n)$.
For $k \geq 3$,
\begin{align*}
    \Vert (\bW_0 \bW_0^\top)^{\circ k} \Vert_\textnormal{op} 
    &\leq \Vert (\bW_0 \bW_0^\top)^{\circ k}-\bI_N  \Vert_\textnormal{op} + 1
    \leq  \Vert (\bW_0 \bW_0^\top)^{\circ k} - \bI_N \Vert_\textnormal{F}+1\\
    &\leq \left( \sum_{1 \leq i \neq j \leq N} \langle \vw_{0, i}, \vw_{0, j} \rangle^{2k} \right)^\frac{1}{2} +1 = o_\sP(1)+1.
\end{align*}
Therefore,
\begin{align*}
    M_{W_0} \leq \max \left\{\Vert \bW_0 \bW_0^\top \Vert_\textnormal{op}, \Vert (\bW_0 \bW_0^\top)^{\circ 2}\Vert_\textnormal{op}, \sup_{k \geq 3} \Vert (\bW_0 \bW_0^\top)^{\circ k} \Vert_\textnormal{op} \right\} = O_\sP(1).  
\end{align*}

Claim (d) is standard, see e.g. \cite[Theorem 4.4.5]{vershynin2018high}.

\subsection{Proof of Lemma \ref{lemma:turn_beta_to_beta_star}}
\label{pflemma:turn_beta_to_beta_star}

We can write 
  \begin{align*}
        \mathbf{v}^{\top} (\vbeta - c_{\star,1} \vbeta_\star)
        &=
        n^{-1} \mathbf{v}^{\top} (\bX^\top (\sigma_\star(\bX \vbeta_\star ) + \boldsymbol{\ep}))- c_{\star,1} \vbeta_\star\nonumber\\
        &=
        n^{-1} \sum_{i=1}^n(\mathbf{v}^{\top}\vx_i \sigma_\star(\vx_i^\top \vbeta_\star ) - c_{\star,1} \mathbf{v}^{\top}\vbeta_\star) 
        +  n^{-1} \mathbf{v}^{\top}\boldsymbol{\ep}.
    \end{align*}
Now $n^{-1} \mathbf{v}^{\top}\boldsymbol{\ep} \sim \normal(0,\sigma_\ep^2\|\mathbf{v}\|_2^2)/n \to_P0$.
Moreover,
by Condition \ref{cond:tehe},
we can write 
$\sigma_\star(\vx^\top_i \vbeta_\star) = 
c_{\star,0} + c_{\star,1} \vx^\top_i \vbeta_\star + (P_{>1}\sigma_\star)(\vx^\top_i \vbeta_\star)$,
where 
conditional on $\vbeta_\star$,
$(P_{>1}\sigma_\star)(\vx^\top_i \vbeta_\star)$
is orthogonal in $L^2$ to the constant function and to $\vx^\top_i \vbeta_\star$.
Hence the first sum above equals
  \begin{align*}
        n^{-1}c_{\star,0} \mathbf{v}^{\top}\sum_{i=1}^n\vx_i
        + n^{-1} c_{\star,1} 
        \mathbf{v}^{\top} \biggl(\sum_{i=1}^n\vx_i \vx^\top_i 
        -  \bI\biggr) \vbeta_\star
        +n^{-1} \sum_{i=1}^n\mathbf{v}^{\top}\vx_i  (P_{>1}\sigma_\star)(\vx^\top_i \vbeta_\star).
    \end{align*}
For the first term,
$n^{-1}c_{\star,0} \mathbf{v}^{\top}\sum_{i=1}^n\vx_i\sim 
n^{-1}c_{\star,0} \cdot \normal(0,n\|\mathbf{v}\|_2^2)\to_P0$.
The 
second term is $c_{\star,1} $ times a  sample mean of i.i.d. random variables of the form 
$\mathbf{v}^{\top}(\vx_i \vx^\top_i-1)\vbeta_\star$, 
which have zero mean by the Gaussianity of $\vx_i$,
and for which all moments are finite.
Hence, by the weak law of large numbers, this term converges to zero in probability.

Similarly, the 
third term
is a sample mean of i.i.d.~random variables of the form
$\mathbf{v}^{\top}\vx_i  (P_{>1}\sigma_\star)(\vx^\top_i \vbeta_\star)$,
which have zero mean by the Gaussianity of $\vx_i$ and Lemma \ref{lem:twohermite},
and whose second moments are finite since $\sigma_\star$ is Lipschitz.
Hence, by the weak law of large numbers, this term also converges to zero in probability. This finishes the proof of the first claim.

Next, the second statement
follows from \cite[Lemma 18]{ba2022high}.
While that  work 
has slightly different assumptions on the teacher function $f_\star$,
it is straightforward to check that their proof goes through unchanged under our assumptions.
Specifically, their proof requires that $\vx\mapsto f_\star(\vx) = \sigma_\star(\vx^\top\vbeta_\star)$
is $O(1)$-Lipschitz, which holds in our case because $\sigma_\star$ is 
$O(1)$-Lipschitz,
and $\|\vbeta_\star\|_2 = O_\sP(1)$.

\subsection{Proof of Lemma \ref{lemma:train_loss}}
\label{pflemma:train_loss}

    By plugging in $\hat \va$ into the training loss, we find
\begin{align*}
    \gL_{\rm tr}(\bF) &= \frac{1}{n}\Vert \tilde\vy - \bF \hat\va\Vert_2^2 + \lambda \Vert\hat\va\Vert_2^2
    = \frac{1}{n}\Vert\tilde\vy\Vert_2^2 - \frac{2}{n}\tilde\vy^\top \bF \hat\va + \frac{1}{n} \hat\va^\top(\bF^\top\bF + \lambda n \bI_N)\hat\va\\
    &= \frac{1}{n}\Vert\tilde\vy\Vert_2^2 - \frac{1}{n}\tilde\vy^\top \bF \hat\va 
    = \frac{1}{n}\Vert\tilde\vy\Vert_2^2 - \frac{1}{n}\tilde\vy^\top \bF (\bF^\top \bF + \lambda n \bI_N)^{-1} \bF^\top \tilde\vy\\
    &= \frac{1}{n}\Vert\tilde\vy\Vert_2^2 - \frac{1}{n}\tilde\vy^\top \bF \bF^\top (\bF \bF^\top + \lambda n \bI_n)^{-1}  \tilde\vy \\
    &= \frac{1}{n}\Vert\tilde\vy\Vert_2^2 - \frac{1}{n}\tilde\vy^\top (\bF \bF^\top+\lambda n \bI_n) (\bF \bF^\top + \lambda n \bI_n)^{-1}  \tilde\vy + \lambda \tilde\vy^\top (\bF \bF^\top + \lambda n \bI_n)^{-1}\tilde\vy\\
    &= \lambda \tilde\vy^\top (\bF \bF^\top + \lambda n \bI_n)^{-1}\tilde\vy,
\end{align*}
which proves the lemma.

\subsection{Proof of Lemma \ref{lemma:poly-concentration}}
\label{pflemma:poly-concentration}

    To prove the concentration of this term around its mean, we will use the generalized Hanson-Wright inequality \cite[Theorem 2.1]{sambale2023some}  for $\alpha$-subexponential random variables. Note that, by definition, if $Z$ is a Gaussian random variable, $H_p(Z)$ is $2/p$-subexponential (see the definition in equation (1.1) of \cite{sambale2023some} and for these variables the Orlicz norm of order $2/p$ is bounded (see equation (1.3) of \cite{sambale2023some}). Also note that  $\Vert\bD\Vert_{\rm Fr} \leq \sqrt{n} \Vert\bD\Vert_{\rm op}= O_\sP(1/\sqrt{n})$. Thus, using \cite[Theorem 2.1]{sambale2023some} and setting $t = \frac{\log(n)}{\sqrt{n}}$, we find 
    \begin{align*}
        \sP\left(\Big|g(\bZ)^\top \bD\;g(\bZ) - \mathbb{E}[g(\bZ)^\top \bD \;g(\bZ)]\Big| \geq \frac{\log n}{\sqrt{n}}\right)\leq 2 \exp\left(-C \min\left\{\log^2(n),(\sqrt{n} \log n)^{1/p}\right\}\right),
    \end{align*}
    where $C>0$ is some constant. This concludes the proof.

\subsection{Proof of Lemma \ref{lemma:l1_limits}}
\label{pflemma:l1_limits}

First, we show that switching from $\vw_{0, i} \stackrel{i.i.d.}{\sim} \mathrm{Unif}(\mathbb{S}^{d-1})$ to $\hat\vw_{0, i} \stackrel{i.i.d.}{\sim} \normal(0, \frac{1}{d} \bI_d)$ will 
not change the limit of 
the terms $\frac{1}{d}\E\mathrm{tr}(\tilde\bX^\top \bar\bR_0\tilde\bX)$ and $\E\mathrm{tr}(\bar\bR_0)$ which will appear later in the proof.
First, we define $\hat\bW_0 = [\hat\vw_{0,1}, \cdots, \hat\vw_{0,N}]^\top$,
\begin{align*}
    &\bD = \text{diag}\left(\frac{1}{{\Vert\hat\vw_{0,1}\Vert_2}}, \cdots, \frac{1}{{\Vert\hat\vw_{0,N}\Vert_2}}\right),
    \, \bW_0 =_d\bD\hat\bW_0,
    \, \hat\bF_0 = \sigma(\tilde\bX \hat\bW_0^\top),\\
    &\,\text{and}\qquad \hat{\bar{\bR}}_0 = (\hat\bF_0\hat\bF_0^\top + \lambda n \bI_n)^{-1}.
\end{align*}
Then, 
\begin{align*}
    \biggl|\tr\left[\bar\bR_0 - \hat{\bar{\bR}}_0\right] \biggr| 
    &= \biggl|\tr\left[(\bF_0\bF_0^\top + \lambda n \bI_n)^{-1} - (\hat\bF_0\hat\bF_0^\top + \lambda n \bI_n)^{-1}\right]\biggr|\\
    &=\biggl|\tr\left[(\bF_0\bF_0^\top + \lambda n \bI_n)^{-1}(\bF_0\bF_0^\top - \hat\bF_0\hat\bF_0^\top)(\hat\bF_0\hat\bF_0^\top + \lambda n \bI_n)^{-1}\right]\biggr|\\
    &\leq \tr(\bF_0\bF_0^\top + \lambda n \bI_n)^{-1}\Vert(\hat\bF_0\hat\bF_0^\top + \lambda n \bI_n)^{-1}\Vert_{\rm op} \Vert\bF_0\bF_0 - \hat\bF_0\hat\bF_0\Vert_{\rm op}\\
    &\leq \frac{C}{n}\Vert\bF_0\bF_0 - \hat\bF_0\hat\bF_0\Vert_{\rm op}.
\end{align*}
Now, using the Gaussian equivalence from Appendix \ref{sec:gec}, we
can replace  $\bF_0$ and $\hat\bF_0$
with 
$\bF_0 = c_1 \tilde \bX \bW_0^\top + c_{>1}\bZ$ and $\hat\bF_0 = c_1 \tilde \bX \hat\bW_0^\top + c_{>1}\bZ$, respectively, without changing the limit. 
With this, we have
\begin{align*}
    &\bF_0\bF_0^\top - \hat\bF_0\hat\bF_0^\top 
    = c_1^2 \tilde \bX (\bW_0\bW_0^\top - \hat\bW_0\hat\bW_0^\top)\tilde \bX^\top + c_1c_{>1}\tilde \bX (\bW_0 - \hat\bW_0)^\top \bZ^\top + c_1c_{>1}\bZ(\bW_0 - \hat\bW_0)\tilde \bX^\top.
\end{align*}
Now, 
\begin{align*}
    \Vert\bW_0\bW_0^\top - \hat\bW_0\hat\bW_0^\top\Vert_{\rm op}\leq\Vert\bI_N - \bD\Vert_{\rm op}\Vert\bW_0\bW_0^\top\Vert_{\rm op}(\Vert\bD\Vert_{\rm op} + 1).
\end{align*}
Note that $\Vert\bW_0\bW_0^\top\Vert_{\rm op} = O_\sP(1)$, $\Vert\bD\Vert_{\rm op} = O_\sP(1)$, and $\Vert\bI_N -\bD\Vert_{\rm op} = o_\sP(1)$. Thus $\Vert\bW_0\bW_0^\top - \hat\bW_0\hat\bW_0^\top\Vert_{\rm op} = o_\sP(1)$. 
Also, similarly, $\Vert \bW_0 - \hat\bW_0\Vert_{\rm op} = o_\sP(1)$. 
Hence, noting that $\Vert\tilde \bX\Vert_{\rm op}$ and $\Vert\bZ\Vert_{\rm op}$ are both $O_\sP(\sqrt{N})$, we have 
$\frac{1}{n}\Vert\bF_0\bF_0^\top - \hat\bF_0\hat\bF_0^\top\Vert_{\rm op} \to_P 0$. This implies that $|\tr[\bar\bR_0 - \hat{\bar{\bR}}_0]| = o_\sP(1)$. 
Also,
\begin{align*}
    \left|\frac{1}{d}\tr\left[\tilde \bX^\top \bar\bR_0 \tilde \bX\right] - \frac{1}{d}\tr\left[\tilde \bX^\top \hat{\bar{\bR}}_0 \tilde \bX\right]\right| \leq |\tr[\bar\bR_0 - \hat{\bar{\bR}}_0]| \frac{\Vert\tilde \bX\tilde\bX^\top\Vert_{\rm op}}{d} \to_P 0.
\end{align*}

Finally, we can prove the required claims as follows:
    \begin{enumerate}
        \item[(a)] Since $\vbeta_\star \sim \normal(0, \frac{1}{d}\bI_d)$, we have
       $\vbeta_\star^{\top} \tilde\bX^\top \bar\bR_0 \tilde\bX\vbeta_\star = \frac{1}{d}\,\E\mathrm{tr}(\tilde\bX^\top \bar\bR_0\tilde\bX) + o_\sP(1)$,
        by the Hanson-Wright inequality. Note that by the argument above, we can assume that $\hat\vw_{0, i} \stackrel{i.i.d.}{\sim} \normal(0, \frac{1}{d} \bI_d)$ without changing the limiting trace.
        Further, from \cite[Proposition 1]{adlam2020neural}, see also \cite{adlam2019random}, 
        we have
        $\frac{1}{d}\,\E\mathrm{tr}(\tilde\bX^\top \bar\bR_0\tilde\bX) \to  \frac{\psi}{\phi} m_2$; see the discussion at the end of this proof for the detailed explanation.
        Now, we arrive at the conclusion by applying Lemma \ref{lemma:turn_beta_to_beta_star}.

        \item[(b)] Since $\va \sim \normal(0, \frac{1}{N} \bI_N)$, we have
$\va^\top \bF_0^\top \bar\bR_0\bF_0\va - \Vert\va\Vert_2^2 = \frac{1}{N}\tr\left(\bF_0^\top \bar\bR_0\bF_0\right) - 1 + o_\sP(1)$
by the Hanson-Wright inequality. Moreover,
\begin{align*}
&\bF_0^\top \bar\bR_0\bF_0 = \bF_0^\top\bF_0 (\bF_0^\top\bF_0 + \lambda n \bI_N)^{-1}\nonumber\\
&= (\bF_0^\top\bF_0+\lambda n \bI_N-\lambda n \bI_N) (\bF_0^\top\bF_0 + \lambda n \bI_N)^{-1} = \bI_N  - \lambda n (\bF_0^\top\bF_0 + \lambda n \bI_N)^{-1}.
\end{align*}
Hence, 
$\frac{1}{N}\tr\left(\bF_0^\top \bar\bR_0\bF_0\right) - 1  = -  \frac{\lambda n}{N} \tr(\bF_0^\top\bF_0 + \lambda n \bI_N)^{-1}$. From the argument above, we can assume that $\hat\vw_{0, i} \stackrel{i.i.d.}{\sim} \normal(0, \frac{1}{d} \bI_d)$ without changing the limiting trace.
It follows from \cite[Proposition 1]{adlam2020neural} that $\E\tr \bar\bR_0 \to \frac{\psi}{\phi} m_1$; 
again see the discussion at the end of this proof for the detailed explanation. 
Note that $\lim \E\tr \bar\bR_0$ is the limiting Stieltjes transform of $\bF_0\bF_0^\top$. Hence, $\bar m_1= \lim \E\tr(\bF_0^\top\bF_0 + \lambda n \bI_N)^{-1}$ is the limiting companion Stieltjes transform of $m_1$ which is given by
\begin{equation}\label{bm1}
\bar m_1 = \frac{\psi}{\phi} m_1 - \left(1-\frac{\phi}{\psi}\right) \frac{1}{\lambda}.
\end{equation} 
\end{enumerate}
    This concludes the proof.
    
For the reader's convenience, we provide the following diagram that shows how the notations of \cite{adlam2020neural} (left) 
match ($\Leftrightarrow$)
 ours  (right):
\begin{align*}
    &n_0 \Leftrightarrow d,\qquad
    n_1  \Leftrightarrow N,\qquad
    m \Leftrightarrow n,\qquad
    \phi, \psi \Leftrightarrow \phi, \psi,\\
    &\bX^\top \in \R^{m \times n_0} \Leftrightarrow \tilde \bX \in \R^{n \times d},
    \qquad
    \bF^\top \in \R^{m \times n_1} \Leftrightarrow \bF_0 \in \R^{n \times N},\qquad
    \sigma_{W_2} = 0,\\
    &\frac{1}{n_1} \bK( \lambda m /n_1)^{-1} = \frac{1}{n_1} \bF^\top \bF + \lambda \bI_m  \Leftrightarrow 
    \bar\bR_0^{-1} = 
    \bF_0 \bF_0^\top + \lambda n \bI_n, \quad \zeta\Leftrightarrow c_1^2,\qquad
    \eta \Leftrightarrow c_1^2 + c_{>1}^2,\\
    &\tau_1 = \frac{1}{m} \E \tr \bK^{-1}\Leftrightarrow m_1 = \frac{N}{n} \E \tr \bar \bR_0, \quad \tau_2 = \frac{1}{m n_0} \E \tr \bX^\top \bX \bK^{-1} \Leftrightarrow m_2 = \frac{N}{nd} \E \tr \tilde \bX \tilde \bX^\top \bar \bR_0 .
\end{align*}

\subsection{Proof of Lemma \ref{lemma:general_orthogonality_21}}
\label{pflemma:general_orthogonality_21}

Define $\hat \bX = \tilde\bX - \tilde\bX\vu \vu^\top$, which implies $\hat \bX \independent \tilde\bX\vu$ due to the Gaussianity of $\bX$. 
Based on the Gaussian equivalence from Appendix \ref{sec:gec}, 
we can replace 
$\bF_0$ with
$c_1 \tilde\bX \bW_0^\top + c_{>1} \bZ$,
    where $\bZ \in \R^{n \times d}$ is an independent random matrix with $\normal(0, 1)$ entries, without changing the conclusion.
    Hence, from now on, we write $\bF_0=c_1 \tilde\bX \bW_0^\top + c_{>1} \bZ$.
    Further, we define
    \begin{align}
    \label{hf}
        \hat \bF_0 = c_1 \hat\bX \bW_0^\top + c_{>1} \bZ.
    \end{align}
    Thus, 
    by the definition of $\hat \bX$,
    $\hat\bF_0 = \bF_0 - c_1 \tilde\bX \vu (\bW_0 \vu)^\top$. As a consequence, we also have
    $
        \bF_0 \bF_0^\top = \hat\bF_0\hat\bF_0^\top + \bV \bD \bV^\top,
    $
    where $\bV = \begin{bmatrix}\hat \bF_0 \bW_0 \vu & \tilde\bX\vu\end{bmatrix} \in \R^{n \times 2}$ and
    \begin{align*}
        \bD = \begin{bmatrix}
            0&c_1\\c_1&c_1^2 \Vert\bW_0\vu\Vert_2^2
        \end{bmatrix}.
    \end{align*}
    Noting that $D$ is invertible, and
    using the Woodbury formula, 
    with $\hat \bR_0 = (\hat\bF_0\hat\bF_0^\top + \lambda n \bI_n)^{-1}$,
    we find 
    \begin{align}\label{wb2}
        \bar\bR_0 = \hat \bR_0 - \hat \bR_0 \bV (\bD^{-1} + \bV^\top \hat\bR_0 \bV)^{-1}\bV^\top \hat\bR_0.
    \end{align}
     Now, we can write
    \begin{align*}
        &H_q(\tilde \bX \vu)^\top \bar\bR_0 H_p(\tilde \bX\vu) 
        = H_q(\tilde \bX \vu)^\top \hat\bR_0H_p(\tilde \bX\vu) - H_q(\tilde \bX \vu)^\top\hat \bR_0 \bV (\bD^{-1} + \bV^\top \hat\bR_0 \bV)^{-1}\bV^\top \hat\bR_0H_p(\tilde \bX\vu).
    \end{align*}
     Next, we can analyze each term in the above sum separately.

     The first term on the right hand side converges to zero by using Lemma \ref{lemma:poly-concentration} to prove the concentration of this term around its mean and noting that the mean is zero using the orthogonality property of Hermite polynomials (Lemma \ref{lem:twohermite}).

     To analyze the second term, we first study the matrix $\bK = (\bD^{-1} + \bV^\top \hat\bR_0 \bV)^{-1}$, writing
        \begin{align*}
            \bK^{-1}  =(\bD^{-1} + \bV^\top \hat\bR_0 \bV) &= 
            \begin{bmatrix}
                \vu^\top\bW_0^\top \hat\bF_0^\top \hat \bR_0 \hat\bF_0 \bW_0\vu - \Vert\bW_0\vu\Vert_2^2
                &\vu^\top \bW_0^\top \hat\bF_0^\top \hat \bR_0 \tilde\bX \vu - \frac{1}{c_1}\\
                \vu^\top \tilde\bX^\top \hat\bR_0 \hat\bF_0 \bW_0 \vu - \frac{1}{c_1} & \vu^\top \tilde\bX^\top \hat \bR_0 \tilde\bX \vu
            \end{bmatrix}.
        \end{align*}
        It can readily verified that all  elements in this matrix are $O_\sP(1)$ by checking the order of the operator and Euclidean norms. 
        Next, we analyze the terms in the expression
    \begin{align*}
         H_q(\tilde \bX \vu)^\top \hat \bR_0 \bV \bK\bV^\top \hat\bR_0H_p(\tilde\bX\vu) &= [\bK]_{1,1}H_q(\tilde \bX \vu)^\top \hat \bR_0  (\hat\bF_0\bW_0\vu)(\hat\bF_0\bW_0\vu)^\top\hat \bR_0 H_p(\tilde\bX \vu)\\
         &\hspace{1cm}+[\bK]_{1,2}H_q(\tilde \bX \vu)^\top \hat \bR_0  (\hat\bF_0\bW_0\vu)(\tilde\bX\vu)^\top\hat \bR_0 H_p(\tilde\bX \vu)\\
         &\hspace{1cm}+[\bK]_{2,1}H_q(\tilde \bX \vu)^\top \hat \bR_0  (\tilde\bX\vu)(\hat\bF_0\bW_0\vu)^\top\hat \bR_0 H_p(\tilde\bX \vu)\\
         &\hspace{1cm}+[\bK]_{2,2}H_q(\tilde \bX \vu)^\top \hat \bR_0  (\tilde\bX\vu)(\tilde\bX\vu)^\top\hat \bR_0 H_p(\tilde\bX \vu).
    \end{align*}
    
    Without loss of generality, we can assume that $p \neq 1$.
    \begin{itemize}
        \item \textbf{First Term.} Note that $H_q(\tilde \bX \vu)^\top$ and $H_p(\tilde \bX \vu)$ are orthogonal in $L^2$ by the properties of the Hermite polynomials, 
        and conditional on $\vu$, they are independent of $\hat \bR_0  (\hat\bF_0\bW_0\vu)(\hat\bF_0\bW_0\vu)^\top\hat \bR_0$. Moreover,
        $\Vert\hat \bR_0  (\hat\bF_0\bW_0\vu)(\hat\bF_0\bW_0\vu)^\top\hat \bR_0\Vert_{\rm op} = O_\sP(1/n)$.
    Thus, by using Lemma \ref{lemma:poly-concentration}, 
 this term converges to zero.

    \item \textbf{Second  Term.} Similar to the argument above, we can show that $(\tilde \bX \vu)^\top\hat \bR_0 H_p(\tilde\bX \vu)$ converges to zero. Also, by analyzing the operator norms, we have $H_q(\tilde \bX \vu)^\top \hat \bR_0  (\hat\bF_0\bW_0\vu) = O(1)$. This implies that the second term converges to zero.

    \item \textbf{Third  Term.} First, note that by a simple order-wise analysis, $H_q(\tilde \bX \vu)^\top \hat\bR_0 (\tilde \bX \vu) = O_\sP(1)$. Now, we have $H_p(\tilde \bX \vu)$ is independent of $(\hat\bF_0\bW_0\vu)^\top\hat \bR_0 $ and $\Vert(\hat\bF_0\bW_0\vu)^\top\hat \bR_0 \Vert_2 = O_\sP(1/\sqrt{n})$. The term $(\hat\bF_0\bW_0\vu)^\top\hat \bR_0 H_p(\tilde\bX \vu)$ converges to zero in probability by noting that $H_p(\tilde \bX \vu)$ is mean zero for $p \neq 0$. For the $p = 0$ case, we can use an orthogonality invariance argument identical to the one used to analyze \eqref{eq:left-ortho-argument}. 
    
    \item \textbf{Fourth Term.} This term also converges to zero because $(\tilde \bX \vu)^\top\hat \bR_0 H_p(\tilde\bX \vu)$ converges to zero, as argued above.
        \end{itemize}
Putting everything together, the proof is completed.

\subsection{Proof of Lemma \ref{lemma:delta_2_to_zero}}
\label{pflemma:delta_2_to_zero}

We will prove part (a) first. To do this, we will first  handle the cases where $p = 0$ and $p = 1$. 

For $p = 0$, we have
$\sqrt{N} H_0(\tilde\vtheta_\star) \bar \bR_0 \bF_0 \va^{\circ 2}= \sqrt{N}.$ This is identical to the second term in \eqref{eq:left-ortho-argument} and it is shown to be $o_{\sP}(1)$

For $p = 1$, we need to analyze
$    \sqrt{N} H_1(\tilde\vtheta_\star) \bar \bR_0 \bF_0 \va^{\circ 2}= \sqrt{N} {\vbeta_\star^{\top}}\tilde \bX^\top \bar \bR_0 \bF_0 \va^{\circ 2}.
$
Note that $\vbeta_\star \sim \normal(0, \frac{1}{d}\bI_d)$ is independent of $\sqrt{N}\tilde \bX^\top \bar \bR_0 \bF_0 \va^{\circ 2}$ and 
\begin{align*}
    \Vert\sqrt{N}\tilde \bX^\top \bar \bR_0 \bF_0 \va^{\circ 2}\Vert_2 \leq \sqrt{N} \Vert\tilde \bX\Vert_{\rm op}\cdot \Vert\bar\bR_0\Vert_{\rm op}\cdot \Vert \bF_0\Vert_{\rm op}\cdot \Vert\va^{\circ 2}\Vert_2 = O_\sP(1).
\end{align*}
Thus, we can conclude that $\sqrt{N} H_1(\tilde\vtheta_\star) \bar \bR_0 \bF_0 \va^{\circ 2}\to 0$ in probability.

To analyze the case where $p>1$, we first  define $\hat \bX = \tilde \bX - \tilde\vtheta_\star \vbeta_\star^{\top}$. 
By construction, we have $\hat\bX \independent \tilde\vtheta_\star$. 
As in the proof of Lemma \ref{lemma:general_orthogonality_21}, 
Based on the Gaussian equivalence from Appendix \ref{sec:gec}, we can replace $\bF_0$ with $c_1 \tilde\bX \bW_0^\top + c_{>1} \bZ$ in our computations without changing the limiting result, where $\bZ \in \R^{n \times d}$ is an independent random matrix with $\normal(0, 1)$ entries. 
Thus, from now on, we denote $\bF_0 = c_1 \tilde\bX \bW_0^\top + c_{>1} \bZ$. 
We define $\hat \bF_0$ as in \eqref{hf}.  Thus, $\hat\bF_0 = \bF_0 - c_1 \tilde\vtheta_\star (\bW_0 \vbeta_\star)^\top$. As a consequence, we can write $\bF_0 \bF_0^\top = \hat\bF_0\hat\bF_0^\top + \bV \bD \bV^\top$,
    where $\bV = \begin{bmatrix}\hat \bF_0 \bW_0 \vbeta_\star & \tilde\vtheta_\star\end{bmatrix} \in \R^{n \times 2}$ and
    \begin{align*}
        \bD = \begin{bmatrix}
            0&c_1\\c_1&c_1^2 \Vert\bW_0\vbeta_\star\Vert_2^2
        \end{bmatrix}.
    \end{align*}
    Using the Woodbury formula, we find
    that \eqref{wb2} still holds.
    Now, we can write
    \begin{align}
    \label{eq:lemma_reference_to_two_terms}
        \sqrt{N}H_p&(\tilde\vtheta_\star)^\top \bar\bR_0 \bF_0 \va^{\circ 2}\\ &= \sqrt{N}H_p(\tilde\vtheta_\star)^\top  \hat\bR_0 \bF_0 \va^{\circ 2}
        - \sqrt{N}H_p(\tilde\vtheta_\star)^\top \hat \bR_0 \bV (\bD^{-1} + \bV^\top \hat\bR_0 \bV)^{-1}\bV^\top \hat\bR_0\bF_0 \va^{\circ 2}\nonumber\\
        &= \sqrt{N}H_p(\tilde\vtheta_\star)^\top  \hat\bR_0 (\hat \bF_0 + c_1 \tilde\vtheta_\star (\bW_0 \vbeta_\star)^\top) \va^{\circ 2}\nonumber\\
         &\hspace{1cm}- \sqrt{N}H_p(\tilde\vtheta_\star)^\top \hat \bR_0 \bV (\bD^{-1} + \bV^\top \hat\bR_0 \bV)^{-1}\bV^\top \hat\bR_0(\hat \bF_0 + c_1 \tilde\vtheta_\star (\bW_0 \vbeta_\star)^\top) \va^{\circ 2}.\nonumber
    \end{align}
     Now, we can analyze each term in the above sum separately.
    \paragraph{Term 1.} Note that by a simple orderwise analysis,
    \begin{align*}
        \Vert\sqrt{N}\,  \hat\bR_0 \hat\bF_0 \va^{\circ 2} \Vert_{\rm op} &\leq \sqrt{N} \Vert\hat\bR_0\Vert_{\rm op}\Vert\hat\bF_0\Vert_{\rm op}\Vert\va^{\circ 2}\Vert_2= O(1/\sqrt{N}).
    \end{align*}
    We have $\Vert H_p(\tilde\vtheta_\star)\Vert_2 = O_\sP(\sqrt{N})$, \;$\mathbb{E}[H_p(\tilde\vtheta_\star)] = 0$, and $H_p(\tilde\vtheta_\star)$ has independent entries. Also $H_p(\tilde\vtheta_\star) \independent  \hat\bR_0 \hat\bF_0 \va^{\circ 2}$. Thus,
    $\sqrt{N}H_p(\tilde\vtheta_\star)^\top  \hat\bR_0 \hat\bF_0 \va^{\circ 2} \to_P 0$.
    
    We now need to analyze $\sqrt{N}H_p(\tilde\vtheta_\star)^\top  \hat\bR_0 \tilde\vtheta_\star \vbeta_\star^\top\bW_0^\top  \va^{\circ 2}$. 
    Note that $H_p(\tilde\vtheta_\star)^\top  \hat\bR_0 \tilde\vtheta_\star = O_\sP(1)$ by a simple order analysis of the norms. We also have 
    $\sqrt{N} \vbeta_\star^{\top} \bW_0^\top \va^{\circ 2} \to_P 0$,
    because $\vbeta_\star \sim \normal(0, \frac{1}{d}\bI_d)$ is independent of the norm bounded vector $\sqrt{N} \bW_0^\top \va^{\circ 2}$.

    \paragraph{Term 2.}  To analyze the second term, we  first study the matrix $\bK = (\bD^{-1} + \bV^\top \hat\bR_0 \bV)^{-1}$:
        \begin{align*}
            \bK^{-1}  =(\bD^{-1} + \bV^\top \hat\bR_0 \bV) &= 
            \begin{bmatrix}
                {\vbeta}_\star^{\top}\bW_0^\top \hat\bF_0^\top \hat \bR_0 \hat\bF_0 \bW_0{\vbeta}_\star - 
                \Vert\bW_0\vbeta_\star\Vert_2^2
                &{\vbeta}_\star^{\top} \bW_0^\top \hat\bF_0^\top \hat \bR_0 \tilde\vtheta_\star - \frac{1}{c_1}\\
                {\vbeta}_\star^{\top} \tilde\bX^\top \hat\bR_0 \hat\bF_0 \bW_0 \vbeta_\star - \frac{1}{c_1} & 
                {\vbeta}_\star^{\top}\tilde\bX^\top \hat \bR_0 \tilde\vtheta_\star
            \end{bmatrix}.
        \end{align*}
        By orderwise analysis, all elements in this matrix converge to  deterministic $O_\sP(1)$ values in probability. We write the second term in \eqref{eq:lemma_reference_to_two_terms} as follows:
             \begin{align*}
         \sqrt{N}H_p(\tilde\vtheta_\star)^\top\hat \bR_0 &\bV \bK\bV^\top \hat\bR_0\bF_0\va^{\circ 2}\\ &= [\bK]_{1,1}H_p(\tilde\vtheta_\star)^\top\hat \bR_0 (\hat\bF_0\bW_0\vbeta_\star)(\hat\bF_0\bW_0\vbeta_\star)^\top\hat\bR_0\bF_0(\sqrt{N}\va^{\circ 2})\\
         &+[\bK]_{1,2}H_p(\tilde\vtheta_\star)^\top\hat \bR_0 (\hat\bF_0\bW_0\vbeta_\star)\tilde\vtheta_\star^\top\hat\bR_0\bF_0(\sqrt{N}\va^{\circ 2})\\
         &+[\bK]_{2,1}H_p(\tilde\vtheta_\star)^\top\hat \bR_0 (\tilde\vtheta_\star)(\hat\bF_0\bW_0\vbeta_\star)^\top\hat\bR_0\bF_0(\sqrt{N}\va^{\circ 2})\\
         &+[\bK]_{2,2}H_p(\tilde\vtheta_\star)^\top\hat \bR_0 (\tilde\vtheta_\star)\tilde\vtheta_\star^\top\hat\bR_0\bF_0(\sqrt{N}\va^{\circ 2}).
    \end{align*}
    In the sum above, we will show that each term converges to zero.
    \begin{itemize}
        \item \underline{First term}: By orderwise analysis, we have
        $\Vert\hat \bR_0 (\hat\bF_0\bW_0\vbeta_\star)\Vert_{\rm op} = O_\sP(1/\sqrt{N})$.
                Further, $H_p(\tilde\vtheta_\star)$ is independent of it (only considering the randomness in $\tilde \bX$) with mean zero and $\Vert H_p(\tilde\vtheta_\star)\Vert_2 = O_\sP(\sqrt{N})$. This implies that 
        \begin{align}
            \label{eq:term1_part1}
            H_p(\tilde\vtheta_\star)^\top\hat \bR_0 (\hat\bF_0\bW_0\vbeta_\star) \to_P 0.
        \end{align}

        We can use a simple order argument to show that 
        $\sqrt{N}(\hat\bF_0\bW_0\vbeta_\star)^\top\hat\bR_0\bF_0\va^{\circ 2} = O_\sP(1)$. Thus, the first term converges to zero.

        \item \underline{Second term}: For this term, we use that $H_p(\tilde\vtheta_\star)^\top\hat \bR_0 (\hat\bF_0\bW_0\vbeta_\star) \to_P 0$. We can also use an orderwise analysis to prove that
        $
        \sqrt{N}(\tilde\vtheta_\star)^\top\hat\bR_0\bF_0\va^{\circ 2} = O_\sP(1).$
        This proves that the second term also converges to zero.

        \item \underline{Third term}: By a simple orderwise analysis, we have
        $
        \sqrt{N}(\hat\bF_0\bW_0\vbeta_\star)^\top\hat\bR_0\bF_0\va^{\circ 2} = O_\sP(1)$.
        To show that the third term converges to zero, it is enough to show that 
        $           H_p(\tilde\vtheta_\star)^\top\hat \bR_0 (\tilde\vtheta_\star) \to_P 0$,
        which is true for $p \neq 1$ by using Lemma \ref{lemma:poly-concentration} and the orthogonality property of Hermite polynomials (Lemma \ref{lem:twohermite}).
        
        \item \underline{Fourth term}: By a simple orderwise analysis, we have
$\sqrt{N}\tilde\vtheta_\star^\top\hat\bR_0\bF_0\va^{\circ 2} = O_\sP(1)$.
        Again,  to show that the fourth term converges to zero, it is enough to show that 
$           H_p(\tilde\vtheta_\star)^\top\hat \bR_0 (\tilde\vtheta_\star) \to_P 0$,
        which is true for $p \neq 1$ as argued above.
    \end{itemize}

Putting everything together, 
part (a) follows.
The proof for part (b) is identical and omitted.

\subsection{Proof of Lemma \ref{lemma:s12}}
\label{pflemma:s12}

We will study the cases where $s = 1$ and $s = 2$ separately.  For $s = 1$,  we can use Lemma \ref{lemma:turn_beta_to_beta_star} to show that $H_p(\tilde\vtheta_\star) \bar\bR_0 \tilde\vtheta = c_{\star,1} H_p(\tilde\vtheta_\star) \bar\bR_0 \tilde\vtheta_\star +o_\sP(1)$. 
Also, by Lemma \ref{lemma:general_orthogonality_21}, we have $H_p(\tilde\vtheta_\star) \bar\bR_0 (\tilde\vtheta_\star) = o(1)$ in probability if $p \neq 1$, which proves the lemma.

For the case $s = 2$, we define $\tilde \vbeta = \vbeta/\Vert\vbeta\Vert_2$ and write
\begin{align*}
    H_p(\tilde\vtheta_\star) \bar\bR_0 (\tilde\vtheta)^{\circ 2} &= \Vert\vbeta\Vert_2^2\;H_p(\tilde\vtheta_\star) \bar\bR_0 (\tilde\bX \tilde\vbeta)^{\circ 2}
    = \Vert\vbeta\Vert_2^2\;H_p(\tilde\vtheta_\star) \bar\bR_0 H_2(\tilde\bX \tilde\vbeta) + o_\sP(1).
\end{align*}
Now, we define $\vbeta_\perp = \frac{\vbeta_\star - \langle \vbeta_\star,\, \tilde\vbeta\rangle \tilde\vbeta}{\Vert\vbeta_\star - \langle \vbeta_\star,\, \tilde\vbeta\rangle \tilde\vbeta \Vert_2}$, and set
$\hat \bX = \tilde\bX - \tilde \bX \tilde\vbeta\tilde \vbeta^\top - \tilde \bX \vbeta_\perp \vbeta_\perp^\top$.
By construction, we have $\hat \bX \independent \tilde \bX \tilde\vbeta, \tilde\vtheta_\star$. Based on the Gaussian equivalence from Appendix \ref{sec:gec}, we
can again replace $\bF_0$ with
$\bF_0 = c_1 \tilde\bX \bW_0^\top + c_{>1} \bZ$,
    where $\bZ \in \R^{n \times d}$ is an independent random matrix with $\normal(0, 1)$ entries. 
    Again, we define $\hat \bF_0$ as in \eqref{hf}.
    Thus, $\hat\bF_0 = \bF_0 - c_1 \tilde\bX \tilde\vbeta (\bW_0 \tilde \vbeta)^\top - c_1 \tilde \bX \vbeta_\perp (\bW_0 \vbeta_\perp)^\top$. As a consequence, we also have
    $
        \bF_0 \bF_0^\top = \hat\bF_0\hat\bF_0^\top + \bV \bD \bV^\top,
    $
    where 
    $\bV = \begin{bmatrix}\tilde\bX \tilde\vbeta& \tilde \bX \vbeta_\perp & \hat\bF_0 \bW_0 \tilde\vbeta&\hat\bF_0\bW_0\vbeta_\perp\end{bmatrix} \in \R^{n \times 4}$ and
    \begin{align*}
        \bD = \begin{bmatrix}
            c_1^2 \langle \bW_0 \tilde\vbeta, \bW_0 \tilde\vbeta\rangle & c_1^2 \langle \bW_0 \tilde\vbeta, \bW_0 \vbeta_\perp\rangle & c_1 & 0\vspace{0.2cm} \\
            c_1^2 \langle \bW_0 \tilde\vbeta, \bW_0 \vbeta_\perp\rangle&c_1^2 \langle \bW_0 \vbeta_\perp, \bW_0 \vbeta_\perp\rangle& 0 & c_1\vspace{0.2cm} \\
            c_1 & 0 & 0 & 0 \vspace{0.2cm} \\
            0 & c_1 & 0 & 0
        \end{bmatrix}.
    \end{align*}
    Using the Woodbury formula, we find that \eqref{wb2} still holds.
    We can write 
    \begin{align}\label{bigwb}
        H_p(\tilde\vtheta_\star)^\top \bar\bR_0 H_2(\tilde\bX \tilde\vbeta)= 
        &H_p(\tilde\vtheta_\star)^\top \hat\bR_0 H_2(\tilde\bX \tilde\vbeta)\\
        &-H_p(\tilde\vtheta_\star)^\top \hat \bR_0 \bV (\bD^{-1} + \bV^\top \hat\bR_0 \bV)^{-1}\bV^\top \hat\bR_0H_2(\tilde\bX \tilde\vbeta).\nonumber
    \end{align}
The first term converges to zero for any $p\neq 2$,
analogously to the argument in Section \ref{section:l2-Hinverse} for the term (1,2).

To prove that the second term will also converge to zero, we first observe that the elements of $\bK = (\bD^{-1} + \bV^\top \hat\bR_0 \bV)^{-1}$ are all $O_\sP(1)$. The second term will involve quantities of the form
$[\bK]_{i,j} H_p(\tilde\vtheta_\star)^\top\hat\bR_0\vv_i\vv_j^\top\hat\bR_0 H_2(\tilde \bX \tilde\vbeta)$,    
where $\vv_i$, for $i \in \{1, 2, 3, 4\}$, is the $i$-th column of the matrix $\bV=\begin{bmatrix}\tilde\bX \tilde\vbeta& \tilde \bX \vbeta_\perp & \hat\bF_0 \bW_0 \tilde\vbeta&\hat\bF_0\bW_0\vbeta_\perp\end{bmatrix}$. We can argue that all these terms converge to zero, as follows:

\begin{itemize}
    \item The terms where $j = 1$ converge to zero because  $(\tilde\bX\tilde\vbeta)^\top\hat\bR_0 H_2(\tilde \bX \tilde\vbeta)$ 
    converges to zero
    analogously to the argument in Section \ref{section:l2-Hinverse} for the term (1,2).
    The same argument applies 
    to the terms where $j = 2$, via the convergence of $(\tilde \bX \vbeta_\perp)^\top \hat\bR_0 H_2(\tilde \bX \tilde\vbeta)$ to zero. 
    
    \item For $j=3,4$,
    since $H_2(\tilde\bX\tilde\vbeta)$ is independent of
$\hat\bR_0[\hat\bF_0\bW_0\tilde\vbeta\quad\,\,\,\hat\bF_0\bW_0\vbeta_\perp]$,
and has zero-mean i.i.d.~entries, it also follows 
that these entries converge to zero in probability.
\end{itemize}

Finally we study $H_2(\tilde\vtheta_\star)^\top \bar\bR_0\tilde\vtheta^{\circ 2}$, by analyzing the terms in \eqref{bigwb} for $p=2$.

For $H_2(\tilde\vtheta_\star)^\top \hat\bR_0 H_2(\tilde\bX \tilde\vbeta)$, 
since 
$H_2(\tilde\vtheta_\star), H_2(\tilde\bX \tilde\vbeta)$
are independent of $ \hat\bR_0$,
it follows from Lemma \ref{lemma:poly-concentration}, 
as in the analysis of term $(1,2)$
in Section \ref{pl2},
that 
$H_2(\tilde\vtheta_\star)^\top \hat\bR_0 H_2(\tilde\bX \tilde\vbeta) -  \E\hat\bR_0 \cdot \E H_2(\tilde\vtheta_\star)^\top H_2(\tilde\bX \tilde\vbeta) \to_P 0$.
Now notice that $\hat \bF_0$ is left-orthogonally invariant in distribution, 
and thus 
$\hat\bR_0 =_d 
\bO \hat\bR_0 \bO^\top$,
 where $\bO$
    is uniformly distributed over the Haar measure of $n$-dimensional orthogonal matrices, independently of all other randomness. 
Hence, 
$\E\hat\bR_0 = \E\tr\hat\bR_0 \bI_n/n$.
Moreover, from the Woodbury formula in \eqref{wb},
    \begin{align*}
        &|\tr\bar\bR_0  - \tr\hat \bR_0|
        \le |\tr \hat \bR_0 \bV (\bD^{-1} + \bV^\top \hat\bR_0 \bV)^{-1}\bV^\top \hat\bR_0|
        \le |\tr (\bD^{-1} + \bV^\top \hat\bR_0 \bV)^{-1}\bV^\top\bV |\cdot \|\hat\bR_0\|_{\op}^2.
    \end{align*}
    From our previous analysis and as the entries of $\bV^\top\bV$ are $O_\sP(n)$, it follows that the first term is $O_\sP(n)$; whereas $\|\hat\bR_0\|_{\op}^2 = O(1/n^2)$. Hence, $|\tr\bar\bR_0  - \tr\hat \bR_0|\to_P 0$, and thus by the bounded convergence theorem $|\E\tr\bar\bR_0  - \E\tr\hat \bR_0|\to_P 0$.
        Moreover, we have already argued 
    in the proof of Lemma \ref{lemma:l1_limits}
    that $\E\tr\bar\bR_0\to  \psi m_1/\phi$.
 
Further, by Lemmas \ref{lem:twohermite} and \ref{lemma:turn_beta_to_beta_star},
\begin{align*}
\E H_2(\tilde\vtheta_\star)^\top H_2(\tilde\bX \tilde\vbeta) &= n \cdot \E H_2 (\tilde \vx_1^\top \vbeta_\star) H_2 (\tilde \vx_1^\top\tilde\vbeta) \\
&= 2n\E(\vbeta_\star^\top \tilde\vbeta)^2
= 2n\E\frac{(\vbeta_\star^\top\vbeta)^2}{\|\vbeta\|^2}
= 2n\frac{c_{\star,1}^2}{\phi(c_\star^2 + \sigma_\ep^2) + c_{\star,1}^2} +o_\sP(1).    
\end{align*}
This shows that 
$$H_2(\tilde\vtheta_\star)^\top \hat\bR_0 H_2(\tilde\bX \tilde\vbeta)\to_P 
2\frac{\psi m_1}{\phi}
\frac{c_{\star,1}^2}{\phi(c_\star^2 + \sigma_\ep^2) + c_{\star,1}^2}.$$

Next, 
we consider
$H_2(\tilde\vtheta_\star)^\top \hat \bR_0 \bV$
with 
$\bV = \begin{bmatrix}\tilde\bX \tilde\vbeta& \tilde \bX \vbeta_\perp & \hat\bF_0 \bW_0 \tilde\vbeta&\hat\bF_0\bW_0\vbeta_\perp\end{bmatrix}$.
For the first two entries of the vector 
$H_2(\tilde\vtheta_\star)^\top \hat \bR_0 \bV$,
an analysis very similar to the one above 
for 
$H_2(\tilde\vtheta_\star)^\top \hat\bR_0 H_2(\tilde\bX \tilde\vbeta)$
shows that they converge to zero in probability.
For the last two entries,
since $H_2(\tilde\vtheta_\star)$ is independent of
$\hat\bR_0[\hat\bF_0\bW_0\tilde\vbeta\quad\,\,\,\hat\bF_0\bW_0\vbeta_\perp]$,
and has zero-mean i.i.d.~entries, it also follows 
that these entries converge to zero in probability.
Moreover, 
     the limiting entries of 
     $(\bD^{-1} + \bV^\top \hat\bR_0 \bV)^{-1}$
     have been shown to be bounded in our above analysis. 
Hence, the second term
converges to zero in probability.

Now, note that $\tilde\vbeta = \vbeta/\Vert\vbeta\Vert_2$. From Lemma \ref{lemma:turn_beta_to_beta_star},
$\|\vbeta\|^2  \to_P\phi(c_\star^2 + \sigma_\ep^2) + c_{\star,1}^2$.
    Hence, 
$$H_2(\tilde\vtheta_\star)^\top \hat\bR_0 H_2(\tilde\bX \vbeta) = 
2\frac{\psi m_1}{\phi}
\frac{c_{\star,1}^2\Vert\vbeta\Vert_2^2}{\phi(c_\star^2 + \sigma_\ep^2) + c_{\star,1}^2} + o_\sP(1) \to_P \frac{2c_{\star,1}^2\psi m_1}{\phi},$$
which concludes the proof.

\subsection{Proof of Lemma \ref{t2rt2}}
\label{pft2rt2}
As in the proof of Lemma \ref{lemma:delta_2_to_zero},
 we define 
 $\hat \bX = \tilde \bX - \tilde\vtheta \vbeta^{\top}$. 
By construction, we have $\hat\bX \independent \tilde\vtheta$. 
As in the proof of Lemma \ref{lemma:general_orthogonality_21}, 
based on the Gaussian equivalence from Appendix \ref{sec:gec}, we can replace $\bF_0$ with $c_1 \bX \bW_0^\top + c_{>1} \bZ$ in our computations without changing the limiting result, where $\bZ \in \R^{n \times d}$ is an independent random matrix with $\normal(0, 1)$ entries. 
Thus, from now on, we denote $\bF_0 = c_1 \bX \bW_0^\top + c_{>1} \bZ$. 
We define $\hat \bF_0$ as in \eqref{hf}; thus, $\hat\bF_0 = \bF_0 - c_1 \tilde\vtheta (\bW_0 \vbeta)^\top$. As a consequence, we can write $\bF_0 \bF_0^\top = \hat\bF_0\hat\bF_0^\top + \bV \bD \bV^\top$,
    where $\bV = \begin{bmatrix}\hat \bF_0 \bW_0 \vbeta & \tilde\vtheta\end{bmatrix} \in \R^{n \times 2}$ and
    \begin{align*}
        \bD = \begin{bmatrix}
            0&c_1\\c_1&c_1^2 \Vert\bW_0\vbeta\Vert_2^2
        \end{bmatrix}.
    \end{align*}
    Using the Woodbury formula, we find
    that \eqref{wb2} still holds.
    Now, we can write
    \begin{align}
    \label{eq:lemma_reference_to_two_terms2}
        \tilde\vtheta^{\circ2\top} \bar\bR_0  \tilde\vtheta^{\circ2} &= \tilde\vtheta^{\circ2\top}  \hat\bR_0  \tilde\vtheta^{\circ2}
        - \tilde\vtheta^{\circ2\top} \hat \bR_0 \bV (\bD^{-1} + \bV^\top \hat\bR_0 \bV)^{-1}\bV^\top \hat\bR_0 \tilde\vtheta^{\circ2}.
    \end{align}
    We can analyze each term in the above sum separately.

    By Lemma \ref{lemma:poly-concentration}, $\tilde\vtheta^{\circ2\top}  \hat\bR_0  \tilde\vtheta^{\circ2} - \E \tilde\vtheta^{\circ2\top}  \hat\bR_0  \tilde\vtheta^{\circ2} \to_P 0$.
    Further,
    conditional on $\vbeta$,
    $\E \tilde\vtheta^{\circ2\top}  \hat\bR_0  \tilde\vtheta^{\circ2} = 3\|\vbeta\|_2^4 \E \tr \hat\bR_0$;
    and as in the proof of Lemma \ref{lemma:s12},
    $\E\tr\hat\bR_0-\E\tr\bar\bR_0\to  0$.
    Moreover, we have already argued 
    in the proof of Lemma \ref{lemma:l1_limits}
    that $\E\tr\bar\bR_0\to  \psi m_1/\phi$.
    In addition, from Lemma \ref{lemma:turn_beta_to_beta_star},
$\|\vbeta\|^2  \to_P\phi(c_\star^2 + \sigma_\ep^2) + c_{\star,1}^2$.
    Hence,
    \begin{align*}
        \tilde\vtheta^{\circ2\top}  \hat\bR_0  \tilde\vtheta^{\circ2}\to_P 
        3\psi m_1[\phi(c_\star^2 + \sigma_\ep^2) + c_{\star,1}^2]^2/\phi.
    \end{align*}
     To analyze the second term in \eqref{eq:lemma_reference_to_two_terms2}, we first study 
      $\tilde\vtheta^{\circ2\top} \hat \bR_0 \hat \bF_0 \bW_0\vbeta$.
    By an argument similar to the ones above, we can show that it concentrates around
    $\mathbf{1}_n^\top \hat \bR_0 \hat \bF_0 \bW_0\vbeta =
    \mathbf{1}_n^\top  \hat \bF_0\hat \bR_0 \bW_0\vbeta$.
    Since $\hat \bF_0$ is left-orthogonally invariant, 
    $\mathbf{1}_n^\top  \hat \bF_0\hat \bR_0 \bW_0\vbeta = _d \mathbf{1}_n^\top \bO \hat \bF_0\hat \bR_0 \bW_0\vbeta$, where $\bO$
    is uniformly distributed over the Haar measure of $n$-dimensional orthogonal matrices, independently of all other randomness. 
    Then, it follows 
    as in the analysis of  term (1,2) from Section \ref{pl2} that $ \mathbf{1}_n^\top \bO \hat \bF_0\hat \bR_0 \bW_0\vbeta\to_P0$; 
    and hence  $\tilde\vtheta^{\circ2\top} \hat \bR_0 \hat \bF_0 \bW_0\vbeta\to_P0$.

     Moreover, 
     the limiting entries of 
     $(\bD^{-1} + \bV^\top \hat\bR_0 \bV)^{-1}$
     can be shown to be bounded by a simple orderwise analysis. Hence,  the second term in \eqref{eq:lemma_reference_to_two_terms2} is $o_\sP(1)$.

\subsection{Proof of Lemma \ref{lemma:hermite_limit_general}}
\label{pflemma:hermite_limit_general}

Denoting $\tilde \vbeta = \vbeta/\Vert\vbeta\Vert_2$, we have
\begin{align*}
(\tilde \bX  \vbeta)^{\circ i \top}&\bar\bR_0(\tilde \bX \vbeta)^{\circ j} =  \Vert\vbeta\Vert_2^{i+j}(\tilde \bX \tilde \vbeta)^{\circ i \top}\bar\bR_0(\tilde \bX \tilde \vbeta)^{\circ j}
=\Vert\vbeta\Vert_2^{i+j}\sum_{k_1 = 0}^{i}\sum_{k_2 = 0}^{j}\xi_{i,k_1} \xi_{j,k_2}H_{k_1}(\tilde\bX\tilde\vbeta)^\top\bar\bR_0 H_{k_2}(\tilde\bX\tilde\vbeta)\\
&=\Vert\vbeta\Vert_2^{i+j}\sum_{k = 0}^{\min(i,j)}\xi_{j,k}\xi_{i,k}H_{k}(\tilde\bX\tilde\vbeta)^\top\bar\bR_0 H_{k}(\tilde\bX\tilde\vbeta) +o_\sP(1)\\
&=\Vert\vbeta\Vert_2^{i+j} \left[ \xi_{i,1}\xi_{j,1} (\tilde \bX\tilde\vbeta)^\top\bar\bR_0 (\tilde \bX\tilde\vbeta)+\sum_{k = 0,\; k \neq 1}^{\min(i,j)}\xi_{i,k}\xi_{j,k} H_{k}(\tilde \bX\tilde\vbeta)^\top\bar\bR_0 H_{k}(\tilde \bX\tilde\vbeta)\right] +o_\sP(1).
\end{align*}
The third line follows from  Lemma~\ref{lemma:generalized_hermite_pq}.  Now, we claim that for any $p \in \{0,2, 3, \dots\}$, we have $H_p(\tilde \bX \vbeta/\Vert\vbeta\Vert_2)^{\top}\bar\bR_0 H_p(\tilde \bX \vbeta/\Vert\vbeta\Vert_2)
    \to_P p!\;\psi m_1/\phi$.
Using this claim, the facts that $\Vert\vbeta\Vert_2^2 \to_P c_{\star,1}^2 + \phi(c_\star^2 + \sigma_\ep^2)$, and $\tr(\tilde \bX^\top (\bF_0\bF_0^\top + \lambda n \bI_n)^{-1} \tilde \bX)/d \to_P  \psi m_2/\phi$, we can conclude
\begin{align*}
    (\tilde \bX  \vbeta)^{\circ i \top}\bar\bR_0(\tilde \bX \vbeta)^{\circ j} \to_P \left(c_{\star,1}^2 + \phi(c_\star^2 + \sigma_\ep^2)\right)^{(i+j)/2} \left[\xi_{i,1}\xi_{j,1} \frac{\psi m_2}{\phi}+\frac{\psi m_1}{\phi}\sum_{k = 0, \; k\neq 1}^{\min(i,j)} k!\;\xi_{i,k}\xi_{j,k}\right].
\end{align*}

Now, it remains to prove the claim that for any $p \in \{0,2, 3, \dots\}$, we have 
\begin{align*}
    H_p(\tilde \bX \vbeta/\Vert\vbeta\Vert_2)^{\top}\bar\bR_0 H_p(\tilde \bX \vbeta/\Vert\vbeta\Vert_2)
    \to_P p!\;\psi m_1/\phi.
\end{align*}

As in the proof of Lemma \ref{t2rt2},
 we define 
 $\hat \bX = \tilde \bX - \tilde\bX \tilde \vbeta \tilde\vbeta^{\top}$. 
By construction, we have $\hat\bX \independent \tilde\bX\tilde\vbeta$. 
As in the proof of Lemma \ref{lemma:general_orthogonality_21}, 
based on the Gaussian equivalence from Appendix \ref{sec:gec}, we can replace $\bF_0$ with $c_1 \tilde\bX \bW_0^\top + c_{>1} \bZ$ in our computations without changing the limiting result, where $\bZ \in \R^{n \times d}$ is an independent random matrix with $\normal(0, 1)$ entries. 
Thus, from now on, we denote $\bF_0 = c_1 \tilde\bX \bW_0^\top + c_{>1} \bZ$. 
We define $\hat \bF_0$ as in \eqref{hf}; thus, $\hat\bF_0 = \bF_0 - c_1 \tilde\bX\tilde\vbeta (\bW_0 \tilde\vbeta)^\top$. As a consequence, we can write $\bF_0 \bF_0^\top = \hat\bF_0\hat\bF_0^\top + \bV \bD \bV^\top$,
    where $\bV = \begin{bmatrix}\hat \bF_0 \bW_0 \tilde\vbeta & \tilde\bX\tilde\vbeta\end{bmatrix} \in \R^{n \times 2}$ and
    \begin{align*}
        \bD = \begin{bmatrix}
            0&c_1\\c_1&c_1^2 \Vert\bW_0\tilde\vbeta\Vert_2^2
        \end{bmatrix}.
    \end{align*}
    Using the Woodbury formula, we find
    that \eqref{wb2} still holds.
    Now, we can write
    \begin{align}
    \label{eq:woodbury-general-ell}
        H_p(\tilde\bX\tilde\vbeta&)^{\top} \bar\bR_0  H_p(\tilde\bX\tilde\vbeta)
        = H_p{(\tilde\bX\tilde\vbeta)}^{\top}  \hat\bR_0  H_p(\tilde\bX\tilde\vbeta)\nonumber\\
        &\hspace{2cm}-H_p{(\tilde\bX\tilde\vbeta)}^{\top}  \hat \bR_0 \bV (\bD^{-1} + \bV^\top \hat\bR_0 \bV)^{-1}\bV^\top \hat\bR_0 H_p{(\tilde\bX\tilde\vbeta)}.
    \end{align}
    We can analyze each term in the above sum separately.

    By Lemma \ref{lemma:poly-concentration}, $H_p{(\tilde\bX\tilde\vbeta)}^{\top}  \hat\bR_0  H_p(\tilde\bX\tilde\vbeta) - \E H_p{(\tilde\bX\tilde\vbeta)}^{\top}  \hat\bR_0  H_p(\tilde\bX\tilde\vbeta)\to_P 0$.
    Further,
    conditional on $\tilde\vbeta$, and using \ref{lem:twohermite}, we have
    \begin{align*}
    \E H_p{(\tilde\bX\tilde\vbeta)}^{\top}  \hat\bR_0  H_p(\tilde\bX\tilde\vbeta) &= 
    \E \tr\left[ \hat\bR_0  H_p(\tilde\bX\tilde\vbeta)H_p{(\tilde\bX\tilde\vbeta)}^{\top} \right] 
    = p!\;\E\tr\left[\hat\bR_0\right],
    \end{align*}
        and as in the proof of Lemma \ref{lemma:s12},
    $\E\tr\hat\bR_0-\E\tr\bar\bR_0\to  0$.
    Moreover, we have already argued 
    in the proof of Lemma \ref{lemma:l1_limits}
    that $\E\tr\bar\bR_0\to  \psi m_1/\phi$. Hence,
 $        H_p{(\tilde\bX\tilde\vbeta)}^{\top}  \hat\bR_0  H_p(\tilde\bX\tilde\vbeta) \to_P 
        p!\;\psi m_1/\phi.
  $
  To analyze the second term in \eqref{eq:woodbury-general-ell}, we first study 
      $H_p(\tilde \bX \tilde \vbeta)^{\top} \hat \bR_0 \hat \bF_0 \bW_0\vbeta$.
    Conditional on $\tilde\vbeta$, 
    $H_p(\tilde \bX\tilde \vbeta)$ is a vector
    with independent mean-zero, bounded variance entries,
    independent of the vector $\hat \bR_0 \hat \bF_0 \bW_0\tilde\vbeta$ that has norm $O(1/\sqrt{n})$. 
    Hence, we conclude that this term goes to zero. 
    Next, note that $H_p(\tilde \bX\tilde\vbeta)^\top \hat \bR_0  (\tilde \bX\tilde\vbeta) \to_P 0$ using Lemma \ref{lemma:poly-concentration} and Lemma \ref{lem:twohermite}. Moreover, 
     the limiting entries of 
     $(\bD^{-1} + \bV^\top \hat\bR_0 \bV)^{-1}$
     can be shown to be bounded by a simple orderwise analysis. Hence,  the second term in \eqref{eq:woodbury-general-ell} is $o_\sP(1)$. This concludes the proof.

\subsection{Proof of Lemma \ref{lemma:mixed_limit_general}}
\label{pflemma:mixed_limit_general}
We define $\vbeta_\perp = \frac{\vbeta_\star - \langle \vbeta_\star,\, \tilde\vbeta\rangle \tilde\vbeta}{\Vert\vbeta_\star - \langle \vbeta_\star,\, \tilde\vbeta\rangle \tilde\vbeta \Vert_2}$, and set
$    \hat \bX = \tilde\bX - \tilde \bX \tilde\vbeta\tilde \vbeta^\top - \tilde \bX \vbeta_\perp \vbeta_\perp^\top$.
By construction, we have $\hat \bX \independent \tilde \bX \tilde\vbeta, \tilde\vtheta_\star$. Based on the Gaussian equivalence from Appendix \ref{sec:gec}, we
can again replace $\bF_0$ with
$\bF_0 = c_1 \tilde\bX \bW_0^\top + c_{>1} \bZ$,
    where $\bZ \in \R^{n \times d}$ is an independent random matrix with $\normal(0, 1)$ entries. 
    Again, we define $\hat \bF_0$ as in \eqref{hf}.
    Thus, $\hat\bF_0 = \bF_0 - c_1 \tilde\bX \tilde\vbeta (\bW_0 \tilde \vbeta)^\top - c_1 \tilde \bX \vbeta_\perp (\bW_0 \vbeta_\perp)^\top$. As a consequence, we also have
    $
        \bF_0 \bF_0^\top = \hat\bF_0\hat\bF_0^\top + \bV \bD \bV^\top,
    $
    where 
    $\bV = \begin{bmatrix}\tilde\bX \tilde\vbeta& \tilde \bX \vbeta_\perp & \hat\bF_0 \bW_0 \tilde\vbeta&\hat\bF_0\bW_0\vbeta_\perp\end{bmatrix} \in \R^{n \times 4}$ and
    \begin{align*}
        \bD = \begin{bmatrix}
            c_1^2 \langle \bW_0 \tilde\vbeta, \bW_0 \tilde\vbeta\rangle & c_1^2 \langle \bW_0 \tilde\vbeta, \bW_0 \vbeta_\perp\rangle & c_1 & 0\vspace{0.2cm} \\
            c_1^2 \langle \bW_0 \tilde\vbeta, \bW_0 \vbeta_\perp\rangle&c_1^2 \langle \bW_0 \vbeta_\perp, \bW_0 \vbeta_\perp\rangle& 0 & c_1\vspace{0.2cm} \\
            c_1 & 0 & 0 & 0 \vspace{0.2cm} \\
            0 & c_1 & 0 & 0
        \end{bmatrix}.
    \end{align*}
    Using the Woodbury formula, we find that \eqref{wb2} still holds.
    We can write 
    \begin{align}
        H_p(\tilde\vtheta_\star)^\top \bar\bR_0 &H_q(\tilde\bX \tilde\vbeta)\nonumber\\
        &= H_p(\tilde\vtheta_\star)^\top \hat\bR_0 H_q(\tilde\bX \tilde\vbeta)
        -H_p(\tilde\vtheta_\star)^\top \hat \bR_0 \bV (\bD^{-1} + \bV^\top \hat\bR_0 \bV)^{-1}\bV^\top \hat\bR_0H_q(\tilde\bX \tilde\vbeta).
    \end{align}

\paragraph{$p \neq q$ case:}
The first term converges to zero for any $p\neq q$,
analogously to the argument in Section \ref{section:l2-Hinverse} for the terms (1,2) and (2,4).
In particular, for $p=0$, we can use orthogonal invariance as in the analysis of the term (2,4). To prove that the second term will also converge to zero when $p\neq q$, we first observe that the elements of $\bK = (\bD^{-1} + \bV^\top \hat\bR_0 \bV)^{-1}$ are all $O(1)$. The second term will involve quantities of the form
$
[\bK]_{i,j} H_p(\tilde\vtheta_\star)^\top\hat\bR_0\vv_i\vv_j^\top\hat\bR_0 H_q(\tilde \bX \tilde\vbeta)$,    
where $\vv_i$, for $i \in \{1, 2, 3, 4\}$, is the $i$-th column of the matrix $\bV=\begin{bmatrix}\tilde\bX \tilde\vbeta& \tilde \bX \vbeta_\perp & \hat\bF_0 \bW_0 \tilde\vbeta&\hat\bF_0\bW_0\vbeta_\perp\end{bmatrix}$. We can argue that all these terms converge to zero, as follows. Because $p \neq q$, without loss of generality, assume that $q \neq 1$.
\begin{itemize}
    \item The terms where $j = 1$ converge to zero because  $(\tilde \bX \tilde\vbeta)^\top\hat\bR_0 H_q(\tilde \bX \tilde\vbeta)$ 
    converges to zero using the concentration argument from Lemma \ref{lemma:poly-concentration} and the orthogonality of Hermite polynomials from Lemma \ref{lem:twohermite}.
    The same argument applies 
    to the terms where $j = 2$, via the convergence of $(\tilde \bX \vbeta_\perp)^\top \hat\bR_0 H_q(\tilde \bX \tilde\vbeta)$ to zero. 
    
    \item For $j=3,4$,
    and for $q>0$,
    since $H_q(\tilde\bX\tilde\vbeta)$ is independent of
$\hat\bR_0[\hat\bF_0\bW_0\tilde\vbeta\quad\,\,\,\hat\bF_0\bW_0\vbeta_\perp]$,
and has zero-mean i.i.d.~entries, it also follows 
that these entries converge to zero in probability.
For $q=0$, we can again use orthogonal invariance as in the analysis of the term (2,4).

\end{itemize}
\paragraph{The case when $p = q \neq 1$:}

Finally we study $H_p(\tilde\vtheta_\star)^\top \bar\bR_0H_p(\tilde\bX\tilde\vbeta)$, by analyzing the terms in \eqref{bigwb}.

For $H_p(\tilde\vtheta_\star)^\top \hat\bR_0 H_p(\tilde\bX \tilde\vbeta)$, 
since 
$H_p(\tilde\vtheta_\star), H_p(\tilde\bX \tilde\vbeta)$
are independent of $ \hat\bR_0$,
it follows from Lemma \ref{lemma:poly-concentration}, 
as in the analysis of term $(1,2)$
in the Section \ref{pl2},
that 
$$H_p(\tilde\vtheta_\star)^\top \hat\bR_0 H_p(\tilde\bX \tilde\vbeta) -  \E\hat\bR_0 \cdot \E H_p(\tilde\vtheta_\star)^\top H_p(\tilde\bX \tilde\vbeta) \to_P 0.$$
Now notice that $\hat \bF_0$ is left-orthogonally invariant in distribution, 
and thus 
$\hat\bR_0 =_d 
\bO \hat\bR_0 \bO^\top$,
 where $\bO$
    is uniformly distributed over the Haar measure of $n$-dimensional orthogonal matrices, independently of all other randomness. 
Hence, 
$\E\hat\bR_0 = \E\tr\hat\bR_0 \bI_n/n$.
Also, similar to the proof of Lemma \ref{lemma:s12}, we have $|\tr \bar\bR_0 - \tr \hat \bR_0| = o_\sP(1)$.
        Moreover, we have already argued 
    in the proof of Lemma \ref{lemma:l1_limits}
    that $\E\tr\bar\bR_0\to  \psi m_1/\phi$. Further, by Lemmas \ref{lem:twohermite} and \ref{lemma:turn_beta_to_beta_star},
$$H_p(\tilde\vtheta_\star)^\top \hat\bR_0 H_p(\tilde\bX \tilde\vbeta)\to_P 
p!\;\frac{\psi m_1}{\phi}
\left(\frac{c_{\star,1}}{\sqrt{\phi(c_\star^2 + \sigma_\ep^2) + c_{\star,1}^2}}\right)^{p}.$$

Next, 
we consider
$H_2(\tilde\vtheta_\star)^\top \hat \bR_0 \bV$
with 
$\bV = \begin{bmatrix}\tilde\bX \tilde\vbeta& \tilde \bX \vbeta_\perp & \hat\bF_0 \bW_0 \tilde\vbeta&\hat\bF_0\bW_0\vbeta_\perp\end{bmatrix}$.
For the first two entries of the vector 
$H_p(\tilde\vtheta_\star)^\top \hat \bR_0 \bV$,
an analysis very similar to the one above 
for 
$H_p(\tilde\vtheta_\star)^\top \hat\bR_0 H_p(\tilde\bX \tilde\vbeta)$
shows that they converge to zero in probability.
For the last two entries,
since $H_p(\tilde\vtheta_\star)$ is independent of
$\hat\bR_0[\hat\bF_0\bW_0\tilde\vbeta\quad\,\,\,\hat\bF_0\bW_0\vbeta_\perp]$,
and has zero-mean i.i.d.~entries, it also follows 
that these entries converge to zero in probability.
Moreover, 
     the limiting entries of 
     $(\bD^{-1} + \bV^\top \hat\bR_0 \bV)^{-1}$
     have been shown to be bounded in our above analysis. 
Hence, the second term
converges to zero in probability.
\paragraph{The case when $p = q =1$:}
In this case, we have
\begin{align*}
    (\tilde\bX\vbeta_\star)^\top \bar\bR_0 (\tilde\bX \tilde\vbeta) &= \frac{(\tilde\bX\vbeta_\star)^\top \bar\bR_0 (\tilde\bX \vbeta)}{\Vert\vbeta\Vert_2} = \frac{c_{\star,1}\frac{\psi m_2}{\phi}}{\sqrt{\phi(c_\star^2 + \sigma_\ep^2) + c_{\star,1}^2}} + o_\sP(1),
\end{align*}
using Lemma \ref{lemma:turn_beta_to_beta_star} and by arguments similar to the ones in the proof of Lemma \ref{lemma:l1_limits}.

Putting everything together concludes the proof.

\end{document}